\documentclass{article}

\usepackage[final]{neurips_2022}

\usepackage[utf8]{inputenc} 
\usepackage[T1]{fontenc}    
\usepackage{url}            
\usepackage{booktabs}       
\usepackage{amsfonts}       
\usepackage{nicefrac}       
\usepackage{microtype}      

\usepackage{bm}
\usepackage{float}
\usepackage{amsthm}
\usepackage{amsmath}
\usepackage{amssymb}
\usepackage{booktabs}
\usepackage{graphicx}
\usepackage{graphbox}
\usepackage{notes2bib}
\usepackage{subfigure}
\usepackage{multirow}
\usepackage{algorithm, algorithmic}

\usepackage[colorlinks,linkcolor=blue]{hyperref}
\usepackage[dvipsnames]{xcolor}
\definecolor{cite_color}{HTML}{114083}
\definecolor{link_color}{RGB}{0,102,102}
\definecolor{link_color}{RGB}{153, 0,0}  
\definecolor{url_color}{RGB}{153, 102,  0}
\definecolor{emp_color}{RGB}{0,0,255}
\hypersetup{
 colorlinks,
 citecolor=cite_color,
 linkcolor=link_color,
 urlcolor=url_color}
\graphicspath{{./figs/}}

\def \exp{\mathrm{exp}}

\newtheorem{innercustomgeneric}{\customgenericname}
\providecommand{\customgenericname}{}
\newcommand{\newcustomtheorem}[2]{%
  \newenvironment{#1}[1]
  {%
   \renewcommand\customgenericname{#2}%
   \renewcommand\theinnercustomgeneric{##1}%
   \innercustomgeneric
  }
  {\endinnercustomgeneric}
}

\newcustomtheorem{customthm}{Theorem}

\DeclareMathOperator*{\argmax}{argmax}
\DeclareMathOperator*{\argmin}{argmin}

\newtheorem{theorem}{Theorem}
\newtheorem{proposition}{Proposition}

\newtheorem{definition}{Definition}

\usepackage{thmtools}
\usepackage{thm-restate}

\usepackage{tikz}
\usetikzlibrary{decorations.pathreplacing,calc}
\newcommand{\tikzmark}[1]{\tikz[overlay,remember picture] \node (#1) {};}

\newcommand*{\AddNote}[5]{%
    \begin{tikzpicture}[overlay, remember picture]
        \draw [decoration={brace,amplitude=0.5em},decorate,ultra thick,ForestGreen]
            ($(#4)!(#2.north)!($(#4)-(0,1)$)$) --  
            ($(#4)!(#3.south)!($(#4)-(0,1)$)$)
                node [align=center, text width={#1}, pos=0.5, anchor=west] {#5};
    \end{tikzpicture}
}%

\usepackage[capitalize,noabbrev]{cleveref}
\crefname{section}{Section}{Sections}
\crefname{theorem}{Theorem}{Theorems}
\crefname{lemma}{Lemma}{Lemmas}
\crefname{equation}{Equation}{Equations}
\crefname{proposition}{Proposition}{Propositions}
\crefname{claim}{Claim}{Claims}
\crefname{appendix}{Appendix}{Appendices}
\crefname{algorithm}{Algorithm}{Algorithms}
\crefname{figure}{Figure}{Figures}
\crefname{table}{Table}{Tables}
\crefname{remark}{Remark}{Remarks}
\crefname{definition}{Def.}{Definitions}
\crefname{corollary}{Corollary}{Corollaries}

\usepackage{xspace}
\newcommand{\ours}{\texttt{EquiVSet}\xspace}
\newcommand{\oursdiff}{\texttt{DiffMF}\xspace}
\newcommand{\oursind}{$\texttt{EquiVSet}_\text{ind}$\xspace}
\newcommand{\ourscopula}{$\texttt{EquiVSet}_\text{copula}$\xspace}
\newcommand{\ebm}{$p_\theta (S)$}

\usepackage{etoc}
\etocdepthtag.toc{mtchapter}
\etocsettagdepth{mtchapter}{subsection}
\etocsettagdepth{mtappendix}{none}

\title{Learning Neural Set Functions Under the \\Optimal Subset Oracle}

\author{%
  Zijing Ou$^{1,2}$, \ Tingyang Xu$^{1}$, \ Qinliang Su$^{3}$, \ Yingzhen Li$^{2}$, \ Peilin Zhao$^{1}$, \ Yatao Bian$^{1}$\thanks{Correspondence to: Yatao Bian.} \\
  $^1$Tencent AI Lab, China\\
  $^2$Imperial College London, United Kingdom\\
  $^3$Sun Yat-sen University, China \\
  \texttt{z.ou22@imperial.ac.uk \{tingyangxu,masonzhao\}@tencent.com } \\
  \texttt{suqliang@mail.sysu.edu.cn yingzhen.li@imperial.ac.uk yatao.bian@gmail.com} \\
\setcounter{footnote}{0}
}

\begin{document}

\maketitle

\begin{abstract}
Learning neural set functions becomes increasingly important in many applications like product recommendation and  compound selection in AI-aided drug discovery. 
The majority of existing works study methodologies of set function learning under the function value oracle, which, however, requires  expensive supervision signals. 
This renders it impractical for applications with only weak supervisions under the Optimal Subset (OS) oracle, the study of which is surprisingly overlooked. 
In this work, we present a principled yet practical maximum likelihood learning framework, termed as \ours,\footnote{Code is available at: \url{https://github.com/SubsetSelection/EquiVSet}.}
that simultaneously meets the following desiderata of learning neural set functions under the OS oracle: i) permutation invariance of the set mass function being modeled; ii) permission of varying ground set; iii) minimum prior; and iv) scalability. The main components of our framework involve: an energy-based treatment of the set mass function, DeepSet-style architectures to handle permutation invariance, mean-field variational inference, and its amortized variants. 
Thanks to the elegant combination of these advanced architectures, empirical studies on three real-world applications (including  Amazon product recommendation, set anomaly detection and compound selection for virtual screening) demonstrate that \ours  outperforms the baselines by a large margin. 
\end{abstract}

\section{Introduction}

Many real-world applications involve prediction of set-value outputs, such as recommender systems which output a set of products to customers, anomaly detection that predicts the outliers from the majority of data \citep{zhang2020set}, and compound selection  for virtual screening in drug discovery aims at extracting the most effective compounds from a given compound database \citep{gimeno2019light}. All of these applications implicitly learn a set function \citep{rezatofighi2017deepsetnet,zaheer2017deep} that measures the utility of a given set input, such that the most desirable set output has the highest (or lowest {\it w.l.o.g}) utility value.

More formally, consider a recommender system: given a set of product candidates $V$, it is expected to recommend a subset of products $S^* \subseteq V$ to the user, which would satisfy the user  most, {\it i.e.}, offering the maximum utility to the user. 
We assume the underlying process of determining $S^*$ can be modelled by a  utility function $F_\theta (S;V)$ parameterized by $\theta$,  and the following criteria: 
\begin{align} \label{prob_set_func_learning}
    S^* = \argmax_{S \in 2^{V}} F_\theta (S; V).
\end{align}
There are mainly two settings for learning  the utility function. The first one, namely function value (FV) oracle, targets at learning $F_\theta (S;V)$ to fit the utility explicitly, under the supervision of data in the form of $\{(S_i, f_i)\}$ for a fixed ground set $V$, where $f_i$ is the true utility function value of the subset $S_i$. However, training in this way is prohibitively expensive, since one needs to construct large amounts of supervision signals for a specific ground set $V$ \citep{DBLP:journals/siamcomp/BalcanH18}. 
Here we consider an alternative setting, which learns $F_\theta (S;V)$ in an implicit way. More formally, with the data in form of $\{(V_i, S_i^*)\}_{i=1}^N$, where $S_i^*$ is the optimal subset (OS) corresponding to $V_i$, our goal is to estimate $\theta$ such that for all possible $(V_i, S_i^*)$, it satisfies equation \eqref{prob_set_func_learning}. The OS oracle is arguably more practical than the FV oracle, 
which alleviates the need for explicitly labeling utility values for a large amount of subsets.\footnote{
Notably, learning set functions under the OS oracle is distinct to  that under the FV oracle; the two settings are not comparable in general. 
To illustrate this, one can easily obtain the FV oracle of maximum cut set functions, but fail to specify the OS oracle since it is NP-complete to solve the maximum cut problem \citep[Appendix A2.2]{garey1979computers}. Moreover, even though the OS oracle naturally shows up in the product recommendation scenario, one cannot identify its FV oracle since the true utility values are hard to obtain.}

Though being critical for practical success, related study on set utility function learning under the OS supervision oracle  is surprisingly lacked. The most relevant work is the probabilistic greedy model (PGM), which solves the optimization problem of  \eqref{prob_set_func_learning} with a greedy maximization algorithm \citep{tschiatschek2018differentiable}. Specifically, PGM interprets the maximization algorithm as to construct differentiable distributions over sequences of items in an auto-regressive manner. However, such construction of distributions is problematic for defining distributions on sets due to the dependency on the sampling order. Therefore, they alleviate this issue by enumerating all possible permutations of the sampling sequence (detailed discussion is given in \cref{appendix_pgm}). Such enumerations scale poorly due to the combinatorial cost $\mathcal{O}(|V|!)$, which hinders PGM's applicability to real-world applications.

To learn set functions under the OS oracle, we advocate the maximum likelihood paradigm \citep{stigler1986history}. Specifically, this learning  problem can be viewed from a probabilistic perspective
\begin{flalign} \label{OS_learning_objective}  
    &\argmax_\theta\  \mathbb{E}_{\mathbb{P}(V, S)} [\log p_\theta (S | V)] \\\notag   
    &\operatorname{s.t.} \  p_\theta (S | V) \propto  F_\theta (S ; V), \forall  S \in 2^V, 
\end{flalign}
where the constraint admits the learned set function to obey the objective defined in \eqref{prob_set_func_learning}. 
Given limited data $\{(V_i, S_i^*)\}_{i=1}^N$ sampled from the underlying data distribution ${\mathbb{P}(V, S)}$, one would maximize the empirical log likelihood:  $\sum_{i=1}^N[\log p_\theta (S_i^* | V_i)]$. 
The most important step is to construct a proper set distribution $p_\theta (S | V)$ whose probability mass monotonically grows with the utility function $F_\theta (S;V)$ and satisfy the following additive requirements: (i) {\emph{permutation invariance}}: the probability mass should not change under any permutation of the elements in $S$; (ii) {\emph{varying ground set}}: the function should be able to process input sets of variable size; iii) {\emph{minimum prior}}: we should make no assumptions of the set probability, {\it i.e.,} with maximum entropy, which is equivalent to the uninformative prior \citep{jeffreys1946invariant}; and iv) {\emph{scalibility}}: the learning algorithm should be scalable to large-scale datasets and run in polynomial time.

In this paper, we propose {\textbf{Equi}variant \textbf{V}ariational inference for \textbf{Set} function learning} (\ours), a new method for learning set functions under the OS oracle, which satisfies all the requirements. Specifically, we use an energy-based model (EBM) to construct the set mass function. EBMs are maximum entropy distributions, which satisfies the \emph{minimum prior} requirement.  Moreover, by modeling the energy function with DeepSet-style architectures \citep{zaheer2017deep,lee2019set}, the two requirements, {\it i.e.,} \emph{permutation invariance} and \emph{varying ground set} are naturally satisfied. Unfortunately, the flexibility of EBMs exacerbates the difficulties of learning and inference, since the inputs of set are discrete and lie in an exponentially-large space. To remedy this issue, we develop an approximate maximum likelihood approach which estimates the marginals via the mean-field variational inference, resulting in an efficient training manner under the supervision of OS oracles. 
In order to ensure \emph{scalability}, an amortized inference network with permutation equivariance is proposed, which allows the model to be trained on large-scale datasets. 

Although it may be seen as combining existing components in approximate inference, the proposed framework addresses a surprisingly overlooked problem in the set function learning communities using an intuitive yet effective method.
Our main contributions are summarized below:
\begin{itemize}
    \item We formulate set functions learning problems under the OS supervision oracle using the maximum likelihood principle;
    \item We present an elegant framework based on EBMs which satisfies the four desirable requirements and is efficient both at training and inference stages;
    \item Real-world experiments demonstrate effectiveness of the proposed OS learning framework.
\end{itemize}

\section{Energy-Based Modeling for Set Function Learning}
The first step to solve problem \eqref{OS_learning_objective} is to construct a proper set mass function $p_\theta (S|V)$ monotonically growing with the utility function $F_\theta (S;V)$. There exits countless ways to construct such a probability mass function, such as the sequential modeling in PGM \citep[Section 4]{tschiatschek2018differentiable}.
Here we resort to the energy-based treatment:
\begin{align} \label{energy_prob}
     p_\theta (S|V) = \frac{\mathrm{exp}( F_\theta (S; V))}{Z}, \; Z := \sum\nolimits_{S' \subseteq V}  \mathrm{exp}( F_\theta (S'; V)),
\end{align}
where the utility function $F_\theta (S; V)$ stands for the negative energy, with higher utility representing lower energy. The energy-based treatment is attractive, partially due to its maximum entropy ({\it i.e.}, minimum prior) property. That is, it assumes nothing about what is unknown, which is known as the ``noninformative prior'' principle in Bayesian modeling \citep{jeffreys1946invariant}. This basic principle is, however, violated by the set mass function defined in PGM. We refer detailed motivation of the energy-based modeling to \cref{proof_maximum_entropy}.

In addition to the \emph{minimum prior}, the energy-based treatment also enables the set mass function $p_\theta (S|V)$ to meet the other two requirements, {\it i.e.} \emph{permutation invariance} and \emph{varying ground set}, by deliberately designing a suitable set function $F_\theta (S;V)$. 
However, modeling such a proper function is nontrivial, since classical feed-forward neural networks (e.g., the ones designed for submodular set functions \citep{DBLP:journals/corr/BilmesB17}) violate both two criteria, which restricts their applicability to the problems involving a set of objects. Fortunately,  \citet{zaheer2017deep} sidestep this issue by introducing a novel architecture, namely DeepSet. They theoretically prove the following \namecref{permutation_invariant}.
\begin{proposition} \label{permutation_invariant}
All permutation invariant set functions can be decomposed in the form $f(S) = \rho \left( \sum_{s \in S} \kappa(s) \right)$, for suitable transformations $\kappa$ and $\rho$.
\end{proposition}
By combining the energy-based model in \eqref{energy_prob} with DeepSet-style architectures, we could construct a valid set mass function to meet two important criteria: \emph{permutation invaraince} and \emph{varying ground set}. However, the flexibility of EBMs exacerbates the difficulties of learning and inference, since the partition function $Z$ is typically intractable and the input of sets is undesirably discrete. 

\section{Approximate Maximum Likelihood Learning with OS  Supervision Oracle}
In this section, we explore an effective framework for learning set functions under the supervision of optimal subset oracles. We start with discussing the principles for learning parameter $\theta$, followed by discussing the detailed inference method for discrete EBMs.

\subsection{Training Discrete EBMs Under the Guidance of Variational Approximation}
For discrete data, {\it e.g.,} set, learning the parameter $\theta$ in \eqref{energy_prob} via maximum likelihood is notoriously difficult. Although one could apply techniques, such as ratio matching \citep{lyu2012interpretation}, noise contrastive estimation \citep{tschiatschek2016learning}, and contrastive divergence \citep{carreira2005contrastive}, they generally suffer from instability on high dimensional data, especially when facing very large ground set in real-world applications. 
Instead of directly maximizing the log likelihood, we consider an alternative optimization objective that is computationally preferable. Specifically, we first fit a variational approximation to the EBM by solving
\begin{align} \label{variational_approximate}
    \boldsymbol{\psi}^* =  \argmin_{\boldsymbol{\psi}} D(q(S;\boldsymbol{\psi}) || p_\theta (S)),
\end{align}
where $D(\cdot || \cdot)$ is a discrepancy measure between two distributions, \ebm\footnote{Here we omit the condition $V$ for brevity. In some specific context, it would be helpful to regard subset $S$ as a binary vector, {\it i.e.}, $S := \{0,1\}^{|V|}$ with the $i$-th element equal to $1$ meaning $i \in S$ and $0$ meaning $i \not\in S$.} is the EBM defined in \eqref{energy_prob}, and $q(S;\boldsymbol{\psi})$ denotes the mean-field variational distribution with the parameter $\boldsymbol{\psi} \in [0,1]^{|V|}$ standing for the odds that each item $s \in V$ shall be selected in the optimal subset $S^*$. Note that the optimal parameter $\boldsymbol{\psi}^*$ of \eqref{variational_approximate} can be viewed as a function of $\theta$. In this regard, we can optimize the parameter $\theta$ by minimizing the following cross entropy loss,\footnote{This objective would suffer from label-imbalanced problem when the size of OS is too small. In practice, we can apply negative sampling to overcome this problem: we randomly select a negative set $N_i \subseteq V_i \backslash S^*_i$ with the size of $|S^*|$, and train the model with an alternative objective $\sum\limits_{i} \!-\!\! \sum\limits_{j \in S^*_i} \! \log \psi_j^* \!-\!\! \sum\limits_{j \in N_i } \! \log (1 - \psi_j^*)$.} which is well-known to be implementing the maximum likelihood estimation \citep{goodfellow2016deep} {\it w.r.t.} the surrogate distribution $q(S;\boldsymbol{\psi}^*)$,
\begin{align} \label{cross_entropy}
    \mathcal{L}(\theta;\boldsymbol{\psi}^*) = \mathbb{E}_{\mathbb{P}(V,S)}[-\log q(S;\boldsymbol{\psi}^*)] \approx \frac{1}{N} \sum\limits_{i=1}^N  \left( - \sum\limits_{j \in S_i^*} \log \psi_j^* - \sum\limits_{j \in V_i \backslash S_i^*} \log (1 - \psi_j^*) \right).
\end{align}
This is also known as the marginal-based loss \citep{domke2013learning}, which trains probabilistic models by evaluating them using the marginals approximated by an inference algorithm.
Despite not exactly bounding the log-likelihood of \eqref{energy_prob}, this objective, as pointed out by \citet{domke2013learning}, benefits from taking the approximation errors of inference algorithm into account while learning.
However, minimizing \eqref{cross_entropy} requires the variational parameter $\boldsymbol{\psi}^*$ being differentiable {\it w.r.t.} $\theta$. Inspired by the differentiable variational approximation to the Markov Random Fields  \citep{krahenbuhl2013parameter,zheng2015conditional,dai2016discriminative}, below, we extend this method to the deep energy-based formulation, which admits an end-to-end training paradigm with the back-propagation algorithm.

\subsection{Differentiable Mean Field Variational Inference}

\begin{figure*}[!t]
\vspace{-5mm}
\centering
\begin{minipage}[t]{0.35\linewidth}
\centering
\begin{algorithm}[H]
\caption{$\operatorname{MFVI}(\boldsymbol{\psi},V,K)$} \small
\label{alg:MFI} 
\begin{algorithmic}[1] 
	\FOR{$k \!\leftarrow\! 1,\dots,K$}
	\FOR{$i \!\leftarrow\! 1,\dots,|V|$ in parallel}
    \STATE sample $m$ subsets \\ $\!\!\!\!\! S_n \!\!\sim\!\! q(S;(\boldsymbol{\psi}^{(k\!-\!1)} | \psi_i^{(k\!-\!1)} \!\leftarrow\! 0))$ 
    \STATE update variational parameter \\ $\!\!\!\!\!\!\!\!\!\!\!\!\boldsymbol{\psi}_i^{(k)} \!\!\! \leftarrow \!\! \sigma (\! \frac{1}{m} \!\!\! \sum\limits_{n=1}^m \! [F_\theta (S_n \!\!+\! i) \!\!-\!\! F_\theta (S_n) ])$ 
    \vspace{-2.9mm}
    \ENDFOR 
    \ENDFOR 
\end{algorithmic}
\end{algorithm}
\end{minipage}
\begin{minipage}[t]{0.31\linewidth}
\centering
\begin{algorithm}[H]
\caption{$\operatorname{DiffMF}(V,S^*)$} \small
\label{alg:diffMF} 
\begin{algorithmic}[1] 
	\STATE initialize variational parameter \\ $\boldsymbol{\psi}^{(0)} \leftarrow 0.5 * \mathbf{1}$
	\STATE compute the marginals \\ $\boldsymbol{\psi}^* \leftarrow \operatorname{MFVI}(\boldsymbol{\psi}^{(0)},V,K)$ 
	\STATE update parameter $\theta$ using \eqref{cross_entropy} \\ $\theta \leftarrow \theta - \eta \nabla_\theta \mathcal{L}(\theta; \boldsymbol{\psi}^*)$ 
	\vspace{8.5mm} 
\end{algorithmic}
\end{algorithm}
\end{minipage}
\begin{minipage}[t]{0.32\linewidth}
\centering
\begin{algorithm}[H]
\caption{$\operatorname{EquiVSet}(V,S^*)$} \small
\label{alg:EAVI} 
\begin{algorithmic}[1] 
	\STATE update parameter $\phi$ using \eqref{elbo} \\ $\phi \leftarrow \phi + \eta \nabla_\phi \textsc{ELBO}(\phi)$
	\STATE initialize variational parameter \\ $\boldsymbol{\psi}^{(0)} \leftarrow \operatorname{EquiNet}(V;\phi)$
	\STATE one step fixed point iteration \\ $\boldsymbol{\psi}^* \!\leftarrow\! \operatorname{MFVI}(\boldsymbol{\psi}^{(0)},V,K=1)$ 
	\STATE update parameter $\theta$ using \eqref{cross_entropy} \\ $\theta \leftarrow \theta - \eta \nabla_\theta \mathcal{L}(\theta;\boldsymbol{\psi}^*)$ \vspace{1.5mm}
\end{algorithmic}
\end{algorithm}
\end{minipage} 
\caption{The main components and algorithms in our framework. Note that \oursdiff and \ours are for one training sample only. Detailed and self-contained  descriptions of each component of these algorithms are presented in \cref{appendix_pseudo_code_4_EquivSet_copula}.}
\vspace{-4mm}
\end{figure*}

To solve the optimization problem \eqref{variational_approximate}, we need to specify the variational distribution $q(S;\boldsymbol{\psi})$ and the divergence measure $D(\cdot || \cdot)$, such that the optimum marginal $\boldsymbol{\psi}^*$ is differentiable {\it w.r.t.} the model parameter $\theta$. A natural choice is to restrain $q(S;\boldsymbol{\psi})$ to be fully factorizable, which leads to a mean-field approximation of \ebm. The simplest form of $q(S;\boldsymbol{\psi})$ would be a $|V|$ independent Bernoulli distribution, {\it i.e.}, $q(S;\boldsymbol{\psi}) = \prod_{i \in S}\psi_i \prod_{i \not\in S}(1-\psi_i), \boldsymbol{\psi}\in[0,1]^{|V|}$. Further restricting the discrepancy measure $D(q||p)$ to be the Kullback-Leibler divergence, we recover the well-known mean-field variational inference method.
It turns out that minimizing the KL divergence amounts to maximizing the evidence lower bound (\textsc{ELBO})
\begin{align} \label{elbo}
    \min_{\boldsymbol{\psi}} \mathbb{KL}(q(S,\boldsymbol{\psi}) || p_\theta (S))
    \quad \Leftrightarrow \quad \max_{\boldsymbol{\psi}} f_{\mathrm{mt}}^{F_\theta} (\boldsymbol{\psi}) + \mathbb{H}(q(S;\boldsymbol{\psi})) =: \textsc{ELBO},
\end{align}
where $f_{\mathrm{mt}}^{F_\theta}(\boldsymbol{\psi})$ is the multilinear extension of $F_\theta (S)$ \citep{calinescu2007maximizing}, which is defined as
\begin{align} \label{multilinear_extension}
    f_{\mathrm{mt}}^{F_\theta} (\boldsymbol{\psi}) :=\! \sum\limits_{S \subseteq V} \! F_\theta (S) \prod_{i \in S} \! \psi_i \! \prod_{i \not\in S} \! (1 - \psi_i), \boldsymbol{\psi} \in [0,1]^{|V|}.
\end{align}
To maximize the ELBO in \eqref{elbo}, one can apply the fixed point iteration algorithm. Specifically, for coordinate $\psi_i$, the partial derivative of the multilinear extension is $\nabla_{\psi_i} f_{\mathrm{mt}}^{F_\theta} (\boldsymbol{\psi})$, and for the entropy term, it is $\nabla_{\psi_i} \mathbb{H}(q) = \log \frac{1-\psi_i}{\psi_i}$. Thus, the stationary condition of maximizing ELBO is $\psi_i = \sigma(\nabla_{\psi_i} f_{\mathrm{mt}}^{F_\theta} (\boldsymbol{\psi})), i = 1, \dots, |V|$, where $\sigma$ is the sigmoid function, which means $\psi_i$ should be updated as $\psi_i \leftarrow \sigma(\nabla_{\psi_i} f_{\mathrm{mt}}^{F_\theta} (\boldsymbol{\psi}))$. This analysis leads to the traditional mean field iteration, which updates each coordinate one by one (detailed derivation in \cref{proof_elbo}). In this paper, we suggest to update $\boldsymbol{\psi}$ in a batch manner, which is more efficient in practice. More specifically, we summarize the mean field approximation as the following fixed-point iterative update steps
\begin{align}
    & \boldsymbol{\psi}^{(0)} \leftarrow \operatorname{Initialize \ in \ } [0,1]^{|V|}, \\
   &  \boldsymbol{\psi}^{(k)} \leftarrow (1 + \mathrm{exp}(- \nabla_{\boldsymbol\psi^{(k-1)}} f_{\mathrm{mt}}^{F_\theta} (\boldsymbol{\psi}^{(k-1)})))^{-1},  \label{fpi-formula} \\
   &  \boldsymbol{\psi}^* \leftarrow \boldsymbol{\psi}^{(K)}.
\end{align}
We denote the above iterative steps as a function termed as $\operatorname{MFVI} (\boldsymbol{\psi},V,K)$, which takes initial vairational parameter $\boldsymbol{\psi}$, ground set $V$, and number of iteration steps $K$ as input, and outputs the parameter $\boldsymbol{\psi}^*$ after $K$ steps. Note that, $\operatorname{MFVI} (\boldsymbol{\psi},V,K)$ is differentiable {\it w.r.t.} the parameter $\theta$, since each fixed-point iterative update step is differentiable. Thereby, one could learn $\theta$ by minimizing the cross entropy loss in \eqref{cross_entropy}. However, the computation complexity raises from the derivative of multilinear extension $f_{\mathrm{mt}}^{F_\theta} (\boldsymbol{\psi})$ defined in \eqref{multilinear_extension}, which sums up all the possible subsets in the space of size $2^{|V|}$. Fortunately, the gradient $\nabla_{\boldsymbol\psi} f_{\mathrm{mt}}^{F_\theta}$ can be estimated efficiently via Monte Carlo approximation methods, since the following equation holds.
\begin{align} \label{mc_gradient}
    \nabla_{\psi_i} f_{\mathrm{mt}}^{F_\theta} (\boldsymbol{\psi}) = \mathbb{E}_{q(S;(\boldsymbol{\psi} | \psi_i \leftarrow 0))} \left[ F_\theta (S + i) - F_\theta (S) \right],
\end{align}
in which we use $S+i$ to denote the set union $S \cup \{i\}$.
Detailed derivation is provided in \cref{proof_gradient_mt}. According to  \eqref{mc_gradient}, we can estimate the partial derivative $\nabla_{\psi_i} f_{\mathrm{mt}}^{F_\theta}$ via Monte Carlo approximation:
i) sample $m$ subsets $S_n, n=1,\dots,m$ from the surrogate distribution $q(S;(\boldsymbol{\psi} | \psi_i \leftarrow 0))$; ii) approximate the expectation by the average $\frac{1}{m} \sum_{k=1}^n [F_\theta (S_n + i) - F_\theta (S_n)]$.
After training, the OS for a given ground set can be sampled via rounding $\boldsymbol{\psi}^*$, which is the optimal variational parameter after $K$-steps mean-field iteration, {\it i.e.,} $\boldsymbol{\psi}^* = \operatorname{MFVI}(\boldsymbol{\psi}, V, K)$, and stands for the probability of each element in the ground set should be sampled.\footnote{Here we simply apply the topN rounding, but it is worthwhile to explore other rounding methods as a future work.}
We term this method as \textbf{Diff}erentiable \textbf{M}ean \textbf{F}ield (\oursdiff) and summarize the training and inference process in Algorithm \ref{alg:diffMF} and \ref{alg:MFI}, respectively.

\section{Amortizing Inference with Equivariant Neural Networks}
Although \oursdiff can learn set function $F_\theta$ in an effective way, it undesirably has two notorious issues: i) the computation is in general prohibitively expensive, since \oursdiff involves a typically expensive sampling loop per data point; ii) some information regarding interactions between elements is discarded, since \oursdiff assumes a fully fatorizable variational distribution. In this section, we first propose to amortize the inference process with an additional recognition neural network, and then extend it to considering correlation for more accurate approximations.

\subsection{Equivariant Amortized Variational Inference}
To enable training the proposed model on a large-scale dataset,  we propose to amortize the approximate inference process with an additional recognition neural network which outputs parameter $\boldsymbol{\psi}$ for the variational distribution $q_\phi(S;\boldsymbol{\psi})$,\footnote{With a slight abuse of notations, we use the same symbol here as in \eqref{elbo}.} where $\phi$ denotes the parameter of neural networks. A proper recognition network involving set objects shall satisfy the property of \emph{permutation equivariance}.
\begin{definition}
A function $f: \mathcal{X}^d \rightarrow \mathcal{Y}^d$ is called permutation equivalent when upon permutation of the input instances permutes the output labels, {\it i.e.,} for any permutation $\pi$: $f( \pi ([x_{1}, \dots, x_{d}]) ) = \pi ( f( [x_1,\dots, x_d] ) )$.
\end{definition}
\cite{zaheer2017deep} propose to formulate the \emph{permutation equivariant} architecture as :
\begin{align} \label{deepse_equivariance}
    f_i (S) = \rho \left(\lambda \kappa(s_i) + \gamma \sum\nolimits_{s \in S} \kappa(s) \right),
\end{align}
where $s_i$ denotes the $i^{\mathrm{th}}$ element in the set $S$, $\lambda, \gamma$ are learnable scalar variables, and $\rho, \kappa$ are any proper transformations. Note that the output value of $f_i: 2^V \rightarrow [0,1]$ is relative to the $i^{\mathrm{th}}$ coordinate, but not the order of the elements in $S$. Thus the equivariant recognition network, denoted as $\boldsymbol{\psi} = \operatorname{EquiNet}(V;\phi): 2^V \rightarrow [0,1]^{|V|}$, can be defined as $\operatorname{EquiNet}_i := f_i$, which takes the ground set $V$ as input and outputs the distribution parameter $\boldsymbol{\psi}$ for $q_\phi(S;\boldsymbol{\psi})$. 

\begin{figure}[t]
\centering
\includegraphics[width=\linewidth]{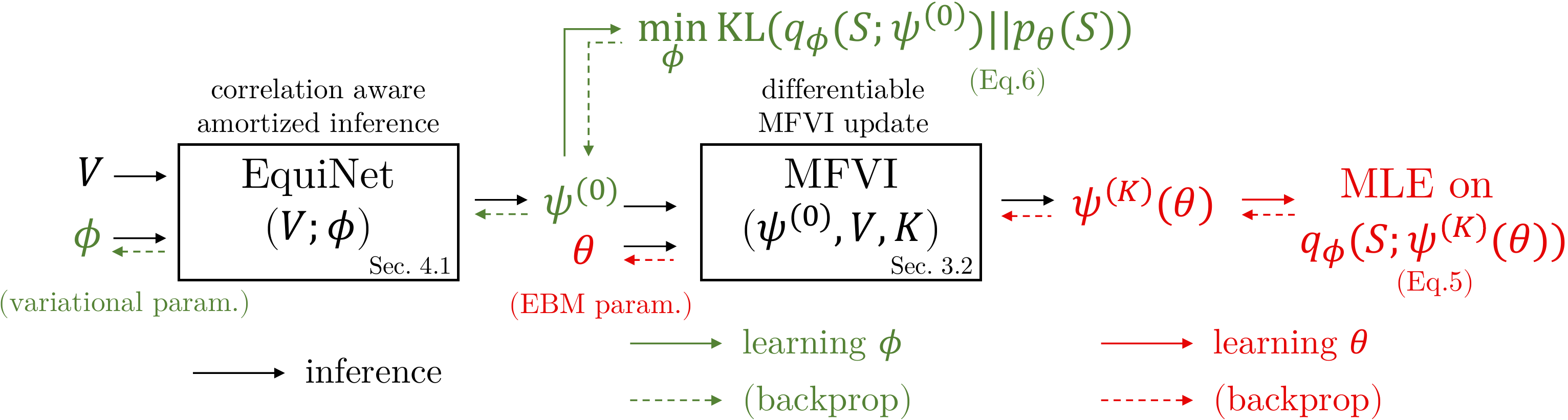}
\caption{Overview of the training and inference processes of \ours.}
\label{fig:equivset}
\vspace{-4mm}
\end{figure}

\subsection{Correlation-aware Inference with Gaussian Copula}
Due to the mean-field assumption, the proposed variational distribution cannot model the interactions among elements in the input set.
We address this issue by introducing Gaussian copula \citep{nelsen2007introduction,tran2015copula,suh2016gaussian,wang2020relaxed}, which is a cumulative distribution function (CDF) of random variables $(u_1 , \dots, u_{|V|})$ over the unit cube $[0,1]^{|V|}$, with $u_i \sim \operatorname{Uniform}(0,1)$. More formally, given a covariance matrix $\boldsymbol{\Sigma}$, the Gaussian copula $C_{\boldsymbol{\Sigma}}$ with parameter $\boldsymbol{\Sigma}$ is defined as
\begin{align}
    C_{\boldsymbol{\Sigma}}(u_1, \cdots, u_{|V|}) = \Phi_{\boldsymbol{\Sigma}} \left(\Phi^{-1}(u_1), \cdots , \Phi^{-1}(u_{|V|})\right), \nonumber
\end{align}
where $\Phi_{\boldsymbol{\Sigma}}$ stands for the joint CDF of a Gaussian distribution with zero mean and covariance matrix $\boldsymbol{\Sigma}$, and $\Phi^{-1}$ is the inverse CDF of standard Gaussian. With the location parameter $\boldsymbol{\psi}$ output by $\operatorname{EquiNet}(V;\phi)$, we can induce correlation into the Bernoulli distribution via the following way: i) sample an auxiliary noise $\boldsymbol{g} \sim \mathcal{N}(\boldsymbol{0}, \boldsymbol{\Sigma})$; ii) apply element-wise Gaussian CDF $\boldsymbol{u} = \boldsymbol{\Phi}_{\operatorname{diag}(\boldsymbol{\Sigma})}(\boldsymbol{g})$; iii) obtain binary sample via $\boldsymbol{s} = \mathbb{I}(\boldsymbol{\psi} \leq \boldsymbol{u})$,\footnote{Here $\boldsymbol{s}$ is a binary vector $\{0,1\}^{|V|}$ with the $i$-th element equal to $1$ meaning $i \in S$ and $0$ meaning $i \not\in S$.} where $\boldsymbol{\psi} \leq \boldsymbol{u}$ means $\forall i, \psi_i \leq u_i$, $\mathbb{I}(\cdot)$ is the indicator function, and $\operatorname{diag}(\boldsymbol{\Sigma})$ returns the diagonal matrix of $\boldsymbol{\Sigma}$. In practice, the covariance matrix $\boldsymbol{\Sigma}$ could be generated by another neural network with the input ground set. We refer the discussion on it to \cref{appendix_lower_rank_per}, and demonstrate how to efficiently construct and sample from a non-diagonal Gaussian distribution, while retaining a \emph{permutation equivariant} sampling process. 

To learn the parameters of the variational distribution, one can maximize the ELBO objective in \eqref{elbo}. However, the \textsc{ELBO} has no differentiable closed-form expression {\it w.r.t.} $\phi$.\footnote{For correlation-aware inference, the variational parameter $\phi$ consists of two parts: i) $\phi$ of the $\operatorname{EquiNet}(V;\phi)$ and ii) $\boldsymbol{\Sigma}$ of the Gaussian copula.} To remedy this issue, we relax the binary variable $\boldsymbol{s}$ to a continuous one by applying the Gumbel-Softmax trick \citep{jang2016categorical,wang2020relaxed}, resulting in an end-to-end training process with backpropagation.

\subsection{Details of Training and Inference}
Our model consists of two components: the EBM \ebm and the variational distribution $q_\phi (S;\boldsymbol{\psi})$. 
As shown in \cref{fig:equivset}, these two components are trained in a cooperative learning fashion \citep{xie2018cooperative}.
Specifically, we train the variational distribution $q_\phi$ with fixed $\theta$ firstly by maximizing the \textsc{ELBO} in \eqref{elbo}. To train the energy model $p_\theta$, we first initialize the variational parameter $\boldsymbol{\psi}^{(0)}$ with the output of equivariant recognition network $\operatorname{EquiNet}(V;\phi)$. This enables us to get a more accurate variational approximate, since $q_\phi$ has modeled the correlation among the elements in the set. 
Notice that $\boldsymbol{\psi}^{(0)}$ does not depend on $\theta$ directly.
To learn $\theta$, we take one further  step of mean-field iteration $\operatorname{MFVI}(\boldsymbol{\psi}^{(0)}, V,1)$, which flows the gradient through $\theta$ and  enables  to optimize $\theta$ using the cross entropy loss in \eqref{cross_entropy} ({\it i.e.,} if we skip step 3 in \cref{alg:EAVI}, and feed $\boldsymbol{\psi}^{(0)}$ to step 4, the gradient would not flow through $\theta$).
However, if we take multiple steps, it inclines to converge to the local optima that is the same as the original mean-field iteration. As a result, the benefit of correlation-aware inference provided by the Gaussian copula would be diminished. Detailed analysis is provided in \cref{sensitive-analysis}. 
The training procedure is summarized in \cref{alg:EAVI} (the complete version is given in \cref{appendix_pseudo_code_4_EquivSet_copula}).

For \emph{inference} in the test time, given a ground set $V$, we initialize the variational parameter via $\boldsymbol{\psi}^{(0)} = \operatorname{EquiNet}(V; \phi)$, then run one step mean-field iteration $\boldsymbol{\psi}^* \leftarrow \operatorname{MFVI}(\boldsymbol{\psi}^{(0)}, V,1)$. Finally, the corresponding OS  is obtained by applying the topN rounding method.
We term our method as \textbf{Equi}variant \textbf{V}ariational Inference for \textbf{Set} Function Learning (\ours), and respectively use \oursind and \ourscopula to represent two variants with independent and copula variational posterior, respectively.

\section{Related Work}

\textbf{Set function learning.} 
There is a growing literature on learning set functions with deep neural networks. \citet{zaheer2017deep} designed the DeepSet architecture to create permutation invariant and equivariant function for set prediction. \citet{lee2019set} enhanced model ability of DeepSet by employing transformer layer to introduce correlation among instances of set, and \citet{horn2020set} extended this framework for time series. It is noteworthy that they all learn set functions under  the function value oralce and can be employed as the backbone of  the utility function $F_\theta (S;V)$ in our model. 
\cite{DBLP:conf/nips/DolhanskyB16,DBLP:journals/corr/BilmesB17, DBLP:journals/eswa/GhadimiB20}  have also designed deep architectures for submodular set functions, however, these designs can not handle the varying ground set requirement.
There are papers studying the learnability of specific  set functions (e.g., submodular functions and subadditive functions) in a distributional learning setting \citep{DBLP:journals/jmlr/BalcanCIW12,DBLP:conf/soda/BadanidiyuruDFKNR12, DBLP:journals/siamcomp/BalcanH18} under the function value oracle, they mainly provide 
sample complexity with inapproximability results  under the probably mostly approximately correct (PMAC)  learning model.  
Other methods relevant to our setting are TSPN \citep{kosiorek2020conditional} and DESP \citep{zhang2020set}. However they both focused on generating set objects under a given condition. While we aim at predicting under   the optimal subset oracle.

\textbf{Energy-based modeling.} 
Energy based learning \citep{lecun2006tutorial} is a classical framework to model the underlying distribution over data. 
Since it makes no assumption of data, energy-based models are extremely flexible and have been applied to wide ranges of domains, such as data generation \citep{nijkamp2019learning}, out-of-distribution detection \citep{liu2020energy}, game-theoretic valuation algorithms \citep{bian2022energy} and biological structure prediction \citep{shi2021learning}. Learning EBMs can be done by applying some principled methods, like contrastive divergence \citep{hinton2002training}, score matching \citep{hyvarinen2005estimation}, and ratio matching \citep{lyu2012interpretation}. For inference, gradient-based MCMC methods \citep{welling2011bayesian,grathwohl2021oops} are  widely exploited. Meanwhile,  \cite{pmlr-v97-bian19a,pmlr-v119-sahin20a} propose provable mean-filed inference algorithms for a class of  EBMs with supermodular energies (also called probabilistic log-submodular models). 
In this paper, we train  EBMs under the supervision of OS oracle by running mean-field inference. 

\textbf{Amortized and Copula variational inference.}
Instead of approximating separate variables for each data point, amortized variational inference (VI) \citep{kingma2013auto} assumes that the variational parameters can be predicted by a parameterized function of the data \citep{zhang2018advances}. The idea of amortized VI has been widely applied in deep probabilistic models \citep{hoffman2013stochastic,garnelo2018neural}. Although this procedure would introduce an amortization gap \citep{cremer2018inference}, which refers to the suboptimality of variational parameters, amortized VI enables significant speedups and combines probabilistic modeling with the representational power of deep learning. Copula is the other method to improve the representational power for VI. \citet{tran2015copula} used copula to augment the mean-field VI for better posterior approximation. \citet{suh2016gaussian} adopted Gaussian copula in VI to model the dependency structure of observed data. 
Moreover, \cite{wang2020relaxed} leveraged Gaussian copula to introduce correlation among discrete latent variables, addressing a problem that is closely related to our setting.

\section{Empirical Studies} \label{sec-experiments}

We evaluate the proposed methods on various tasks: product recommendation, set anomaly detection, compound selection, and synthetic experiments. All experiments are repeated five times with different random seeds and their means and standard deviations are reported. The model architectures and training details are deferred to \cref{appendix_exp_details}.
Additional experiments of varying ground set are given in \cref{exp-varying-ground-set}. Comparisons with Set Transformer \citep{lee2019set} are in \cref{appendix_comp_set_trans}.
Ablation studies on hyper-parameter choices (e.g.~MFVI iteration steps, number of MC samples, rank of perturbation, temperature of Gumbel-Softmax) are provided in \cref{sensitive-analysis}. 

\textbf{Evaluations.} We evaluate the methods using the  mean Jaccard coefficient (MJC). Specifically, for each sample $(V, S^*)$, denoting the corresponding model predict as $S^\prime$, the Jaccard coefficient is defined as $ \operatorname{JC}(S, S') = \frac{|S^\prime \cap S|}{|S^\prime \cup S|}$.
Then the MJC metric can be computed by averaging over all samples  in the test set: $\operatorname{MJC} = \frac{1}{|\mathcal{D}_t|} \sum_{(V, S^*) \in \mathcal{D}_t} \operatorname{JC}(S^*, S')$.

\textbf{Baselines.} We compare our solution variants, {\it i.e.,} \oursdiff, \oursind, and \ourscopula to the following three baselines:

    - Random: The expected performance of random guess. This baseline provides an estimate of how difficult the task is.
    Specifically, given a data point $(V, S^*)$, it can be computed as $\mathbb{E}(JC(V, S^*)) = \sum_{k=0}^{|S^*|} \frac{ \binom{|S^*|}{k} \binom{|V| - |S^*|}{|S^*| - k} }{\binom{|V|}{|S^*|}} \frac{k}{2|S^*| - k}$.

    - PGM (\citet{tschiatschek2018differentiable}, see \cref{appendix_pgm}): The probabilistic greedy model, which is permutation invariant but computationally prohibitive. 
    
    - DeepSet (NoSetFn) \citep{zaheer2017deep}: The deepset architecture, satisfying permutation invariant, is the backbone of our models. Its adapted version: $2^V \rightarrow [0,1]^{|V|}$, which serves as the amortized networks in \ours, could work as a baseline since its output stands for the probability of which instance should be selected. We train it with cross entropy loss and sample the subset via $\operatorname{topN}$ rounding. The term ``NoSetFn'' is used to emphasize that this baseline does not learn a set function explicitly, although it can be adapted to our empirical studies.
    
\paragraph{Synthetic Experiments.}  \label{exp-synthetic}
We demonstrate the effectiveness of our models on learning set functions with two synthetic datasets: the two-moons dataset with additional noise of variance $\sigma^2 = 0.1$, and mixture of Gaussians $\frac{1}{2} \mathcal{N}( \boldsymbol{\mu}_0, \boldsymbol{\Sigma}) + \frac{1}{2} \mathcal{N}(\boldsymbol{\mu}_1, \boldsymbol{\Sigma})$, with $\boldsymbol{\mu}_0 = [\frac{1}{\sqrt{2}}, \frac{1}{\sqrt{2}}]^T$, $\boldsymbol{\mu}_1 = -\boldsymbol{\mu}_0$, $\boldsymbol{\Sigma} = \frac{1}{4}\mathbf{I}$. Take the Gaussian mixture as an example, the data generation procedure is as follow: i) select index: $b \sim Bernoulli(\frac{1}{2})$; ii) sample $10$ points from $\mathcal{N}( \boldsymbol{\mu}_b, \boldsymbol{\Sigma})$ to construct $S^*$; iii) sample $90$ points for $V\backslash S^*$ from $\mathcal{N}( \boldsymbol{\mu}_{1-b}, \boldsymbol{\Sigma})$. We collect $1,000$ samples  for training, validation, and test, respectively.

A qualitative result of the \ourscopula is shown in \cref{fig:qualitative_res_synthetic}, where the green dots represent correct model predictions, the red crosses are incorrect model predictions, and the yellow triangles represent the data points in the subset oracle $S^*$ that are missed by the model. One can see  that the most confusing points are located at the intersection of two components. We also illustrate the quantitative results in \cref{tab:sythetic_quantitative_res}. As expected, our methods  
achieve significantly better performance over other methods, with averaged $59.16\%$ and $100.07\%$ improvements compared to PGM on the Two-Moons and Gaussian-Mixture datasets, respectively.

\makeatletter
\newcommand\figcaption{\def\@captype{figure}\caption}
\newcommand\tabcaption{\def\@captype{table}\caption}
\makeatother
\begin{figure}[!t]
    \begin{minipage}[!b]{0.52\linewidth}
        \centering
		{\includegraphics[width=1.3in]{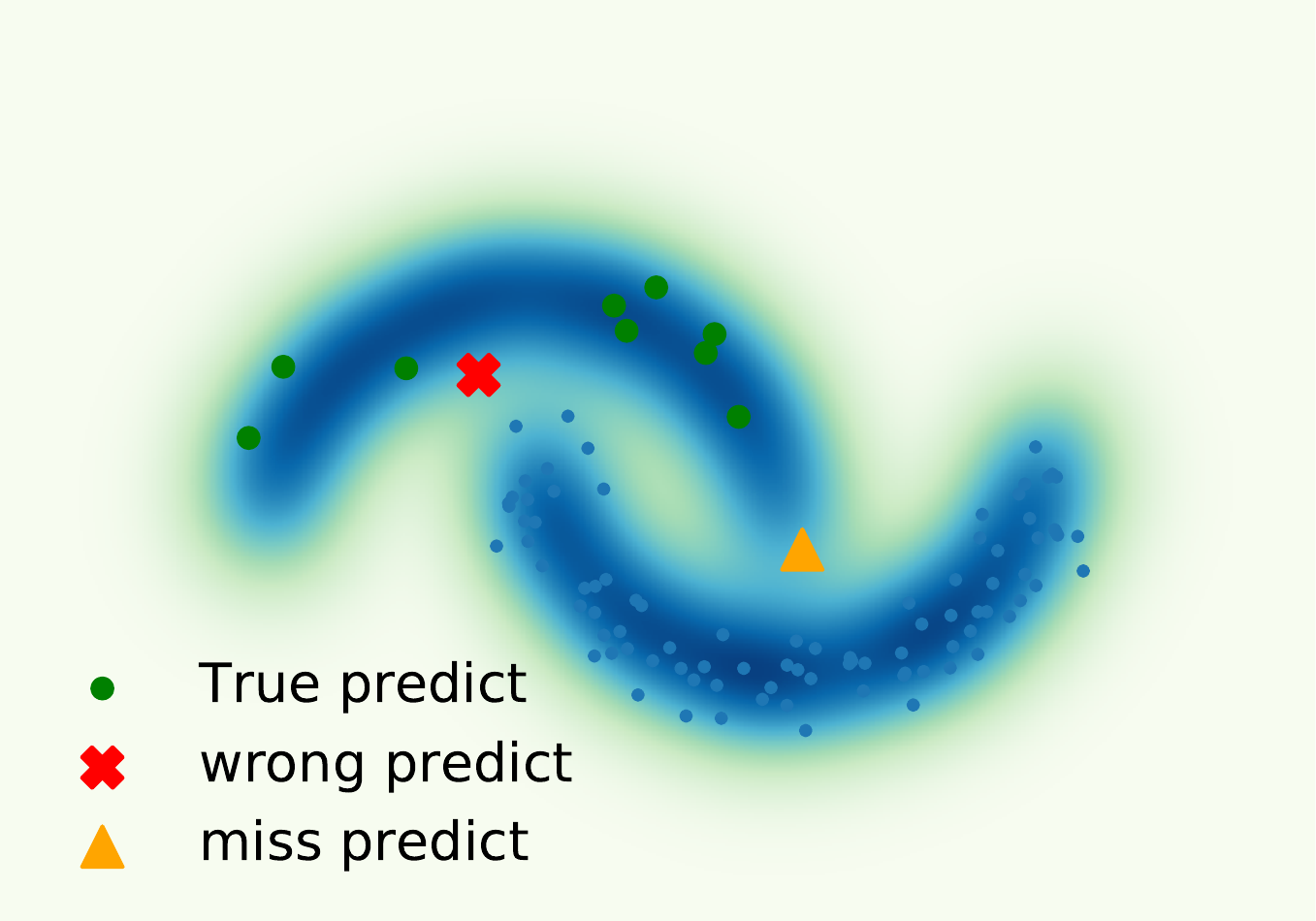}}
		{\includegraphics[width=1.3in]{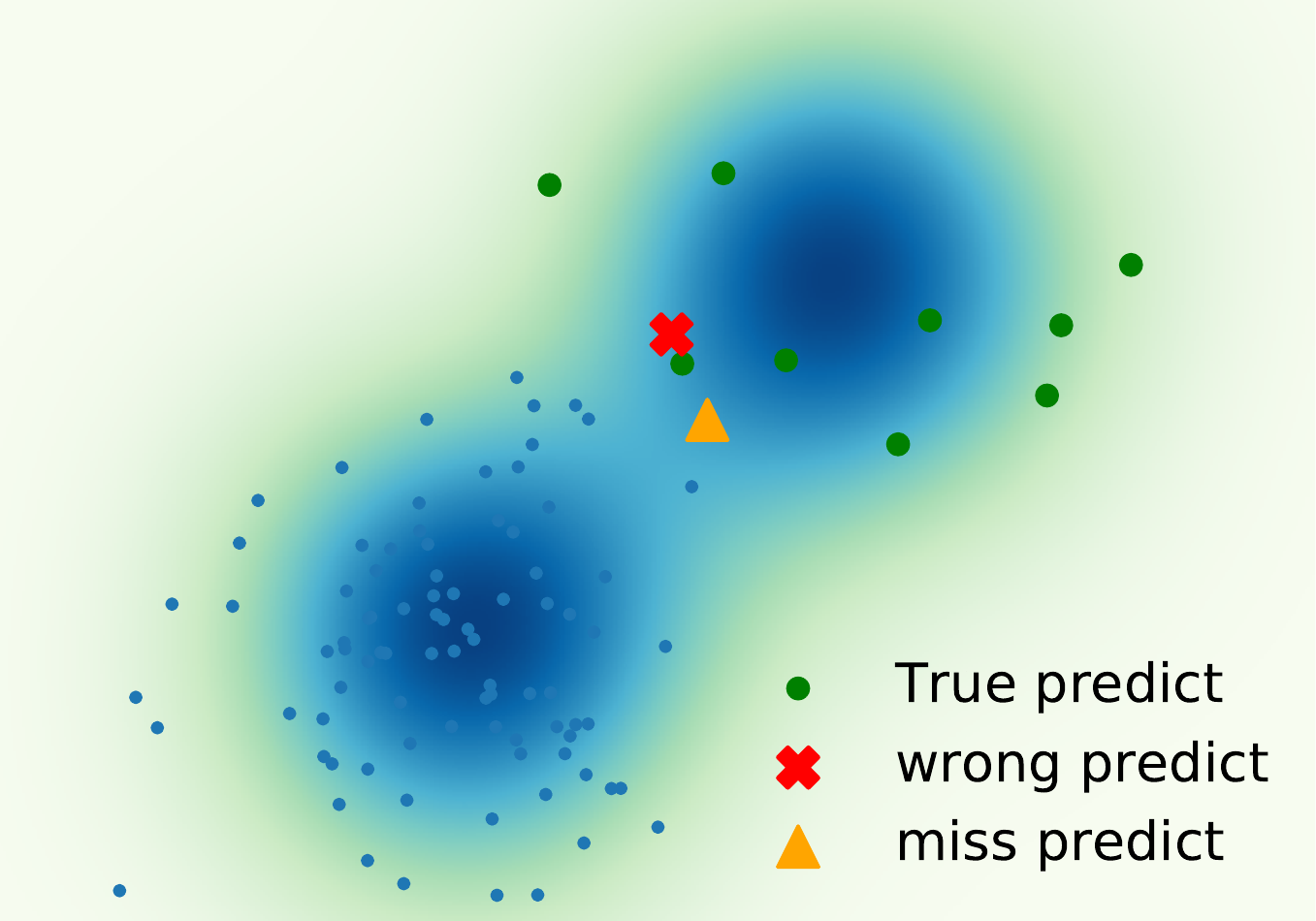}}
	    \figcaption{Visualization of the prediction of \ourscopula on the Two-Moons (left) and Gaussian-Mixture (right) datasets.}
	    \label{fig:qualitative_res_synthetic}
    \end{minipage}
    \begin{minipage}[!b]{0.48\linewidth}
        \centering
        \small
        \vspace{-4mm}
        \tabcaption{Results in the MJC metric on Two-Moons and Gaussian-Mixture datasets.}
        \label{tab:sythetic_quantitative_res}
        \vspace{4mm}
        \setlength{\tabcolsep}{.5mm}{
        \scalebox{0.9}{
        \begin{tabular}{@{}ccccc@{}}\toprule
        Method & Two Moons & Gaussian Mixture  \\
        \midrule
        Random & 0.055 & 0.055 \\
        PGM & 0.360 $\pm$ 0.020 & 0.438 $\pm$ 0.009 \\
        DeepSet (NoSetFn) & 0.472 $\pm$ 0.003 & 0.446 $\pm$ 0.002 \\
        \oursdiff (ours) & 0.584 $\pm$ 0.001 & 0.908 $\pm$ 0.002 \\
        \oursind (ours)  & 0.570 $\pm$ 0.003 & 0.907 $\pm$ 0.002 \\
        \ourscopula (ours)  & \textbf{0.587} $\pm$ \textbf{0.002} & \textbf{0.909} $\pm$ \textbf{0.002} \\
        \bottomrule
        \end{tabular}}}
    \end{minipage}
\end{figure}

\begin{table}[t]
\centering
\small
\caption{Product recommendation results on the Amazon dataset with different categories.}
\label{tab:product_recom}
\hspace*{-0.3cm}
\setlength{\tabcolsep}{1.mm}{
\scalebox{0.92}{
\begin{tabular}{@{}ccccccc@{}}\toprule
Categories & Random & PGM & DeepSet (NoSetFn) & \oursdiff (ours) & \oursind (ours) & \ourscopula (ours) \\
\midrule
Toys & 0.083 & 0.441 $\pm$ 0.004 & 0.429 $\pm$ 0.005 & 0.610 $\pm$ 0.010 & 0.650 $\pm$ 0.015 & \textbf{0.680} $\pm$ \textbf{0.020}\\
Furniture & 0.065 & 0.175 $\pm$ 0.007 & \textbf{0.176} $\pm$ \textbf{0.007} & 0.170 $\pm$ 0.010 & {0.170} $\pm$ {0.011} & 0.172 $\pm$ 0.009\\
Gear & 0.077 & 0.471 $\pm$ 0.004 & 0.381 $\pm$ 0.002 & 0.560 $\pm$ 0.020 & 0.610 $\pm$ 0.020 & \textbf{0.700} $\pm$ \textbf{0.020}\\
Carseats & 0.066 & \textbf{0.230} $\pm$ \textbf{0.010} & 0.210 $\pm$ 0.010 & 0.220 $\pm$ 0.010 & 0.214 $\pm$ 0.007 & 0.210 $\pm$ 0.010\\
Bath & 0.076 & 0.564 $\pm$ 0.008 & 0.424 $\pm$ 0.006 & 0.690 $\pm$ 0.006 & 0.650 $\pm$ 0.020 & \textbf{0.757} $\pm$ \textbf{0.009}\\
Health & 0.076 & 0.449 $\pm$ 0.002 & 0.448 $\pm$ 0.004 & 0.565 $\pm$ 0.009 & 0.630 $\pm$ 0.020 &\textbf{0.700} $\pm$ \textbf{0.020}\\
Diaper & 0.084 & 0.580 $\pm$ 0.009 & 0.457 $\pm$ 0.005 & 0.700 $\pm$ 0.010 & 0.730 $\pm$ 0.020 &\textbf{0.830} $\pm$ \textbf{0.010}\\
Bedding & 0.079 & 0.480 $\pm$ 0.006 & 0.482 $\pm$ 0.008 & 0.641 $\pm$ 0.009 & 0.630 $\pm$ 0.020 &\textbf{0.770} $\pm$ \textbf{0.010}\\
Safety & 0.065 & 0.250 $\pm$ 0.006 & 0.221 $\pm$ 0.004 & 0.200 $\pm$ 0.050 & 0.230 $\pm$ 0.030 &\textbf{0.250} $\pm$ \textbf{0.030}\\
Feeding & 0.093 & 0.560 $\pm$ 0.008 & 0.430 $\pm$ 0.002 & 0.750 $\pm$ 0.010 & 0.696 $\pm$ 0.006 &\textbf{0.810} $\pm$ \textbf{0.007}\\
Apparel & 0.090 & 0.533 $\pm$ 0.005 & 0.507 $\pm$ 0.004 & 0.670 $\pm$ 0.020 & 0.650 $\pm$ 0.020 &\textbf{0.750} $\pm$ \textbf{0.010}\\
Media & 0.094 & 0.441 $\pm$ 0.009 & 0.420 $\pm$ 0.010 & 0.510 $\pm$ 0.010 & 0.551 $\pm$ 0.007 &\textbf{0.570} $\pm$ \textbf{0.010}\\
\bottomrule
\end{tabular}}
}
\end{table}

\paragraph{Product Recommendation.}
In this experiment, we use the Amazon baby registry dataset \citep{gillenwater2014expectation}, which contains numerous subsets of products selected by different customers. Amazon characterizes each product in a baby registry as belonging to a specific category, such as ``toys'' and ``furniture''. Each product is characterized by a short textual description and we represent it as a $768$ dimensional vector using the pre-trained BERT model \citep{devlin2018bert}.

For each category, we generate samples $(V, S^*)$ as follows. Firstly, we filter out those subsets selected by customers whose size is equal to $1$ or larger than $30$. Then we split the remaining  subset collection $\mathcal{S}$ into training, validation and test folds with a $1:1:1$ ratio. Finally for each OS oracle $S^* \in \mathcal{S}$, we randomly sample additional $30 - |S^*|$ products from the same category to construct $V\backslash S^*$. In this way, we construct one data point $(V, S^*)$ for each customer, which reflects this real world scenario: $V$ contains 30 products displayed to the customer, and the customer is interested in  checking $|S^*|$ of them.  
Note that this curation process is different from that of \citep[Section 5.3]{tschiatschek2018differentiable}, which is deviated from the real world scenario (Detailed discussion in \cref{appendix_prod_rec}.).

The performance of all the models on different categories are shown in \cref{tab:product_recom}. Evidently, our models perform favorably to the baselines. Compared with PGM, which learns the set function via a probabilistic greedy algorithm, we can observe that our models, which model the the set functions with energy-based treatments, achieves better results on all settings. Although DeepSet is also permutation invariant, our model still outperforms it by a substantial margin, indicating the superiority of learning the set function explicitly.

\paragraph{Set Anomaly Detection.}
In this experiment, we evaluate our methods on two {image} datasets: the double MNIST \citep{mulitdigitmnist} and the CelebA \citep{liu2015faceattributes}. For each dataset, we randomly split the training, validation, and test set to the size of $10,000$, $1,000$, and $1,000$, respectively.

\textbf{Double MNIST:} The dataset consists of 1000 images for each digit ranging from $00$ to $99$. For each sample $(V, S^*)$, we randomly sample $n \in \{2,\dots,5\}$ images with the same digit to construct the OS oracle $S^*$, and then select $20 - |S^*|$ images with different digits to construct the set $V \backslash S^*$.
\textbf{CelebA:} The CelebA dataset contains $202,599$ images with $40$ attributes. We select two attributes at random and construct the set with the size of $8$. For each ground set $V$, we randomly select $n \in \{2,3\}$ images as the OS oracle $S^*$, in which neither of the two attributes is present. See \cref{fig:double_mnist} and \cref{fig:celeba} in \cref{appendix_set_anomaly_det} for illustrations of sampled data.

From \cref{tab:meta_anomaly}, we see that the variants of our model consistently outperform baseline methods strongly. Furthermore, we observe that by introducing the correlation to the variational distribution, significant performance gains can be obtained, demonstrating the benefits of relaxing the independent assumption by using Gaussian copula. Additional experiments on the other two datasets F-MNIST \citep{xiao2017fashion} and CIFAR-10 \citep{krizhevsky2009learning} are provided in \cref{appendix_sad_fmnist_cifar10}.

\paragraph{Compound Selection in AI-aided Drug Discovery.} \label{sec_compound_selection}

A critical step in drug discovery is to select  compounds with high biological activity \citep{wallach2015atomnet,li2021structure,2022arXiv220109637J}, diversity and satisfactory ADME (absorption, distribution, metabolism, and excretion) properties \citep{gimeno2019light}.  As a result, virtual screening is typically a hierarchical filtering process with several necessary filters, e.g., first choosing the highly active compounds, then selecting diverse subsets from them, and finally excluding compounds that are bad for ADME.
We finally arrive at a compound subset after a series of these steps. 
Given the OS supervision signals,  we can learn to conduct this complicated  selection process in an end to end manner. 
As a result, it will eliminate  the need for intermediate supervision signals, which can be very expensive or impossible to obtain due to pharmacy's personal protection policy. For example, measuring the bioactivity and ADME properties of a compound has to be done in wet labs, and pharmaceutical companies are reluctant to disclose the data. 
Here we simulate the OS oracle of compound selection by applying the \emph{two filters}:  high bioactivity and diversity filters,  based on the following two datasets.

\textbf{PDBBind \citep{liu2015pdb}:} This dataset consists of experimentally measured binding affinities for bio-molecular complexes. We construct our dataset using the ``refined'' subsets therein, which contains $179$ protein-ligand complexes.
\textbf{BindingDB\footnote{We take the curated one from  \href{https://tdcommons.ai/multi\_pred\_tasks/dti/}{https://tdcommons.ai/multi\_pred\_tasks/dti/}}:} It is a public database of measured binding affinities, which consists of $52,273$ drug-targets with small, drug-like molecules.
Instead of providing complexes, here only the  target amino acid sequence and compound SMILES string are provided.

We apply the same filtering process to construct samples $(V, S^*)$ for these two datasets. Specifically, we first randomly select a number of compounds to construct the ground set $V$, whose size is $30$ and $300$ for PDBBind and BindingDB, respectively. Then $\frac{1}{3}$ compounds with the highest bioactivity are filtered out, accompanied by a distance matrix measured by the corresponding fingerprint similarity of molecules. To ensure diversity, the OS oracle $S^*$ is generated by the centers of clusters which are presented by applying the affinity propagation algorithm. We finally obtain the training, validation, and test set with the size of 1,000, 100, and 100, respectively, for both two datasets.
Detailed description is provided in \cref{appendix_compount_selection}.

From \cref{tab:affinity_predict}, one can see that our methods magnificently outperform the random guess. This indicates that the proposed \ours framework has great potential for drug discovery to facilitate the virtual screening task by modeling the complicated hierarchical selection process.
Besides, improvements of \ours can be further observed by comparing with DeepSet, which simply equips the  deepset architecture with cross entropy loss, illustrating the superiority of explicit set function learning and energy-based modeling.
Although comparable results could be achieved by PGM with sequential modeling, which satisfies permutation invariance and differentiability, our models still outperform it. This is partially because our models additionally maintain the other three desiderata of learning set functions, {\it i.e.,} varying ground set, minimum prior, and scalability.
We also conduct a fairly simple task in \cref{appendix_exp_compount_selection_activity}, in which only the bioactivity filter is considered.
To simulate the full selection process, we leave it as important future work due to limited labels.

\begin{table}[!t]
\begin{minipage}[t]{0.48\linewidth}
\centering
\small
\caption{Set anomaly detection results.}
\label{tab:meta_anomaly}
\setlength{\tabcolsep}{1.5mm}{
\scalebox{0.9}{
\begin{tabular}{@{}ccccc@{}}\toprule
Method & Double MNIST & CelebA  \\
\midrule
Random & 0.082 & 0.219 \\
PGM & 0.300 $\pm$ 0.010 & 0.481 $\pm$ 0.006 \\
DeepSet (NoSetFn) & 0.111 $\pm$ 0.003 & 0.390 $\pm$ 0.010 \\
\oursdiff (ours) & \textbf{0.610} $\pm$ \textbf{0.010} & 0.546 $\pm$ 0.008 \\
\oursind (ours) & 0.410 $\pm$ 0.010 & 0.530 $\pm$ 0.010 \\
\ourscopula (ours) & 0.588 $\pm$ 0.007 & \textbf{0.555} $\pm$ \textbf{0.005} \\
\bottomrule
\end{tabular}}}
\end{minipage}
\begin{minipage}[t]{0.48\linewidth}
\centering
\small
\caption{Compound selection results.}
\label{tab:affinity_predict}
\setlength{\tabcolsep}{1.5mm}{
\scalebox{0.9}{
\begin{tabular}{@{}ccccc@{}}\toprule
Method & PDBBind & BindingDB  \\
\midrule
Random & 0.073 & 0.027 \\
PGM & 0.350 $\pm$ 0.009 & 0.176 $\pm$ 0.006\\
DeepSet (NoSetFn) & 0.319 $\pm$ 0.003 & 0.162 $\pm$ 0.007\\
\oursdiff (ours) & \textbf{0.360} $\pm$ \textbf{0.010} & 0.189 $\pm$ 0.002\\
\oursind (ours) & {0.355} $\pm$ {0.005} & \textbf{0.190} $\pm$ \textbf{0.003}\\
\ourscopula (ours) & 0.354 $\pm$ 0.008 & 0.188 $\pm$ 0.003\\
\bottomrule
\end{tabular}}}
\end{minipage}
\end{table}

\section{Discussion and Conclusion} \label{sec-conclusions}
We proposed a simple yet effective framework for set function learning under the OS oracle. By formulating the set probability with energy-based treatments, the resulting model enjoys the virtues of \emph{permutation invariance}, \emph{varying ground set}, and \emph{minimum prior}. A \emph{scalable} training and inference algorithm is further proposed by applying maximum log likelihood principle with the surrogate of mean-field inference. Real-world applications confirm the effectiveness of our approaches.

\textbf{Limitations \& Future Works.} The training objective in \eqref{cross_entropy} does not bound the log-likelihood of EBMs. A more principled discrete EBMs trainer is worth exploring. In addition, the proposed framework has the potential to facilitate learning to select subsets for other applications \citep{iyer2021generalized}, including active learning \citep{kothawade2021similar}, targeted selection of subsets, selection of subsets for robustness \citep{killamsetty2020glister}, and selection of subsets for fairness.  Though we consider learning generic neural  set functions in this work, it is beneficial to consider building useful priors into the neural set function architectures, such as set functions with the diminishing returns prior \citep{DBLP:journals/corr/BilmesB17} and the bounded curvature/submodularity ratio prior \citep{bian2017guarantees}.

{
\bibliography{main}
\bibliographystyle{neurips_2022}
}

\newpage 
\appendix

\begin{center}
\LARGE
\textbf{Appendix for ``Learning Neural Set Functions  Under the Optimal Subset Oracle''}
\end{center}

\etocdepthtag.toc{mtappendix}
\etocsettagdepth{mtchapter}{none}
\etocsettagdepth{mtappendix}{subsection}
{\small \tableofcontents}

\section{Details of the Probabilistic Greedy Model} \label{appendix_pgm}
The probabilistic greedy model (PGM) solves optimization \eqref{prob_set_func_learning} with a differentiable extension of greedy maximization algorithm \citep{tschiatschek2018differentiable}. Specifically, denote the first $j$ chosen elements as $S_j = \{s_1, \dots, s_j\} \subseteq V$, PGM samples the $(j+1)^{\mathrm{th}}$ element from the candidate set $V \backslash S_j$ with the probability proportional to $\mathrm{exp}(F_\theta (s_{j+1} + S_j )/\gamma)$, which raises the probability of the selected elements in the sequence $\pi = \{s_1, s_2, \dots, s_k\}$ as
\begin{align} \label{per_varaint_prob}
    p_\theta (\pi | V) = \prod\limits_{j=0}^{k-1} \frac{\mathrm{exp}(F_\theta (s_{j+1} + S_j) / \gamma)}{\sum_{s\in V\backslash S_j} \mathrm{exp}(F_\theta (s + S_j) / \gamma)},
\end{align}
where $\gamma$ is a temperature parameter, $S_0 = \emptyset$, and $s+S := S \cup \{s\}$. Note that, the computation of $p_\theta (\pi | V)$ depends on the order of sequence $\pi$, which would make the learned parameter $\theta$ sensitive to the sampling order. To alleviate this problem, \citet{tschiatschek2018differentiable} finally construct the set mass function by enumerating all possible permutations
\begin{align} \label{pgm_objective}
    p_\theta (S|V) =  \sum\limits_{\pi \in \Pi^S} p_\theta (\pi | V),
\end{align}
where $\Pi^S$ is the permutation space generated from $S$. After training, the OS oracle $S$ can be sampled via sequential decoupling $p (s_{i+1} | S_i) \propto \exp(F_{\theta} (s_{i+1} + S_i) / \gamma )$. However, maximizing the log likelihood of \eqref{pgm_objective} is prohibitively expensive and unscalable due to the exponential time complexity of enumerating all permutations. Although one can apply Monte Carlo approximation to avoid  that, {\it i.e.,} approximating $\log p_\theta  (S|V) =  \log \sum_{\pi \in \Pi^S} p_\theta (\pi | V)$ with $\log p_\theta (\pi | V), \pi \sim \Pi^S$, such a simple estimator is biased, resulting in a permutation variant model.

\section{Derivations} \label{app-derivations}
\subsection{Derivations of the Maximum Entropy Distribution} \label{proof_maximum_entropy}
The first step to solve problem \eqref{OS_learning_objective} is to construct a proper set mass function $p_\theta (S|V)$ monotonically growing with the utility function $F_\theta (S;V)$. There exits countless ways to construct such a probability mass function, such as $p_\theta (S|V) \propto F_\theta (S;V)$ and the set mass function defined in PGM, {\it i.e.,} \cref{pgm_objective}. Here, one would care about what the most appropriate set mass function should be? Generally we prefer the model to assume nothing about what is unknown. More formally, we should choose the most ``uniform" distribution, which maximizes the Shannon entropy $\mathbb{H}(p) = -\sum_{S \subseteq V} p(S)\log p(S)$. This principle is known as ``noninformative prior" \citep{jeffreys1946invariant}, which has been widely applied in many physical systems \citep{jaynes1957information,jaynes1957information2}.
It turns out that the energy-based model is the only distribution with maximum entropy. More specifically, the following theorem holds:
\begin{theorem}
\label{maximum_entropy_theorem}
Let $\mathcal{P}_\mu \!:=\! \{p(S)\!:\! \mathbb{E}_p [F (S)] \!=\! \mu\}$  be a set of distributions satisfying the expectation constraint $\mathbb{E}_p [F (S)] = \mu$, and $p_\lambda$ have density
\begin{align}
    p_\lambda (S) = \frac{\mathrm{exp}(\lambda F (S))}{Z}, \quad Z := \sum\limits_{S \subseteq V} \mathrm{exp}(\lambda F (S)). \nonumber
\end{align}
If $\mathbb{E}_{p_\lambda} [F (S)]=\mu$, then $p_\lambda$ maximizes the entropy $\mathbb{H}(p)$ over $\mathcal{P}_\mu$; moreover, the distribution $p_\lambda$ is unique.
\end{theorem}

\begin{proof}
The derivation below is adapted from \citet{jaynes1957information} in the context of set function learning, for completeness.
We rewrite the maximum entropy problem in the form of
\begin{align}
    &\operatorname{maximize}\  - \sum_{S \subseteq V} p(S) \log p(S) \nonumber \\
    &\operatorname{subject \ to}\  \sum_{S \subseteq V} p(S)F (S) = \mu, \quad \forall S \subseteq V p(S) \geq 0, \quad \sum_{S \subseteq V} p(S) = 1. \nonumber
\end{align}
Introducing Lagrange multipliers $\alpha(S) > 0$ for the constraint $p(S) > 0$, $\beta \in \mathbb{R}$ for the normalization constraint that $\sum_{S \subseteq V} p(S) = 1$, $\lambda$ for the constraint that $\mathbb{E}_p [F (S)] = \mu$, and , we obtain the following Lagrangian:
\begin{align}
    L(p,\alpha_0, \alpha_1, \beta) 
    &= \sum_{S \subseteq V} p(S)\log p(S) + \beta \left(  \sum_{S \subseteq V} p(S) -1 \right) + \nonumber \\
    &\qquad \lambda \left( \mu - \sum_{S \subseteq V} p(S)F (S) \right) - \sum_{S \subseteq V} \alpha(S) p(S).
\end{align}
Now we take derivatives and obtain
\begin{align}
    \frac{\partial}{\partial p(S)} L(p,\alpha, \beta, \lambda) = 1 + \log p(S) + \beta - \lambda F (S) - \alpha (S).
\end{align}
Since this function is convex in $p$, the minimizing $p$ can be find by setting this equal to zero
\begin{align}
    p(S) = \mathrm{exp}(\lambda F (S) - 1 - \beta + \alpha(S)).
\end{align}
Note that in this setting we always have $p(S) > 0$. By complementary slackness, the constraint $p(S) > 0$ is unnecessary and we have $\alpha(S) = 0$. To satisfy the constraint $\sum_{S \subseteq V} p(S) = 1$, we take $\beta = 1- + \log Z = -1 + \log \sum_{S \subseteq V} \mathrm{exp}(\lambda F (S))$. Then the optimal mass $p$ has the form
\begin{align}
    p_\lambda (S) = \frac{\exp (\lambda F (S))}{\sum_{S \subseteq V} \mathrm{exp}(\lambda F (S))}.
\end{align}
So we reach the form of $p(S)$ we would like to have.

Next we show the distribution $p_\lambda$ is unique. Assume there exists any other distribution $p \in \mathcal{P}_\mu$, such that $p = \argmax_p \mathbb{H}(p)$. In this case, we have
\begin{align}
    \mathbb{H}(p) 
    &= -\sum\limits_{S \subseteq V} p(S) \log p(S) =  -\sum\limits_{S \subseteq V} p(S) \log \frac{p(S)}{p_\lambda (S)} - \sum\limits_{S \subseteq V} p(S) \log p_\lambda (S) \nonumber \\
    &= -\mathbb{KL}(p || p_\lambda) - \sum\limits_{S \subseteq V} p(S) \left( \lambda F (S) - Z \right) \nonumber \\
    &= -\mathbb{KL}(p || p_\lambda) - \sum\limits_{S \subseteq V} p_\lambda (S) \left( \lambda F (S) - Z \right) \nonumber \\
    &= -\mathbb{KL}(p || p_\lambda) + \mathbb{H}(p_\lambda). \nonumber
\end{align}
As $\mathbb{KL}(p || p_\lambda) \geq 0$ unless $p = p_\lambda$, we have shown that $p_\lambda$ is the unique distribution maximizing the entropy, as desired.
\end{proof}

\textbf{Discussion.}
\cref{maximum_entropy_theorem} shows that EBM is the maximum entropy distribution, which verifies the assertion that energy-based treatments of set function enjoy the \emph{minimum prior} property. It should be noted that the model proposed by \citet{tschiatschek2018differentiable} violates this requirement. They used sequencial modeling to construct $p(S)$ (see \eqref{per_varaint_prob} and \eqref{pgm_objective}).  Although this approach simplifies the sampling process, it introduces undesirable inductive bias.

\subsection{Derivations of the Fixed Point Iteration} \label{proof_elbo}
In this section, we give the detailed derivation for the fixed point iteration (FPI) of MFVI:
\begin{align} \label{app-fpi-formula}
    {\psi}^{(k+1)}_i \leftarrow (1 + \mathrm{exp}(- \nabla_{\psi^{(k)}_i} f_{\mathrm{mt}}^{F_\theta} (\boldsymbol{\psi}^{(k)})))^{-1}.
\end{align}
First, recall that we want to maximize the ELBO:
\begin{align}
    \max_{\boldsymbol{\psi}} \underbrace{\sum\limits_{S \subseteq V} F_\theta (S) \prod_{i \in S} \psi_i \prod_{i \not\in S} (1 - \psi_i)}_{f_{\mathrm{mt}}^{F_\theta} (\boldsymbol{\psi}) } - \underbrace{\sum\limits_{i=1}^{|V|} \left[ \psi_i \log \psi_i + (1-\psi_i)\log (1-\psi_i)) \right]}_{- \mathbb{H}(q(S;\boldsymbol{\psi})) }.
\end{align}
The formula \eqref{app-fpi-formula} is obtained by setting the partial derivative {\it w.r.t.} coordinate $i$ of ELBO to be $0$:
\begin{align}
\nabla_{\psi_i} f_{\mathrm{mt}}^{F_\theta} (\boldsymbol{\psi}) + \nabla_{\psi_i} \mathbb{H}(q(S;\boldsymbol{\psi})) = \nabla_{\psi_i} f_{\mathrm{mt}}^{F_\theta} (\boldsymbol{\psi}) + \log \frac{1 - \psi_i}{\psi_i} = 0, \nonumber
\end{align}
which implies
\begin{align}
    \psi_i = (1 + \exp (-\nabla_{\psi_i} f_{\mathrm{mt}}^{F_\theta} (\boldsymbol{\psi}) ) )^{-1}. \nonumber
\end{align}
This is exactly the formula of FPI used in the mean-field variational inference algorithm. Note that the FPI actually corresponds to the gradient ascent with an adaptive step size vector $\boldsymbol{\alpha}$ as
\begin{align}
    \alpha_i = \frac{\sigma(\nabla_{\psi_i} f_{\text{mt}}^{F_\theta}(\boldsymbol{\psi})) - \psi_i }{\nabla_{\psi_i} f_{\text{mt}}^{F_\theta} (\boldsymbol{\psi}) +\log[ {(1-{\psi}_i)}/{{\psi}_i}]}, \nonumber
\end{align}
where $\sigma(x) = (1 + \exp(-x))^{-1}$ denotes the sigmoid function. To verify this, we have
\begin{align}
    {\psi}^{(k+1)}_i 
    &= {\psi}^{(k)}_i + \alpha_i^{(k)} \left( \nabla_{\psi_i^{(k)}} f_{\mathrm{mt}}^{F_\theta} (\boldsymbol{\psi}^{(k)}) + \nabla_{\psi_i^{(k)}} \mathbb{H}(q(S;\boldsymbol{\psi}^{(k)})) \right) \nonumber \\
    &= {\psi}^{(k)}_i + \frac{\sigma(\nabla_{\psi_i} f_{\text{mt}}^{F_\theta}(\boldsymbol{\psi}^{(k)})) - \psi_i^{(k)} }{\nabla_{\psi_i^{(k)}} f_{\text{mt}}^{F_\theta} (\boldsymbol{\psi}^{(k)}) +\log[ {(1-{\psi}_i^{(k)})}/{{\psi}_i^{(k)}}]} \left( \nabla_{\psi_i^{(k)}} f_{\text{mt}}^{F_\theta} (\boldsymbol{\psi}^{(k)}) +\log \frac{(1-{\psi}_i^{(k)})}{{\psi}_i^{(k)}} \right) \nonumber \\
    &= (1 + \exp (-\nabla_{\psi_i^{(k)}} f_{\mathrm{mt}}^{F_\theta} (\boldsymbol{\psi}^{(k)}) ) )^{-1}. \nonumber
\end{align}
The connection to gradient ascent further confirms the soundness of our FPI algorithm.

\subsection{Derivations of the Gradient of Multilinear Extension} \label{proof_gradient_mt}
In this section, we prove that the gradient of multilinear extension can be estimated using Monte Carlo sampling. Specifically we have
\begin{align} \label{appendix_gradient_mt}
    \nabla_{\psi_i} f_{\mathrm{mt}}^{F_\theta}
    &= \nabla_{\psi_i} \sum\limits_{S \subseteq V} F_\theta (S) \prod_{i \in S}\psi_i \prod_{i \not\in S}(1-\psi_i) \nonumber \\
    &= \mathbb{E}_{q(S;(\boldsymbol{\psi} | \psi_i \leftarrow 1))}[F_\theta (S)] - \mathbb{E}_{q(S;(\boldsymbol{\psi} | \psi_i \leftarrow 0))}[F_\theta (S)] \nonumber \\
    &= \sum\limits_{S \subseteq V, i \in S} F_\theta (S) \prod\limits_{j\in S\backslash\{i\}} \psi_j \prod\limits_{j^\prime \not\in S} (1 - \psi_{j^\prime}) - \sum\limits_{S \subseteq V\backslash \{i\}} F_\theta (S) \prod\limits_{j\in S} \psi_j \prod\limits_{j^\prime \in S, j^\prime \not= i} (1 - \psi_{j^\prime}) \nonumber \\
    &= \sum\limits_{S \subseteq V\backslash \{i\}} [F_\theta (S + i) - F_\theta (S)] \prod\limits_{j\in S} \psi_j \prod\limits_{j^\prime \in V\backslash S \backslash\{i\} } (1 - \psi_{j^\prime}) \nonumber \\
    &= \mathbb{E}_{q(S;(\boldsymbol{\psi} | \psi_i \leftarrow 0))} \left[ F_\theta (S + i) - F_\theta (S) \right]. 
\end{align}
\textbf{Discussion.} The Monte Carlo (MC) approximation of $\nabla_{\psi_i} f_{\mathrm{mt}}^{F_\theta}$ is unbiased. Thereby, although exactly calculating \eqref{appendix_gradient_mt} has exponential time complexity, we can apply MC sampling to approximate it in a polynomial time, resulting a scalable training algorithm. It is worth to note that the MC approximation used in PGM (see \eqref{pgm_objective}) is biased. That is they approximate $\log p_\theta  (S|V) = \log \sum_{\pi \in \Pi^S} p_\theta (\pi | V)$ with $\log p_\theta (\pi | V), \pi \sim \Pi^S$. Although such a biased approximation can be computed in polynomial time, they undesirably introduce permutation variance.

\section{Low-Rank Perturbation for the Covariance Matrix} \label{appendix_lower_rank_per}
In the construction of Gaussian copula $C_{\boldsymbol{\Sigma}}$, we require a positive semi-definite matrix $\boldsymbol{\Sigma} \in \mathbb{R}^{|V| \times |V|}$, whose elements are generally modeled as the output of neural networks. Thereby, if the size of ground size $V$ is large, the number of neural network outputs will be prohibitively large. Meanwhile, based on the definition of set, covariance matrix $\boldsymbol{\Sigma}$ is further required to satisfy \emph{permutation equivariance}. To remedy this issue, we propose to employ a more efficient strategy, namely \emph{Lower-Rank Perturbation}, which restricts the covariance matrix to the form
\begin{align} \label{lower-rank-perturb}
    \boldsymbol{\Sigma} = \boldsymbol{D} + \boldsymbol{P}\boldsymbol{P}^T,
\end{align}
where $\boldsymbol{D} \in \mathbb{R}_{+}^{|V| \times |V|}$ is a diagonal matrix with positive entries and $\boldsymbol{P} = [\boldsymbol{p}_1, \boldsymbol{p}_2, \dots, \boldsymbol{p}_v]$ is a lower-rank perturbation matrix with $\boldsymbol{p}_i \in \mathbb{R}^{|V|}$ and $v \ll |V|$. In this way, the number of neural network outputs can be dramatically reduced from $|V|^2$ to $v|V|$. Another benefit of constructing $\boldsymbol{\Sigma}$ in this way is that, it is convenient to employ the DeepSet architecture in \eqref{deepse_equivariance} to output $\boldsymbol{D}$ and $\boldsymbol{p}_i$ for $i=1,\dots,v$, such that they are \emph{permutation equivariant}, and the resulting covariance matrix $\boldsymbol{\Sigma} = \boldsymbol{D} + \boldsymbol{P}\boldsymbol{P}^T$ is also \emph{permutation equivariant}. Moreover, the lower-rank perturbation trick permits us to avoid using Cholesky decomposition to sample a Gaussian noise with covariance $\boldsymbol{\Sigma}$, which is prohibitively expensive. Specifically, the Gaussian noise $\boldsymbol{g} \sim \mathcal{N}(\boldsymbol{0},\boldsymbol{\Sigma})$ can be reparameterized as
\begin{align}
\boldsymbol{g} = \boldsymbol{D}^{1/2} \cdot \boldsymbol{\epsilon}_1 + \boldsymbol{P} \cdot \boldsymbol{\epsilon}_2,
\end{align}
where $\boldsymbol{\epsilon}_1 \sim \mathcal{N}(\boldsymbol{0}, \boldsymbol{I}_{|V|})$ and $\boldsymbol{\epsilon}_2 \sim \mathcal{N}(\boldsymbol{0}, \boldsymbol{I}_{v})$. In this way, the sampling complexity can be reduced from $\mathcal{O}(|V|^3)$ to $\mathcal{O}(v^2 |V|)$.

\section{Detailed  Pseudo Code of EquiVSet Algorithms} \label{appendix_pseudo_code_4_EquivSet_copula}
We provide the pseudo-code for \ours in \cref{alg:complete_equivset}.
The training procedure consists of two steps: i) train  $q_\phi$ with fixed $\theta$; ii) train $p_\theta$ under the guidance of $q_\phi$.
Specifically, to train $q_\phi$, we first fix the parameter $\theta$ of the set function and then optimize $\phi$ by maximizing the \textsc{ELBO} in \eqref{elbo}. To train $p_\theta$, we first initialize the variational parameter $\boldsymbol{\psi}$ via EquiNet and then run $K$ steps mean-field iteration to make $\boldsymbol{\psi}$ dependent with $\theta$. Finally, the parameter $\theta$ can be optimized by minimizing the cross entropy in \eqref{cross_entropy}. Note that, we set $K$ as $1$ in our experiments. 

\begin{algorithm}[H]
\caption{EquiVSet (complete version)}
\label{alg:complete_equivset}
\small
\textbf{Input}: $\{V_i, S^*_i\}_{i=1}^N$: training dataset; $\eta$: learning rate; $K:$ number of mean-field iteration step; $m:$ number of Monte Carlo approximations; $v$: rank of perturbation; $\tau$: temperature for Gumbel-Softmax \\
\textbf{Output}: Optimal parameters ($\theta, \phi$)
\begin{algorithmic}[1]
    \STATE $\theta, \phi \leftarrow$ Initialize parameter 
	\REPEAT
    	\STATE Sample training data point \\ $(V, S^*) \sim \{V_i, S^*_i\}_{i=1}^N$ 
    	\STATE Obtain variational parameter $\boldsymbol{\psi}$ via EquiNet \\ $\boldsymbol{\psi} \leftarrow \operatorname{EquiNet}(V;\phi)$
    	\STATE Sample $m$ subsets via $\operatorname{CopulaBernoulli}(V,v,\tau)$ or $\operatorname{IndBernoulli}(V,\tau)$ \tikzmark{top1}  \quad\  \tikzmark{right} \\ $S_n \sim q(S;\boldsymbol{\psi}), n = 1, 2, \dots, m$ 
    	\STATE Update the parameter $\phi$ by maximizing \textsc{ELBO} in \eqref{elbo}  \\ $\phi \leftarrow \phi + \eta \nabla_\phi \left( \frac{1}{m} \sum\limits_{n=1}^m F_\theta (S_n) - \sum\limits_{i = 1}^{|V|} [\psi_i \log \psi_i + (1 - \psi_i )\log (1 - \psi_i)] \right)$ \tikzmark{bottom1}
    	\STATE Initialize parameter $\boldsymbol{\psi}$ via EquiNet \\ $\boldsymbol{\psi}^{(0)} \leftarrow \operatorname{EquiNet(V;\operatorname{stop\_gradient}(\phi))}$
    	\FOR{$k \leftarrow 1, \dots, K$ \tikzmark{top2}} 
    	\FOR{$i \leftarrow 1, \dots, |V|$ in parallel} 
	        \STATE Sample $m$ subsets via the variational distribution\footnotemark \\ $S_n \sim q(S;\boldsymbol{\psi}^{(0)} | \psi_i^{(0)} \leftarrow 0), n = 1, 2, \dots, m$
	        \STATE Update the variational parameter $\boldsymbol{\psi}$ \\  $\psi^*_i \leftarrow \sigma \left( \frac{1}{m} \sum\limits_{n=1}^m \left[ F_\theta (S_n + i) - F_\theta (S_n ) \right] \right)$
    	\ENDFOR
    	\ENDFOR \tikzmark{bottom2}
    	\STATE Update the parameter $\theta$ by minimizing the cross entropy loss in \eqref{cross_entropy}  \tikzmark{top3} \\ $\theta \leftarrow \theta - \eta  \nabla_\theta \left( - \sum\limits_{i \in S^*} \log \psi_i^* - \sum\limits_{i \in V \backslash S^*} \log (1 - \psi_i^*) \right)$  \tikzmark{bottom3}
	\UNTIL convergence of parameters ($\boldsymbol{\theta}, \boldsymbol{\phi}$) 
\end{algorithmic}
\AddNote{2.5cm}{top1}{bottom1}{right}{\textbf{Optimize $\phi$}}
\AddNote{4.0cm}{top2}{bottom2}{right}{\textbf{Mean-field Iteration $\operatorname{MFVI}(\boldsymbol{\psi}^{(0)}, V, K)$}}
\AddNote{2.5cm}{top3}{bottom3}{right}{\textbf{Optimize $\theta$}}
\end{algorithm}
\footnotetext{Here we apply mean-field variational inference, which means the variational distribution $q$ is an independent Bernoulli distribution.}

\begin{algorithm}[H]
\caption{$\operatorname{IndBernoulli}(V, \tau)$}
\label{alg:indbernoulli}
\small
\textbf{Input}: $V$: ground set; $\tau$: temperature for Gumbel-Softmax \\
\textbf{Output}:  Sampled subset $\boldsymbol{s}$
\begin{algorithmic}[1]
    \STATE Obtain location parameter $\boldsymbol{\psi}$ via EquiNet: $\boldsymbol{\psi} \leftarrow \operatorname{EquiNet}(V;\phi)$
    \STATE Draw uniform noise: $u_i \sim \mathcal{U}({0}, 1), i = 1, \dots, |V|$
    \STATE Apply Gumbel-Softmax trick: $\Tilde{s}_i = \sigma \left( \frac{1}{\tau} \left(\log  \frac{\psi_i}{1 - \psi_i} + \log \frac{u_i}{1-u_i} \right) \right), i=1,\dots,|V|$
    \STATE Apply Straight-Through estimator: $\boldsymbol{s} = \operatorname{stop\_gradient}(\mathbb{I}(\Tilde{\boldsymbol{s}} \geq \boldsymbol{\epsilon}) - \Tilde{\boldsymbol{s}}) + \Tilde{\boldsymbol{s}}$, $\boldsymbol{\epsilon} \sim \mathcal{U}(\boldsymbol{0}, \boldsymbol{I})$
\end{algorithmic}
\end{algorithm}

\begin{algorithm}[H]
\caption{$\operatorname{CopulaBernoulli}(V,v, \tau)$}
\label{alg:copulabernoulli}
\small
\textbf{Input}: $V$: ground set; $v$: rank of perturbation; $\tau$: temperature for Gumbel-Softmax \\
\textbf{Output}:  Sampled subset $\boldsymbol{s}$
\begin{algorithmic}[1]
    \STATE Obtain location parameter $\boldsymbol{\psi}$ via EquiNet: $\boldsymbol{\psi} \leftarrow \operatorname{EquiNet}(V;\phi)$
    \STATE Draw Gaussian noise: \\ \COMMENT{{ \footnotesize In the following, $\boldsymbol{D}$ is a diagonal matrix and $\boldsymbol{P}$ is the lower-rank perturbation matrix. }} \\ $\boldsymbol{g} = \boldsymbol{D}^{1/2} \cdot \boldsymbol{\epsilon}_1 + \boldsymbol{P} \cdot \boldsymbol{\epsilon}_2$, $\boldsymbol{P} = [\boldsymbol{p}_1, \boldsymbol{p}_2, \dots, \boldsymbol{p}_v]$, $\boldsymbol{\epsilon}_1 \sim \mathcal{N}(\boldsymbol{0}, \boldsymbol{I}_{|V|})$ and $\boldsymbol{\epsilon}_2 \sim \mathcal{N}(\boldsymbol{0}, \boldsymbol{I}_{v})$
    \STATE Apply element-wise Gaussian CDF: $\boldsymbol{u} = \boldsymbol{\Phi}_{\operatorname{diag}(\boldsymbol{D} + \boldsymbol{P}\boldsymbol{P}^T)}(\boldsymbol{g})$
    \STATE Apply Gumbel-Softmax trick: $\Tilde{s}_i = \sigma \left( \frac{1}{\tau} \left( \log \frac{\psi_i}{1 - \psi_i} + \log \frac{u_i}{1-u_i} \right) \right), i=1,\dots,|V|$
    \STATE Apply Straight-Through estimator: $\boldsymbol{s} = \operatorname{stop\_gradient}(\mathbb{I}(\Tilde{\boldsymbol{s}} \geq \boldsymbol{\epsilon}) - \Tilde{\boldsymbol{s}}) + \Tilde{\boldsymbol{s}}$,  $\boldsymbol{\epsilon} \sim \mathcal{U}(\boldsymbol{0}, \boldsymbol{I})$
\end{algorithmic}
\end{algorithm}

\section{Experimental Details} \label{appendix_exp_details}
\subsection{The Architecture of EquiVSet}
In this section, we provide a detail architecture description of \ourscopula. \ourscopula consists of two different components that are implemented as neural networks: (i) the set function which is permutation invaraint and (ii) the recognition network which is permutation equivariant. We employ the DeepSet architecture to implement these two components, with the detailed architectures are given in \cref{tab:equivset_architecture}.

\newcommand{\relu}{\mathrm{ReLU}}
\newcommand{\sigmoid}{\mathrm{sigmoid}}
\newcommand{\softplus}{\mathrm{softplus}}
\newcommand{\tanhact}{\mathrm{tanh}}
\newcommand{\fc}{\mathrm{FC}}
\newcommand{\conv}{\mathrm{Conv}}
\newcommand{\sab}{\mathrm{SAB}}
\newcommand{\initlayer}{\mathrm{InitLayer}}
\newcommand{\reducesum}{\mathrm{SumPooling}}
\begin{table}[ht]
\small
\centering
\caption{Detailed architectures of  \ourscopula.}
\vspace{5pt}
\begin{tabular}{@{}ccc@{}}\toprule
\multicolumn{1}{c}{\textbf{Set Function}} & \multicolumn{2}{c}{\textbf{Recognition Network}}\\
\cmidrule(lr){1-1} \cmidrule(lr){2-3}
$\initlayer(S,256)$ & \multicolumn{2}{c}{$\initlayer(V,256)$} \\
$\reducesum$ & \multicolumn{2}{c}{$\fc(256, 500, \relu)$} \\
$\fc(256, 500, \relu)$ & \multicolumn{2}{c}{$\fc(500, 500, \relu)$} \\
$\fc(500, 500, \relu)$ & \multirow{2}*{$\boldsymbol{\psi} = \fc(500, 1, \sigmoid)$} & $\boldsymbol{D} = \operatorname{diag}(\fc(500, 1, \softplus))$ \\
$\fc(256, 1, -)$ &  & $\boldsymbol{P} = [\fc(500, 1, \tanhact)]^v$ \\
\bottomrule
\end{tabular}
\label{tab:equivset_architecture}
\end{table}

In \cref{tab:equivset_architecture}, $\initlayer(S,d)$ denotes the set transformation function, which encodes the set objects into vector representations. $\fc(d, h, f)$ denotes the fully-connected layer with activation function $f$. $\operatorname{diag}(\boldsymbol{v})$ is a diagonal matrix with the elements of diagonal being vector $\boldsymbol{v}$. $[\boldsymbol{p}]^v = [\boldsymbol{p}_1 , \dots, \boldsymbol{p}_v]$ denotes a matrix with $\boldsymbol{p}$ representing a column perturbation vector. 
Note that we also propose two variant methods, {\it i.e.,} \oursdiff and \oursind. For \oursdiff, we apply the same architecture of the the set function in \cref{tab:equivset_architecture}. We also exploit the same architecture for \oursind, but discarding the copula components, {\it i.e.,} $\boldsymbol{D}$ and $\boldsymbol{P}$.
In all experiments, we implement our models following the same architecture with the difference being that we apply various $\initlayer$ to different datasets. The architectures of $\initlayer$ for different datasets are depicted below. 

\textbf{Synthetic datasets.}
The synthetic datasets consist of the Tow-Moons and Gaussian-Mixture datasets. Each instance of the set is a two-dimensional vector, which represents the corresponding Cartesian coordinates. In this dataset, the $\initlayer$ is a one-layer feed-forward neural network $\fc(2, 256, -)$.

\textbf{Amazon Baby Registry.}
The Amazon baby registry dataset consists of a set of products that are characterized by a short textual description. We transform them into vector representations using the pre-trained BERT module \citep{devlin2018bert}. Thereby, each instance of the set is a $768$ dimensional feature vector. The $\initlayer$ is modelled as $\fc(768, 256, -)$.

\textbf{Double MNIST.} The double MNIST dataset consists of different digit images ranging from $00$ to $99$. Each image has the shape of $(64,64)$ and we reshape it into $(4096,)$. Therefore, the $\initlayer$ is designed as $\fc(4096, 256, -)$.

\textbf{CelebA.} The CelebA dataset contains $202,599$ number of face images. Each image is in the shape of $(3,64,64)$. We employ convolutional neural networks as $\initlayer$. Specifically, the architecture of $\initlayer$ is
\begin{align}
    \operatorname{ModuleList}([\conv(32,3,2,\relu), \conv(64,4,2,\relu), \nonumber \\ \conv(128,5,2,\relu), \operatorname{MaxPooling}, \fc(128, 256, -)]), \nonumber
\end{align}
where $\conv(d,k,s,f)$ is a convolutional layer with $d$ output channels, $k$ kernel size, $s$ stride size, and activation function $f$.

\textbf{PDBBind.}
The PDBBind database consists of experimentally measured binding affinities for biomolecular complexes \citep{liu2015pdb}. It provides detailed 3D Cartesian coordinates of both ligands and their target proteins derived from experimental ({\it e.g.}, X-ray crystallography) measurements. The
atomic convolutional network (ACNN) \citep{gomes2017atomic} provides meaningful vector features for complexes by constructing nearest neighbor graphs based on the 3D coordinates of atoms and predicting binding free energies. In this work, we apply the output of last second layer of the ACNN model followed by feed-forward neural networks to obtain the representations of complexes. More formally, the $\initlayer$ is defined as
\begin{align}
    \operatorname{ModuleList}([\mathrm{ACNN}[:-1], \fc(1922, 2048, \relu), \fc(2048, 256, -)]), \nonumber
\end{align}
where $\mathrm{ACNN}[:-1]$ denotes the ACNN module without the last prediction layer, whose output dimensionality is $1922$.

\textbf{BindingDB.}
The BindingBD dataset containts $52,273$ drug-target pairs. We exploit the DeepDTA model \citep{ozturk2018deepdta} to encode drug-target pairs as vector representations. Specifically, the DeepDTA model first represents the drug compound and target protein as sequences of one-hot vectors and encodes them as feature vectors using convolutional neural networks. The detailed architecture of $\initlayer$ used in this dataset is demonstrated in \cref{tab:bindingdb_initlayer_architecture}.
\begin{table}[ht]
\small
\centering
\caption{Detailed architectures of InitLayer in the BindingDB dataset.}
\vspace{5pt}
\begin{tabular}{@{}cc@{}}\toprule
\textbf{Drug} & \textbf{Target} \\
\cmidrule(lr){1-1} \cmidrule(lr){2-2}
$\conv(32,4,1,\relu)$ & $\conv(32,4,1,\relu)$ \\
$\conv(64,6,1,\relu)$ & $\conv(64,8,1,\relu)$ \\
$\conv(96,8,1,\relu)$ & $\conv(96,12,1,\relu)$ \\
$\mathrm{MaxPooling}$ & $\mathrm{MaxPooling}$ \\
$\fc(96, 256, \relu)$ & $\fc(96, 256, \relu)$ \\
\multicolumn{2}{c}{$\mathrm{Concat}$} \\
\multicolumn{2}{c}{$\fc(512, 256, -)$} \\
\bottomrule
\end{tabular}
\label{tab:bindingdb_initlayer_architecture}
\end{table}

\subsection{Implementation Details}
Here we provide a detailed description of the hyperparameters setup for our model \ours and its variants. \ours contains four important hyperprameters: the number of Monte Carlo sampling $m$ and mean-field iteration steps $K$ in \cref{alg:MFI}, and the rank of lower-rank perturbation $v$ in \eqref{lower-rank-perturb}. We set $m=5,v=5$ throughout the experiments. For the mean-field iteration steps $K$, we set it as $5$ for the variant model \oursdiff, and $1$ for \oursind and \ourscopula. It is noted that the hyperparameters above are empirically set, and we have detail sensitivity analysis in \cref{sensitive-analysis}. The proposed models are trained using the Adam optimizer \citep{kingma2014adam} with a fixed learning rate $1e-4$ and weight decay rate $1e-5$. We choose the batch size from $\{4,8,16,32,64,128\}$, since the model sizes for various datasets are different and we choose the largest batch size to enable it can be trained on a single Tesla V100-SXM2-32GB GPU.

We apply the early stopping strategy to train the models, including the baselines and our models. That is if the performances are not improved in continuous $6$ epochs, we early stop the training process. Each dataset is trained for maximum $100$ epochs. After each epoch, we validate the model and save the model with the best performance on the validation set. After training, we evaluate the performance of saved models on the test set. We repeat all experiments $5$ times with different random seeds and the average performance metrics and their standard deviations are reported as the final performances.

\subsection{Baselines}
Throughout the experiments, we compared our models with three conventional approaches: random guess, probabilistic greedy model (PGM) \citep{tschiatschek2018differentiable} and DeepSet \citep{zaheer2017deep}. Further descriptions of the benchmarks and implementation details are as follows.

    - Random: We report the expected value of the Jaccard coefficient (JC) of random guess. This baseline provides an estimate of how
    difficult the task is. Specifically, given a data point $(V, S^*)$, it can be computed as $\mathbb{E}(JC(V, S^*)) = \sum_{k=0}^{|S^*|} \frac{ \binom{|S^*|}{k} \binom{|V| - |S^*|}{|S^*| - k} }{\binom{|V|}{|S^*|}} \frac{k}{2|S^*| - k}$.
    
    -  PGM \citep{tschiatschek2018differentiable}: PGM is the most relevant method to the set functions learning under the OS oracle, that solves optimization \eqref{prob_set_func_learning} using greedy maximization algorithm with the virtues of differentiability and permutation invariance. We employ the same architecture defined in \cref{tab:equivset_architecture} to model the set function $F_\theta(S)$ in \eqref{per_varaint_prob}. The temperature parameter $\gamma$ is empirically set as $1$. We use Monte Carlo sampling to estimate \eqref{pgm_objective}. That is we randomly sample one permutation $\pi \sim \Pi^S$ and use $p_\theta (\pi | V)$ to approximate $p_\theta (S|V)$. The model is trained using the Adam optimizer, with batch size choosing from $\{4,8,16,32,64,128\}$, fixed learning rate $1e-4$, and fixed weight decay rate $1e-5$.
    
    -  DeepSet (NoSetFn) \citep{zaheer2017deep}: DeepSet is a neural-network-based architecture that satisfies permutation invariance and varying ground sets. Although the DeepSet architecture can be employed here to sample the optimal subset oracle, it does not learn the set functions explicitly. We exploit the same architecture of $F_\theta(S)$ in \cref{tab:equivset_architecture}, but drop the $\operatorname{SumPooling}$ operator to ensure the dimensionality of output is $|V|$. This baseline is trained by minimizing the objective in \eqref{cross_entropy} using the Adam optimizer with batch size choosing from $\{4,8,16,32,64,128\}$, fixed learning rate $1e-4$, and fixed weight decay rate $1e-5$. 

\subsection{Assumptions on the Underlying Data Generative Distribution of the OS Oracle}

In this section, we discuss the assumptions made about the data distribution for better understanding the set functions learning under optimal subset (OS) oracle.  Generally speaking, for any scenario with the output being a subset $S$ of the given ground set $V$ of the input, the proposed approach could be applied to predict the subset $S$ of the given ground set $V$. The only loose assumption is that the optimal subset oracle $S^*$ of a given ground set $V$ is generated by some underlying distribution formulated via a utility function that maximizes the utility value of OS oracle (see \eqref{prob_set_func_learning} in the main text). We further assume the utility function could be parameterized by a deep neural network, thanks to the universal approximation theorem \citep{leshno1993multilayer}.

This assumption is very weak and generally makes sense in practice. We also apply this assumption to the datasets used in the experiments. Specifically, in the product recommendation (\cref{appendix_prod_rec}), $V$ is the set of recommended products, and $S^*$ is the one the customer buys (or adds to the cart). Undoubtedly, the underlying generative distribution, or say the utility function is specified by the selection process of customers. In the set anomaly detection (\cref{appendix_set_anomaly_det}), given a ground set $V$, $S^*$ is generated as the one containing anomaly data points. Therefore, the utility function in this setting is formulated as the anomaly pattern. Moreover, in the compound selection \cref{appendix_compount_selection}, we applied high bioactivity and diversity filters to select compounds. In this case, the utility function is determined by the bioactivity and diversity of the group of compounds.

\subsection{Detailed Experimental Settings for Product Recommendation} \label{appendix_prod_rec}
\textbf{Detailed Descriptions of the Amazon Baby Registry Dataset.}
The Amazon baby registry data \citep{gillenwater2014expectation} consists of baby registry data collected from Amazon and is split into several datasets according to product categories, such as toys, furniture, etc.
For each category, which can be considered as the product database, Amazon provides multiple sets of products selected by different customers. Thereby, these subsets of products can be viewed as OS oracles. 
To ensure that  each ground set $V$ only contains one OS oracle $S^*$, we construct the sample $(V, S^*)$ as follows.
For each subset of products selected by an anonymous user, we filter it out if its size is equal to $1$ or larger than $30$. For each OS oracle $S^*$ in the remaining subsets, we randomly sample $30 - |S^*|$ products in the same category to construct $V \backslash S^*$.
We summarize the statistics of the categories in \cref{tab:statistic_amazon_product}.

\begin{table}[t]
\centering
\small
\caption{The statistics of Amazon product dataset. \#products: number of all products in the category. $|\mathcal{D}|$:  number of samples in the dataset. }
\label{tab:statistic_amazon_product}
\vspace{5pt}
\begin{tabular}{@{}cccccccc@{}}\toprule
Categories & \#products & $|\mathcal{D}|$ & $|V|$ & $\sum |S^*|$ & $\mathbb{E}[|S^*|]$ & $\min_{S^*} |S^*|$ & $\max_{S^*} |S^*|$\\
\midrule
Toys & 62 & 2,421  & 30 & 9,924 & 4.09 & 3 & 14 \\
Furniture & 32 & 280  & 30 & 892 & 3.18 & 3 & 6 \\
Gear & 100 & 4,277  & 30 & 16,288 & 3.80 & 3 & 10 \\
Carseats & 34 & 483  & 30 & 1,576 & 3.26 & 3 & 6 \\
Bath & 100 & 3,195  & 30 & 12,147 & 3.80 & 3 & 11 \\
Health & 62 & 2,995  & 30 & 11,053 & 3.69 & 3 & 9 \\
Diaper & 100 & 6,108  & 30 & 25,333 & 4.14 & 3 & 15 \\
Bedding & 100 & 4,524  & 30 & 17,509 & 3.87 & 3 & 12 \\
Safety & 36 & 267  & 30 & 846 & 3.16 & 3 & 5 \\
Feeding & 100 & 8,202  & 30 & 37,901 & 4.62 & 3 & 23 \\
Apparel & 100 & 4,675  & 30 & 21,176 & 4.52 & 3 & 21 \\
Media & 58 & 1,485  & 30 & 6,723 & 4.52 & 3 & 19 \\
\bottomrule
\end{tabular}
\end{table}

\textbf{Comparing with the Setting of \cite[Section 5.3]{tschiatschek2018differentiable}.}
In \citep[Section 5.3]{tschiatschek2018differentiable}, \citet{tschiatschek2018differentiable} consider an alternative setting which is different from ours. Specifically, they construct the ground set $V$ as all the products in a category, and view the selected subsets of  all the customers as the corresponding optimal subsets. That is why they have the data points  in the form of $\{V, (S^*_1, \dots, S^*_N)\}$. This is a bit 
problematic since it is deviated from the real world scenario:  naturally the chosen subset $S^*_i$ shall depend on both $V$ and the $i$-th customer's personal preference. However, the customer is fully anonymized, so no information can be extracted from this dataset.

In order to be aligned with the real world scenario, we curate the dataset in the following way, in order to make data samples in the OS supervision oracle with the data in the form of $\{V_i, S^*_i\}$.   

For each category, we generate samples $(V, S^*)$ as follows. Firstly, we filter out those subsets selected by customers whose size is equal to $1$ or larger than $30$. Then we split the remaining  subset collection $\mathcal{S}$ into training, validation and test folds with a $1:1:1$ ratio. Finally for each OS oracle $S^* \in \mathcal{S}$, we randomly sample additional $30 - |S^*|$ products from the same category to construct $V\backslash S^*$. 

In this way, we construct one data point $(V, S^*)$ for each customer, which reflects this real world scenario: $V$ contains 30 products displayed to the customer, and the customer is interested in  checking $|S^*|$ of them.  
This is also consistent with real world recommender system, as users can only browse a small number of products at a time since the screen size of the device is limited, and the user has limited attention.

\subsection{Detailed Experimental Settings for Set Anomaly Detection} \label{appendix_set_anomaly_det}
In this experiment, we evaluate our methods on two real-world datasets:

\textbf{Double MNIST:} The dataset consists of 1000 images for each digit ranging from $00$ to $99$. For each sample $(V, S^*)$, we randomly sample $n \in \{2,\dots,5\}$ images with the same digit to construct the OS oracle $S^*$, and then select $20 - |S^*|$ images with different digits to construct the set $V \backslash S^*$. An example is shown in \cref{fig:double_mnist}.

\textbf{CelebA:} The CelebA dataset contains $202,599$ images with $40$ attributes. As shown in \cref{fig:celeba}, we select two attributes at random and construct the set with the size of $8$. For each ground set $V$, we randomly select $n \in \{2,3\}$ images as the OS oracle $S^*$, in which neither of the two attributes is present.
In this way, we arrive at  train, val, test datasets with 10,000, 1000, 1000 samples respectively. 

\begin{figure}[t]
\centering
\includegraphics[width=4.0in]{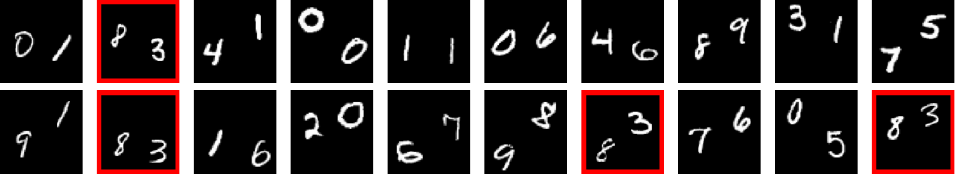}
\caption{A sampled data for the Double MNIST dataset, which consists of $|S^*|$ images with the same digit (red box, $83$ in this case) and $20 - |S^*|$ images with different digits.}
\label{fig:double_mnist}
\end{figure}

\begin{figure}[t]
\centering
\includegraphics[width=4.0in]{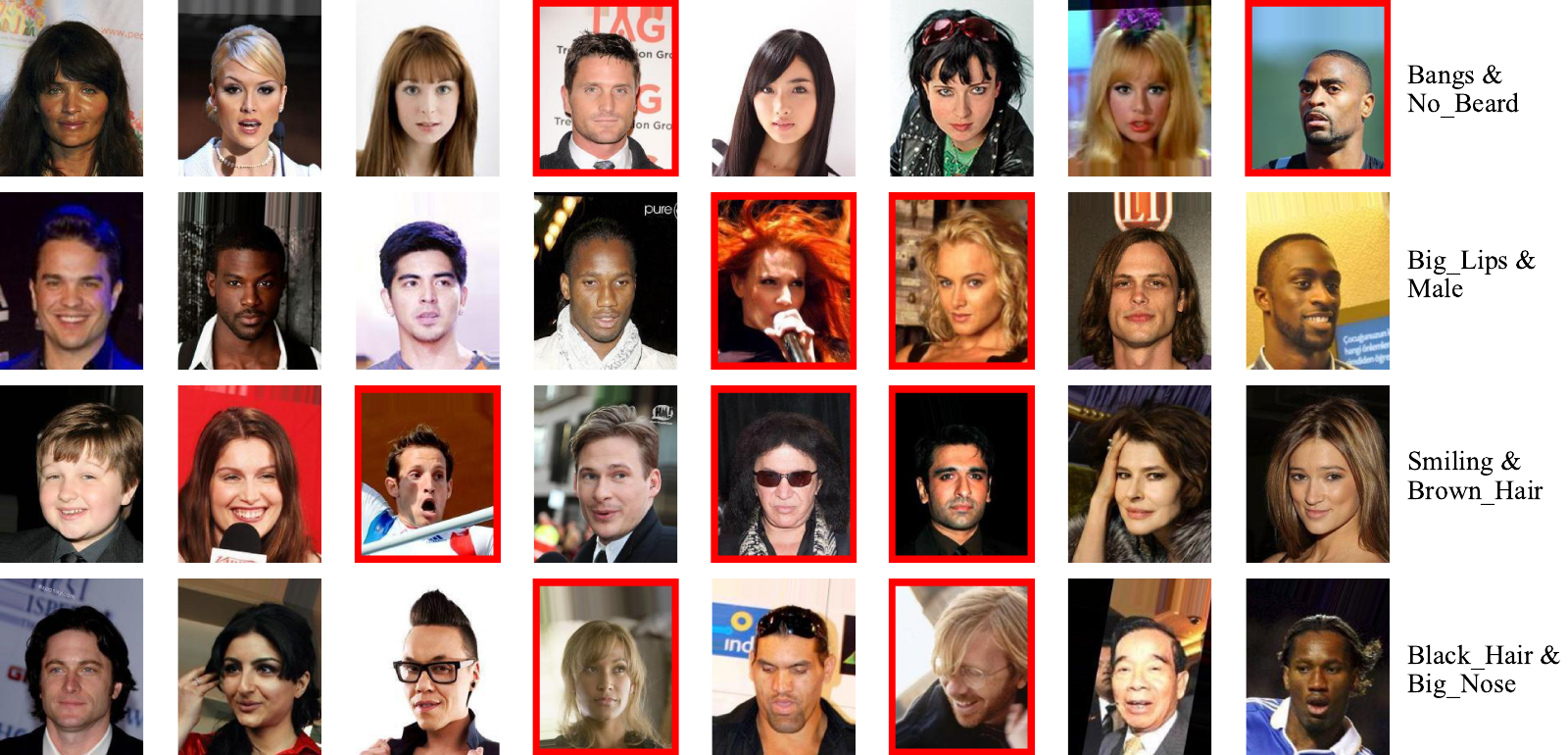}
\caption{Sampled data points for the CelebA dataset. Each row is a sample, consisting of $|S^*|$ anomalies (red box) and $8-|S^*|$ normal images. In each sample, a normal image has two attributes (rightmost column) while anomalies do not have neither of them.}
\label{fig:celeba}
\end{figure}

\subsection{Detailed Experimental Settings for Compound Selection} \label{appendix_compount_selection}
\cref{alg:compound_selection} shows the corresponding data generation process of simulating the OS oracle for compound selection. In this algorithm, $\operatorname{random\_choose}(\mathcal{C},n)$ means randomly choosing $n$ compounds from the database $\mathcal{C}$ ({\it i.e.,} PDBBind or BindingDB), and $\operatorname{topK\_bioactivity}(V,m)$ represents selecting the top-$m$ compounds with highest biological activity from the ground set $V$. These two operators combine together to form the bioactivity filter, in which we set $(n, m)$ as $(30, 10)$, and $(300, 100)$ for PDBBind and BindingDB, respectively.
To further apply the diversity filter, we use the RDKit\footnote{\url{https://github.com/rdkit/rdkit}} tools to compute the similarity between each molecule pair based on their topological fingerprints. This operator corresponds to the line $3$ of \cref{alg:compound_selection}, in which $\operatorname{cal\_fingerprint\_similarity}(S)$ returns the similarity matrix $\boldsymbol{M} \in \mathbb{R}^{|S|\times |S|}$ of the set of compounds $S$.
Since rows (or columns) of the similarity matrix can be regarded as the features of the corresponding molecules, the molecules are clustered based on these similarity features by applying the affinity propagation algorithm. 
The OS oracle $S^*$ is finally represented by the center of each cluster.
Note that, each compound consists of two small molecules, {\it i.e.,} the protein-ligand molecules in PDBBind, and the drug-target molecules in BindingDB. We use the protein and drug molecules to compute the fingerprint similarity for PDBBind and BindingDB, respectively.

\begin{algorithm}[!t]
\small
\textbf{Input}: $\mathcal{C}$: compound database; $n$: size of ground set; $m$: number of the most active compounds \\
\textbf{Output}: Data point $ (V, S^*)$
\begin{algorithmic}[1]
    \STATE Randomly select $n$ compounds to construct the ground set \tikzmark{top1} \qquad \tikzmark{right} \\ $V \leftarrow \operatorname{random\_choose}(\mathcal{C}, n)$
    \STATE Filter out $m$ compounds with the highest bioactivity \\ $S \leftarrow \operatorname{topK\_bioactivity}(V, m)$ \tikzmark{bottom1}
    \STATE Calculate the similarity matrix \tikzmark{top2} \\ $\boldsymbol{M} \leftarrow \operatorname{cal\_fingerprint\_similarity}(S)$
    \STATE Apply the affinity propagation algorithm \\ $\text{af} \leftarrow \operatorname{affinity\_propogation}(\boldsymbol{M})$
    \STATE Assign the OS oracle as cluster centers \\ $S^* \leftarrow \operatorname{af.cluster\_centers\_indices}$ \tikzmark{bottom2}  \vspace{-3.mm} 
\end{algorithmic}
\AddNote{3.cm}{top1}{bottom1}{right}{\textbf{bioactivity filter}}
\AddNote{3.cm}{top2}{bottom2}{right}{\textbf{diversity filter}}
\caption{OS Oracle Generation Algorithm for Compound Selection Task}
\label{alg:compound_selection}
\end{algorithm}

\section{Additional Experiments}

\subsection{Experiments on Varying Ground Set} \label{exp-varying-ground-set}
Thanks to the virtues of DeepSet, our models are able to process input sets of variable sizes, which is termed as \emph{varying ground set} property. To examine the impact of ground set sizes, we care about the following two questions: i) how well the model performs on different sizes of ground set during the test time; and ii) how well does the model train on ground sets of different sizes? To answer these two questions, we conduct experiments on the synthetic datasets using the proposed model $\ours_\texttt{copula}$.

\paragraph{Set Size Transferability Analysis}
We first experiment to understand the pattern of set size transferability. In this experiment, we train the model using fixed sizes of the ground set but test the trained model on different sizes. We present two scenarios: train on a small size but test on a large one, and train on a large size but test on a small one. For the former one, we fix the size of OS oracle $S^*$ to be $10$, and train the model with ground set $V$ of size $100$. After training, we test it using varying sizes of ground set in the range of $\{200, 400, 600, 800, 1000\}$. For the latter one, we fix the size of OS oracle $S^*$ to be $10$, and train the model with ground set $V$ of size $1000$. After training, we test it using varying sizes of ground set in the range of $\{100, 200, 400, 600, 800\}$. The former and latter experiments are conducted on the Two-Moons and Gaussian-Mixture datasets, respectively, with the results shown in \cref{fig:synthetic_size_tran}. As can be seen, the performance would be slightly reduced if tested on a different size. Moreover, increasing the difference would enlarge the reduction.
\begin{figure}[!t]
	\centering
		\begin{minipage}[t]{0.95\linewidth}
			\centering
			\includegraphics[width=2.0in]{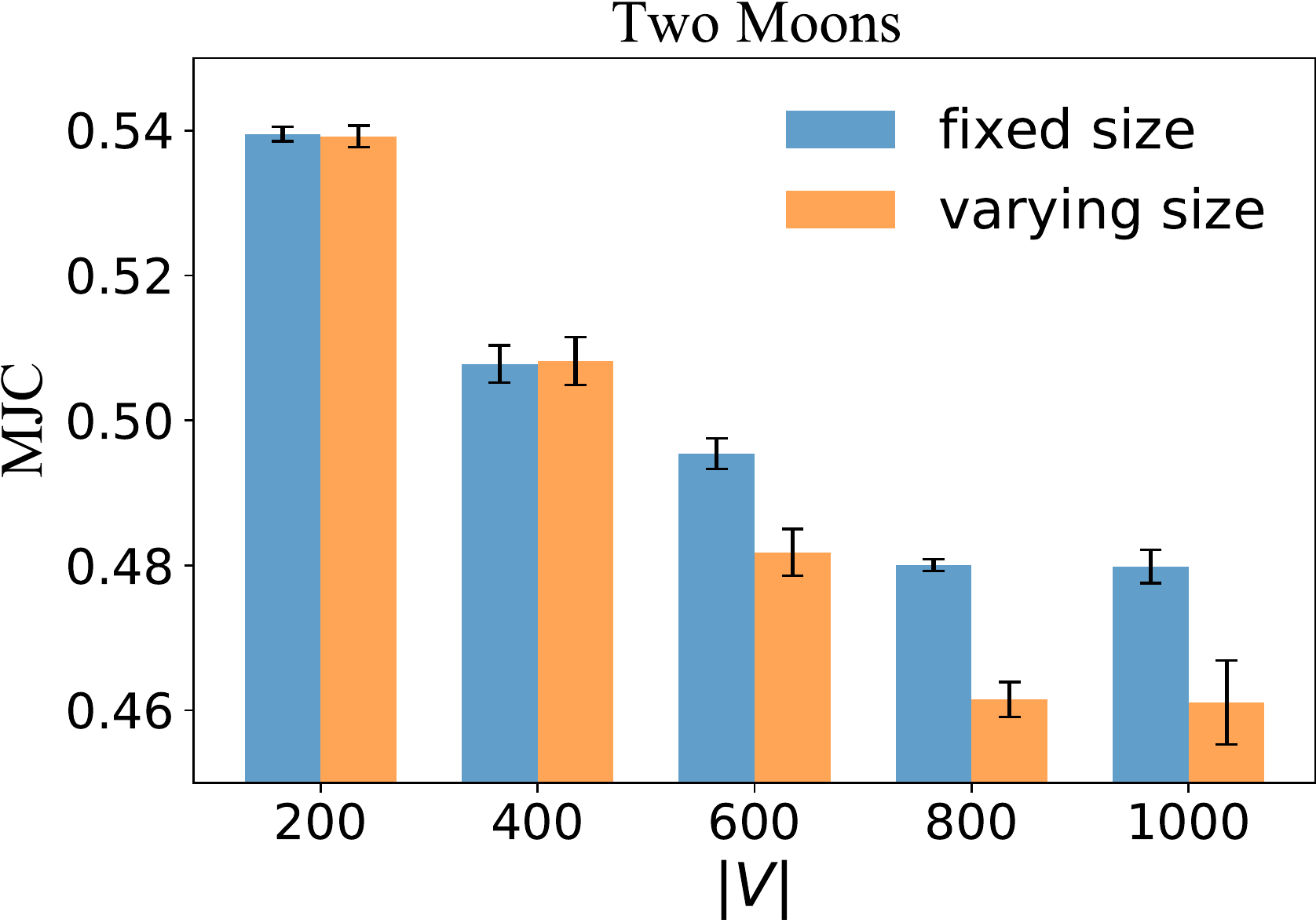}
			\includegraphics[width=2.0in]{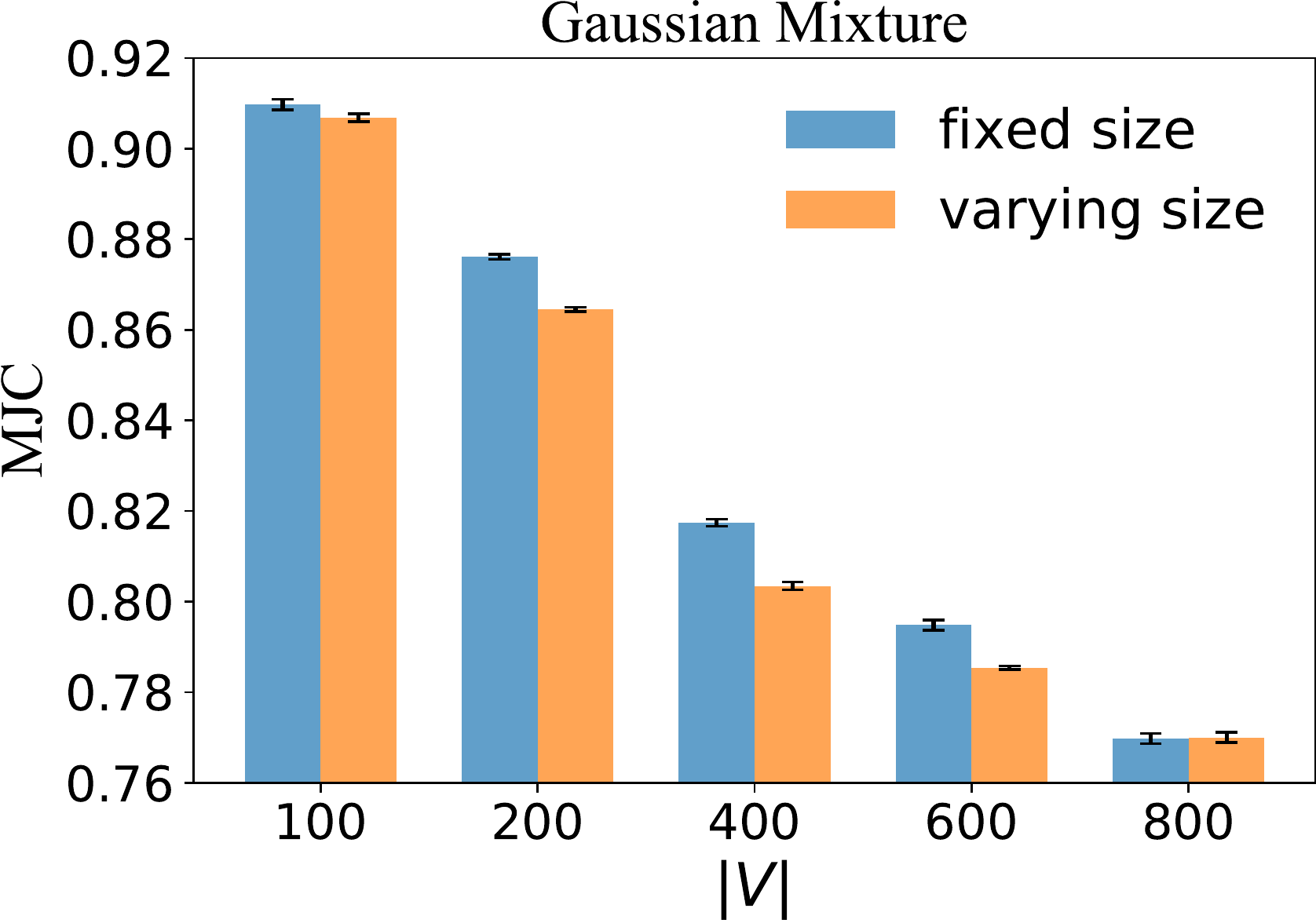}
	        \caption{ Synthetic results of $\ours_\texttt{copula}$ for set size transferability analysis, in which the blue bars represent the performances of using the same sizes of ground set during the training and test time, while the yellow bars mean using different sizes of ground set during the test time. Detailed descriptions are given in the main text.}
	        \label{fig:synthetic_size_tran}
		\end{minipage}
\end{figure}

\paragraph{Selection Ratio Analysis}
To answer the second question, we fix the size of OS oracle $S^*$ to be $10$, and experiment with different selection ratios $\frac{|S^*|}{|V|}$ in the range of $\{0.5, 0.2, 0.1, 0.05, 0.01\}$. Unlike the set size transferability analysis, in this experiment, the selection ratios are the same during training and testing. \cref{fig:synthetic_diff_set_size} shows the performance of different ratios on two synthetic datasets. We observe that increasing the ratio would deteriorate the model performance. This phenomenon makes intuitive sense, since sampling subset from a large collection is more difficult. Moreover, the model performs worst when the ratio is equal to $0.5$. This is partly because $S^*$ and $V \backslash S^*$ are randomly sampled from one of two components. When $|S^*|$ = $|V \backslash S^*|$, the model struggles to identify the optimal subset. 
\begin{figure}[!t]
	\centering
		\begin{minipage}[t]{0.95\linewidth}
			\centering
			\includegraphics[width=2.0in]{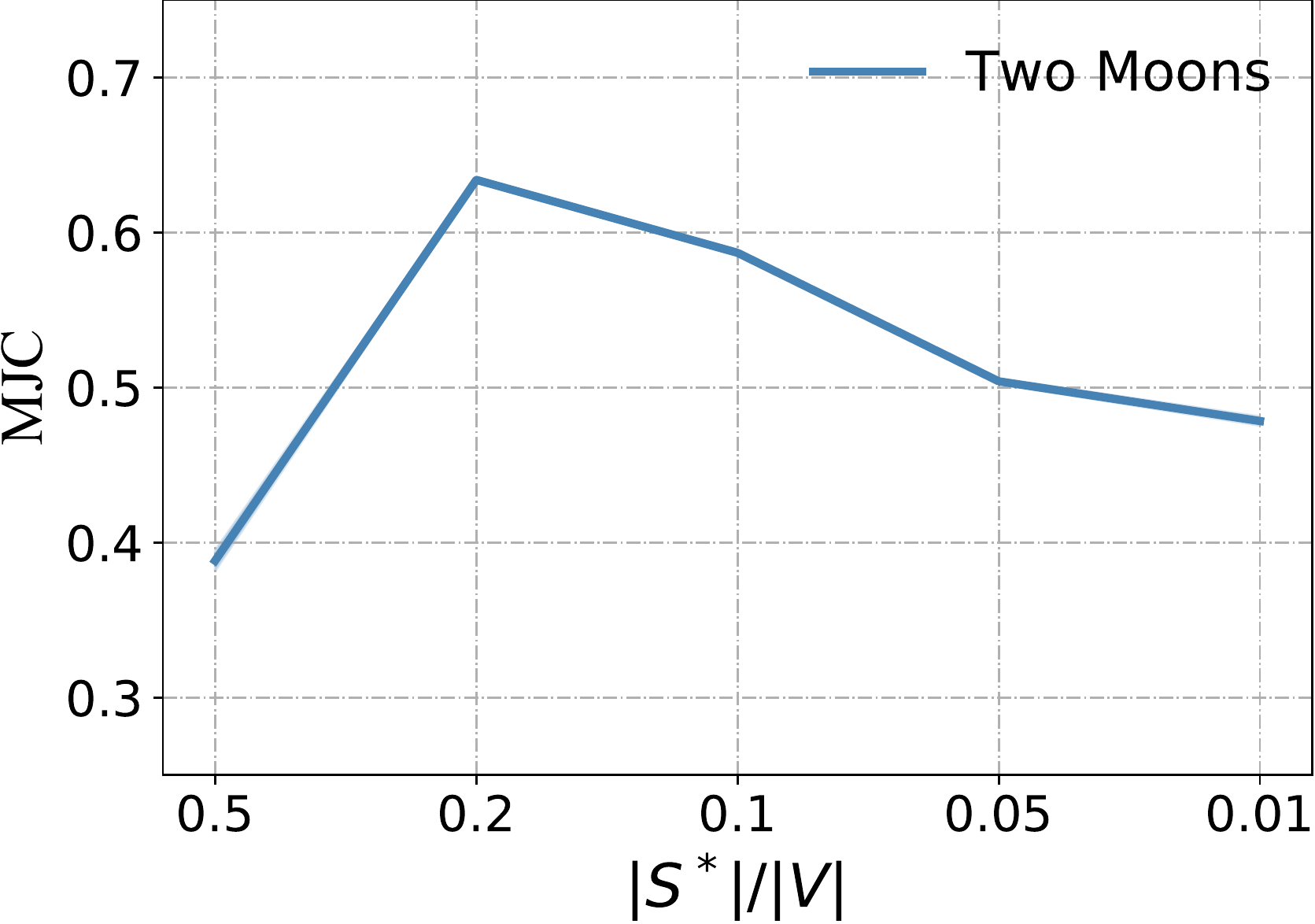}
			\includegraphics[width=2.0in]{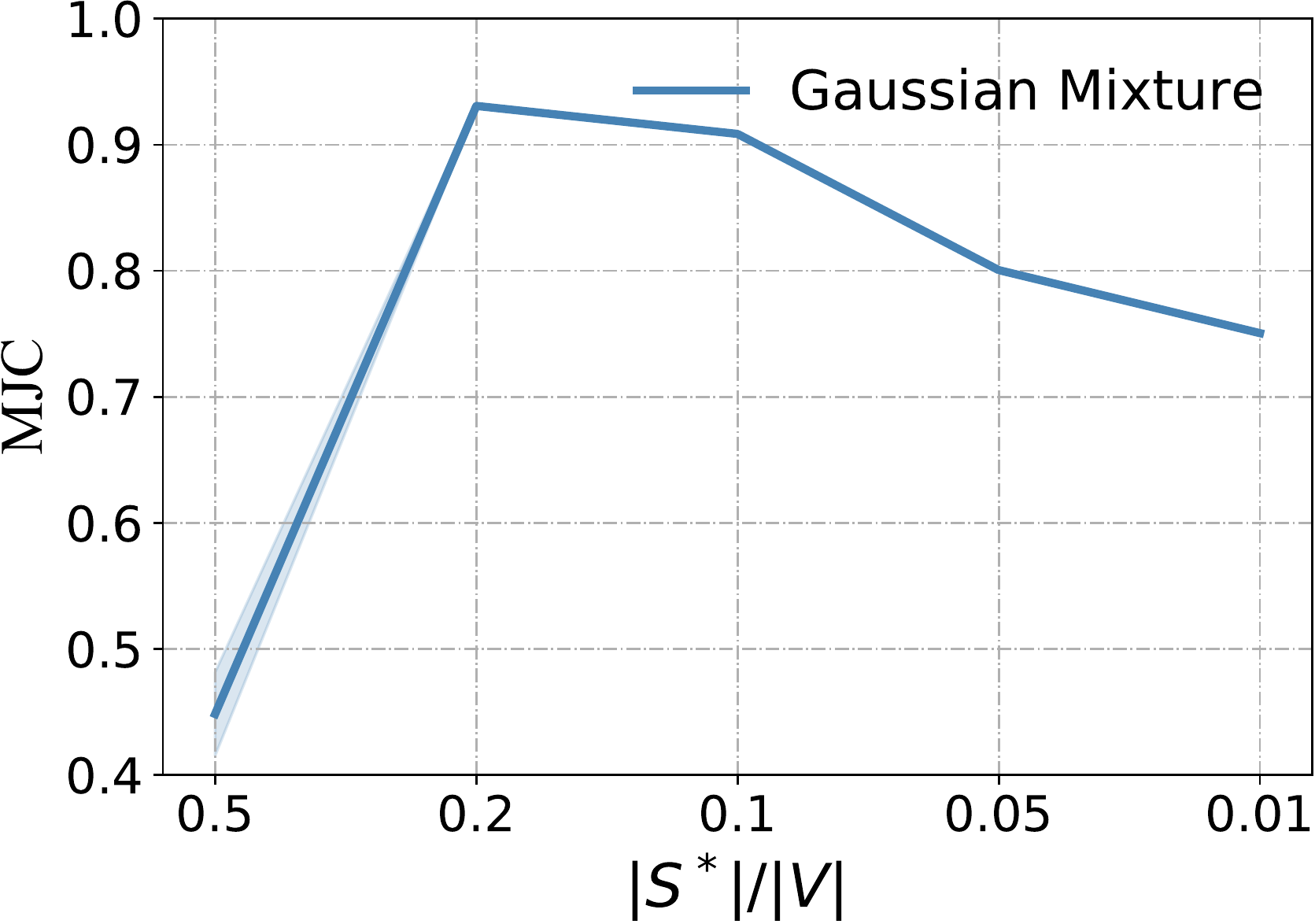}
	        \caption{ Synthetic results of \ourscopula with varying selection ratios.}
	        \label{fig:synthetic_diff_set_size}
		\end{minipage}
\end{figure}

\subsection{Comparisons with Set Transformer} \label{appendix_comp_set_trans}

Set Transformer \citep{lee2019set}, which satisfies permutation invariant, is a well-known architecture used to model interactions among elements in the input set. Similar to DeepSet, Set Transformer could be adapted to serve as a baseline. Specifically, the architecture of the SetTransformer (NoSetFn) baseline is
\begin{align}
     \operatorname{ModuleList}( \initlayer(V, 256), \sab(256, 500, 2, -), \sab(500, 1, 2, \sigmoid)), \nonumber
\end{align}
where $\sab(d, h, m, f)$\footnote{We take the implementation of SAB from \url{https://github.com/juho-lee/set_transformer}.} denotes the set attention block \citep{lee2019set} with $d$ dimensional set input, $h$ dimensional set output, $m$ multi-head attentions, and activation function $f$. We train the adapted Set Transformer model: $2^V \rightarrow [0,1]^{|V|}$ with cross entropy loss and sample the subset via the topN rounding. It is noteworthy that, like DeepSet (NoSetFn), SetTransformer (NoSetFn) does not learn a set function explicitly, although it can be adapted as a baseline and can be viewed as merely modelling the amortized network in our \ours framework.

For fair comparison, we also replace the DeepSet backbone with Set Transformer in \ours. Specifically, the $\initlayer$ in \cref{tab:equivset_architecture} is replaced with 
\begin{align}
     \operatorname{ModuleList}( \initlayer(S, 256), \sab(256, 500, 2, -), \sab(500, 256, 2, -)). \nonumber
\end{align}
Experiments are conducted on product recommendations, with the results shown in \cref{tab:product_recom_settrans}. It shows that the proposed approaches with the Set Transformer backbone outperform the Set Transformer (NoSetFn) comprehensively. One could also compare the results of \cref{tab:product_recom} in the paper. It can be seen that the proposed \ours (with DeepSet backbone) also performs better than the Set Transformer baseline. Moreover, \ours (with DeepSet backbone) outperforms \ours (with Set Transformer backbone) consistently, indicating that \ours has great potential to be improved with more advanced architecture.

\begin{table}[t]
\centering
\small
\caption{Product recommendation results on Set Transformer baselines and backbones.}
\label{tab:product_recom_settrans}
\hspace*{-0.3cm}
\setlength{\tabcolsep}{2.mm}{
\begin{tabular}{@{}ccccccc@{}}\toprule
Categories & Set Transformer (NoSetFn) & \oursdiff (ours) & \oursind (ours) & \ourscopula (ours) \\
\midrule
Toys & 0.640 $\pm$ 0.030 & 0.690 $\pm$ 0.030 & 0.680 $\pm$ 0.020 & \textbf{0.717} $\pm$ \textbf{0.006}\\
Furniture & \textbf{0.175} $\pm$ \textbf{0.008} & 0.170 $\pm$ 0.020 & {0.159} $\pm$ {0.006} & 0.166 $\pm$ 0.007\\
Gear & 0.639 $\pm$ 0.006 & 0.750 $\pm$ 0.030 & 0.690 $\pm$ 0.020 & \textbf{0.700} $\pm$ \textbf{0.010}\\
Carseats  & 0.219 $\pm$ 0.005 & \textbf{0.219} $\pm$ \textbf{0.006} & 0.219 $\pm$ 0.009 & 0.216 $\pm$ 0.008\\
Bath & 0.725 $\pm$ 0.005 & 0.800 $\pm$ 0.020 & 0.800 $\pm$ 0.010 & \textbf{0.810} $\pm$ \textbf{0.010}\\
Health & 0.680 $\pm$ 0.010 & 0.750 $\pm$ 0.020 & 0.750 $\pm$ 0.020 &\textbf{0.760} $\pm$ \textbf{0.020}\\
Diaper & 0.789 $\pm$ 0.005 & 0.871 $\pm$ 0.009 & 0.870 $\pm$ 0.010 &\textbf{0.886} $\pm$ \textbf{0.009}\\
Bedding & 0.760 $\pm$ 0.020 & 0.859 $\pm$ 0.008 & 0.860 $\pm$ 0.020 &\textbf{0.860} $\pm$ \textbf{0.007}\\
Safety & 0.257 $\pm$ 0.005 & 0.240 $\pm$ 0.006 & 0.240 $\pm$ 0.010 &\textbf{0.260} $\pm$ \textbf{0.030}\\
Feeding & 0.783 $\pm$ 0.006 & \textbf{0.886} $\pm$ \textbf{0.004} & 0.881 $\pm$ 0.010 &{0.878} $\pm$ 0.009\\
Apparel & 0.680 $\pm$ 0.020 & 0.760 $\pm$ 0.010 & 0.550 $\pm$ 0.010 &\textbf{0.770} $\pm$ \textbf{0.010}\\
Media & 0.540 $\pm$ 0.020 & 0.615 $\pm$ 0.008 & 0.610 $\pm$ 0.010 &\textbf{0.620} $\pm$ \textbf{0.009}\\
\bottomrule
\end{tabular}
\vspace{-2mm}
}
\end{table}

\subsection{Experiments on Set Anomaly Detection with F-MNIST and CIFAR-10} \label{appendix_sad_fmnist_cifar10}

In this experiment, we further perform set anomaly detection on the other two datasets: F-MNIST \citep{xiao2017fashion} and CIFAR-10 \citep{krizhevsky2009learning}. Both two datasets contain images with 10 different labels. For each dataset, we randomly sample $n \in \{2,3\}$ images as the OS oracle $S^*$, and then select $8 - |S^*|$ images with different labels to construct the set $V\backslash S^*$. We finally obtain the training, validation, and test set with the size of $10,000, 1,000, 1,000$, respectively, for both two datasets. Illustrations of sampled data are shown in \cref{fig:fmnist-cifar10}.

The results are shown in \cref{fig:fmnist-cifar10}. We see that the variants of our model consistently outperform baseline methods strongly. Moreover, \oursdiff seems to perform better than \oursind and \ourscopula in set anomaly detection (similar results can be found in \cref{tab:meta_anomaly}). However, this is not a consistent phenomenon. It seems that in most scenarios, e.g., product recommendation, compound selection, and synthetic dataset, \ours performs better than \oursdiff.

\begin{figure}[!t]
    \begin{minipage}[!b]{0.48\linewidth}
        \centering
		\includegraphics[width=2.6in]{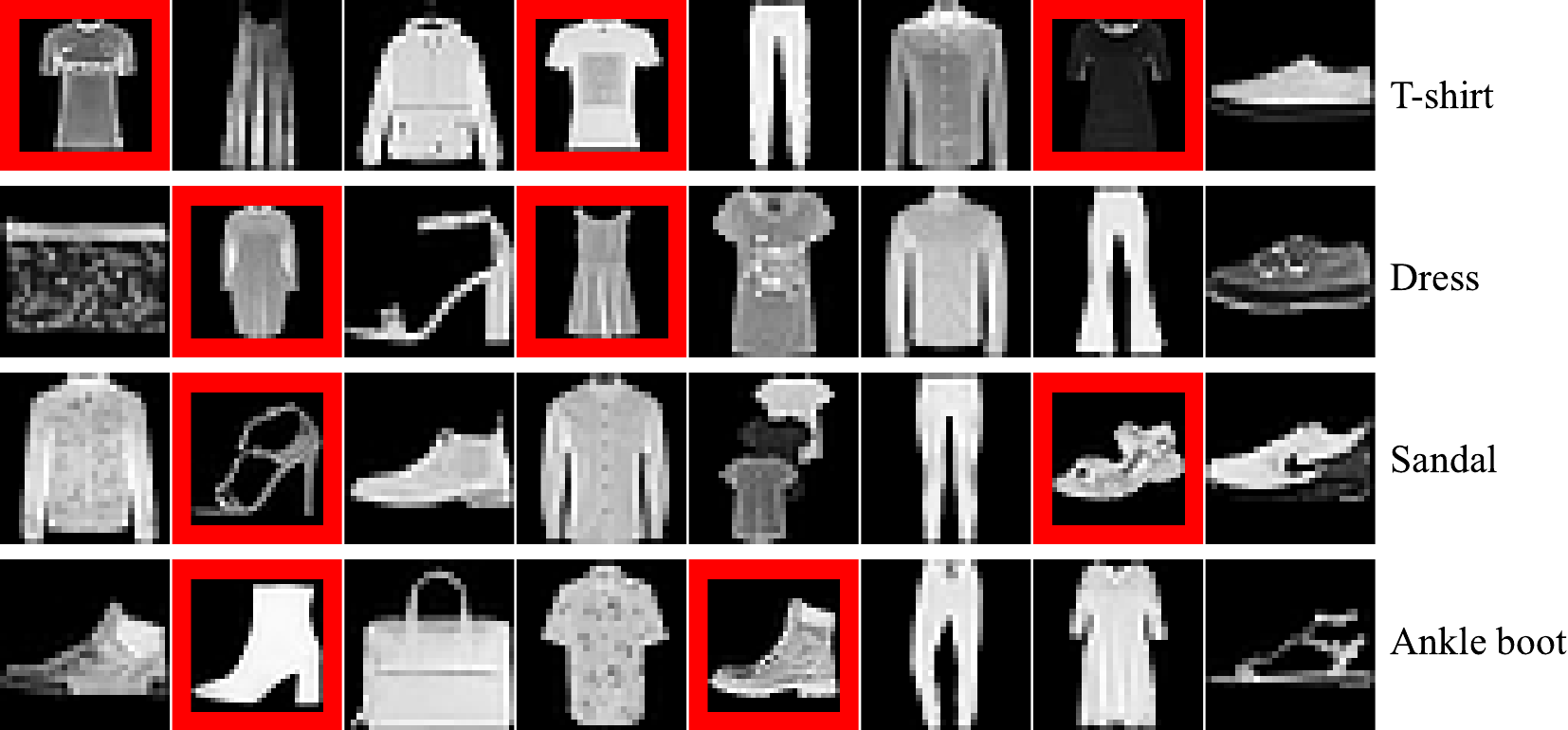}
    \end{minipage}
    \begin{minipage}[!b]{0.48\linewidth}
        \centering
		\includegraphics[width=2.6in]{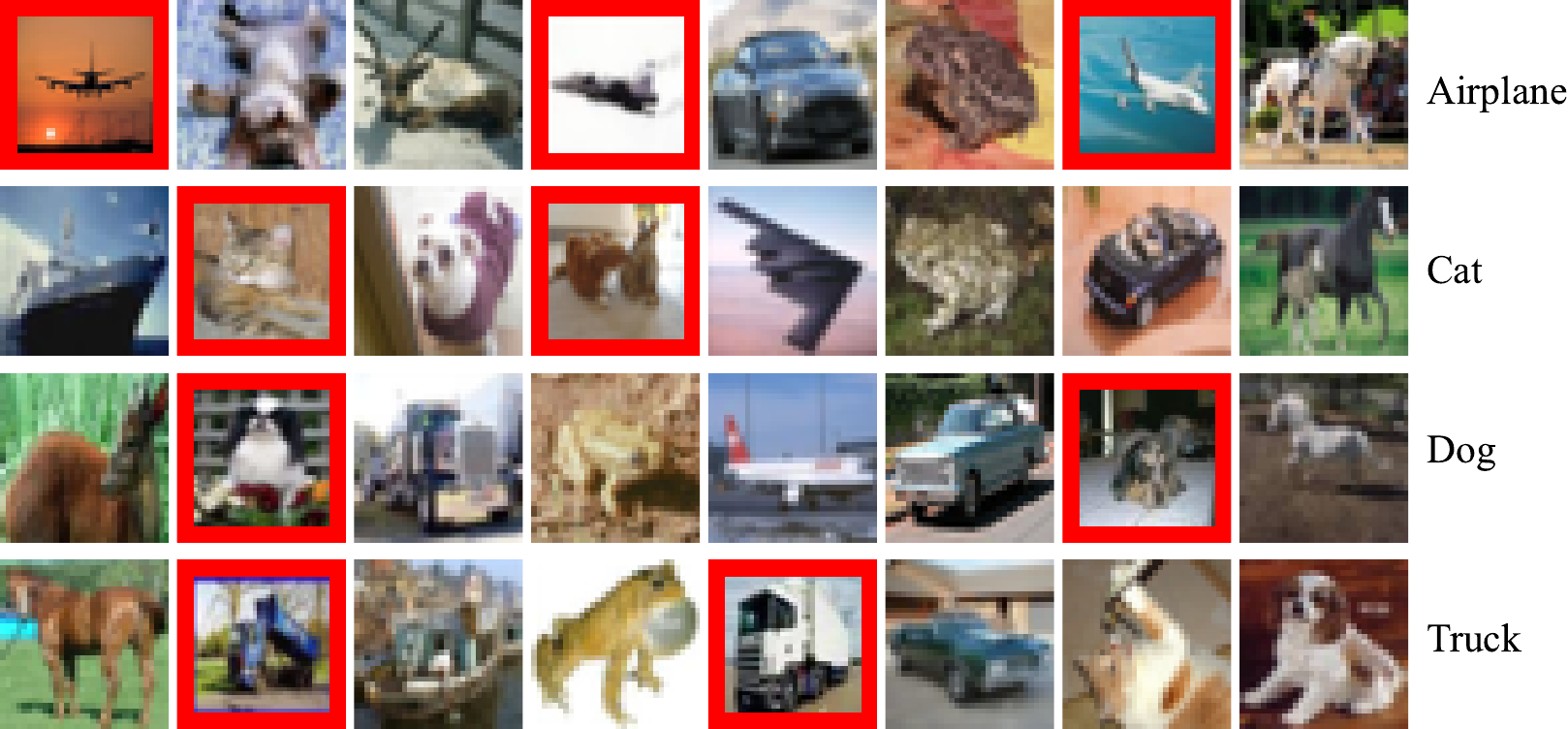}
    \end{minipage}
    \caption{Sampled data points for the F-MNIST (left) and CIFAR-10 (right) datasets. Each row is a sample, containing of $|S^*|$ images (red box) with the same label (rightmost column)  and $8-|S^*|$ images with different labels.}
	\label{fig:fmnist-cifar10}
\end{figure}

\subsection{Experiments on Compound Selection with Only the  Bioactivity Filter} \label{appendix_exp_compount_selection_activity}
To further evaluate the potential of \ours for drug discovery, we consider an alternative setting here. 
In contrast to the task in \cref{sec_compound_selection}, which aims at selecting the most active compounds while preserving diversity, the task defined here only focuses on selecting the compounds with the highest bioactivity, which results a relatively simple selection process. 
The following is a detailed description.

\textbf{PDBBind:} To construct a data point $(V, S^*)$, we randomly sample $30$ complexes as the ground set $V$ from the PDBBind database, and  $S^*$ is generated by the five most active complexes in $V$. Finally, we obtain the training, validation, and test set with the size of $1,000$, $100$, $100$, respectively. \textbf{BindingDB:} We construct the ground set $V$ by randomly sampling $300$ drug-targets from the BindingDB database and generate $S^*$ with the $15$ most active drug-target pairs. We finally obtain the training, validation, and test set with the size of $10,00$, $1,00$, and $1,00$, respectively.

\cref{tab:appebdix_drug_ai2} shows that our methods outperform the baselines. Meanwhile, the baselines also show satisfactory results.  
That is because identifying the most active compounds is a relatively  simple task, especially for the PDBBind dataset with complex structures. 
More specifically, the model could predict the activity value of complexes precisely without considering the interactions between elements in the set, since the structure of complexes has provided sufficient information for this task.
It is worth noting that the models in this task perform better than that in \cref{sec_compound_selection} partly because a one-layer filter (i.e. bioactivity) represents an easier way to replicate the OS oracle than a two-layer filter (i.e. bioactivity and diversity).
Nevertheless, both experimental results in \cref{sec_compound_selection} and here demonstrate the effectiveness of \ours for facilitating the complicated compound selection process.

\begin{table}[!t]
\begin{minipage}[t]{0.48\linewidth}
\centering
\small
\caption{Set anomaly detection results on the F-MNIST and CIFAR-10.}
\label{tab:appendix_meta_anomaly}
\setlength{\tabcolsep}{1.5mm}{
\scalebox{0.9}{
\begin{tabular}{@{}ccccc@{}}\toprule
Method & F-MNIST & CIFAR-10  \\
\midrule
Random & 0.193 & 0.193 \\
PGM & 0.540 $\pm$ 0.020 & 0.450 $\pm$ 0.020 \\
DeepSet (NoSetFn) & 0.490 $\pm$ 0.020 & 0.316 $\pm$ 0.008 \\
\oursdiff (ours) & \textbf{0.700} $\pm$ \textbf{0.020} & \textbf{0.710} $\pm$ \textbf{0.010} \\
\oursind (ours) & 0.590 $\pm$ 0.010 & 0.570 $\pm$ 0.020 \\
\ourscopula (ours) & 0.650 $\pm$ 0.010 & 0.600 $\pm$ 0.010 \\
\bottomrule
\end{tabular}}}
\end{minipage}
\begin{minipage}[t]{0.48\linewidth}
\centering
\small
\caption{Compound selection results with only the bioactivity filter.}
\vspace{-1mm}
\label{tab:appebdix_drug_ai2}
\setlength{\tabcolsep}{1.5mm}{
\scalebox{0.9}{
\begin{tabular}{@{}ccccc@{}}\toprule
Method & PDBBind & BindingDB  \\
\midrule
Random & 0.099 & 0.009 \\
PGM & 0.910 $\pm$ 0.010 & 0.690 $\pm$ 0.020\\
DeepSet (NoSetFn) & 0.910 $\pm$ 0.010 & 0.680 $\pm$ 0.010\\
\oursdiff (ours) & 0.920 $\pm$ 0.010 & 0.690 $\pm$ 0.020\\
\oursind (ours) & {0.930} $\pm$ {0.010} & 0.697 $\pm$ 0.006\\
\ourscopula (ours) & \textbf{0.931} $\pm$ \textbf{0.008} & \textbf{0.700} $\pm$ \textbf{0.008}\\
\bottomrule
\end{tabular}}}
\end{minipage}
\vspace{-2mm}
\end{table}

\subsection{Sensitivity Analysis of Hyperparameters} \label{sensitive-analysis}
The proposed model \ourscopula has four important hyperparameters: the number of Monte Carlo sampling $m$ and mean-field iteration step $K$ in \cref{alg:MFI}, the rank of lower-rank perturbation $v$ in \eqref{lower-rank-perturb}, and the temperature $\tau$ of Gumbel-Softmax in \cref{alg:indbernoulli,alg:copulabernoulli}. In this section, we discuss the impact of these hyperparameters through a sensitivity analysis on the Amazon product datasets.

\paragraph{Impact of the Mean Field Iteration Step}
Since iteration step $K$ controls the convergence of mean-field iterative algorithms, this hyperparameter is highly relevant to the final performance of \ourscopula. We experiment with different $K$ on the Amazon product dataset. The results are shown in the first row of \cref{fig:sensitive_analysis}. We notice that increasing K would degenerate the model's performance. This seems to be embarrassingly surprising at first glance, since a large stride $K$ encourages convergence with guarantee, resulting in a more robust training process. It is worth to be noted that in this method, we apply an amortized variational distribution to initialize the parameters for mean-field iterative algorithms. Since the amortized variational distribution is modeled with Gaussian copula, it can effectively capture the correlation among elements in the set, such that obtaining a better local optimal. However, if the iterative step $K$ is large, the model inclines to converge to the local optimal that is the same as the original mean-field iteration. As a result, the benefit of correlation-aware inference provided by the Gaussian copula would be diminished. This explains why the iterative step $K$ cannot be set too large.

\paragraph{Impact of the Number of MC Sampling}
The number of Monte Carlo (MC) sampling $m$ plays an important role in the proposed method. It is widely known that increasing number of samples would reduce the variance of MC sampling. Therefore, using larger $m$ would result in a better approximation of the gradient of multilinear extension $\nabla_{\psi} f_{\mathrm{mt}}^{F_\theta}$ and thus better performance. This hypothesis is validated by the empirical results show in the second row of \cref{fig:sensitive_analysis}. It can be seen that as the sample number increases, the performance rises steadily at first and then gradually converges into a certain level. Undoubtedly, a large number would increase the computational complexity. In this regard, we uniformly set it as $5$ in all experiments.

\paragraph{Impact of the Lower-rank Perturbation}
Lower-rank perturbed covariance matrix enables the proposed method to model the correlation information of elements in the set. To investigate its impacts, we evaluate the performance of \ourscopula under different values of rank $v$. The results are demonstrated in the third row of \cref{fig:sensitive_analysis}. Notably, the proposed model with $v=0$ is equivalent to \oursind. 
It can be seen that as the number of ranks increases, the performances also increase, indicating the hypothesis that employing the variational distribution with correlations can increase the model's representational capacity and thereby results in a better approximation in turn.  It is worth noting that the most significant performance improvement is observed between the models with $v = 0$ and $v = 1$, and then as the value of $v$ continues to increase, the improvement becomes relatively small. This indicates that it is feasible to set the $v$ to a relatively small value to save computational resources while retaining competitive performance.

\paragraph{Impact of the Temperature Parameter of Gumbel-Softmax}
The temperature parameter $\tau$ controls the trade-off between accuracy and variance of the approximation. With lower temperatures ($\tau \rightarrow 0$), the samples become more discrete but have a high variance of gradients. Alternatively, high temperatures ($\tau \rightarrow \infty$) result in smooth variables while enjoying a low variance of gradients. Fortunately, the experimental results in the last row of \cref{fig:sensitive_analysis} show that our model is quite robust with varying temperature values. It can be seen that the performance of models drops when $\tau = 1$, but the variance of performances is mild. We set $\tau = 0.1$ in the experiments.

\begin{figure}[!t]
	\centering
		\begin{minipage}[t]{0.95\linewidth}
			\centering
			\includegraphics[width=1.15in]{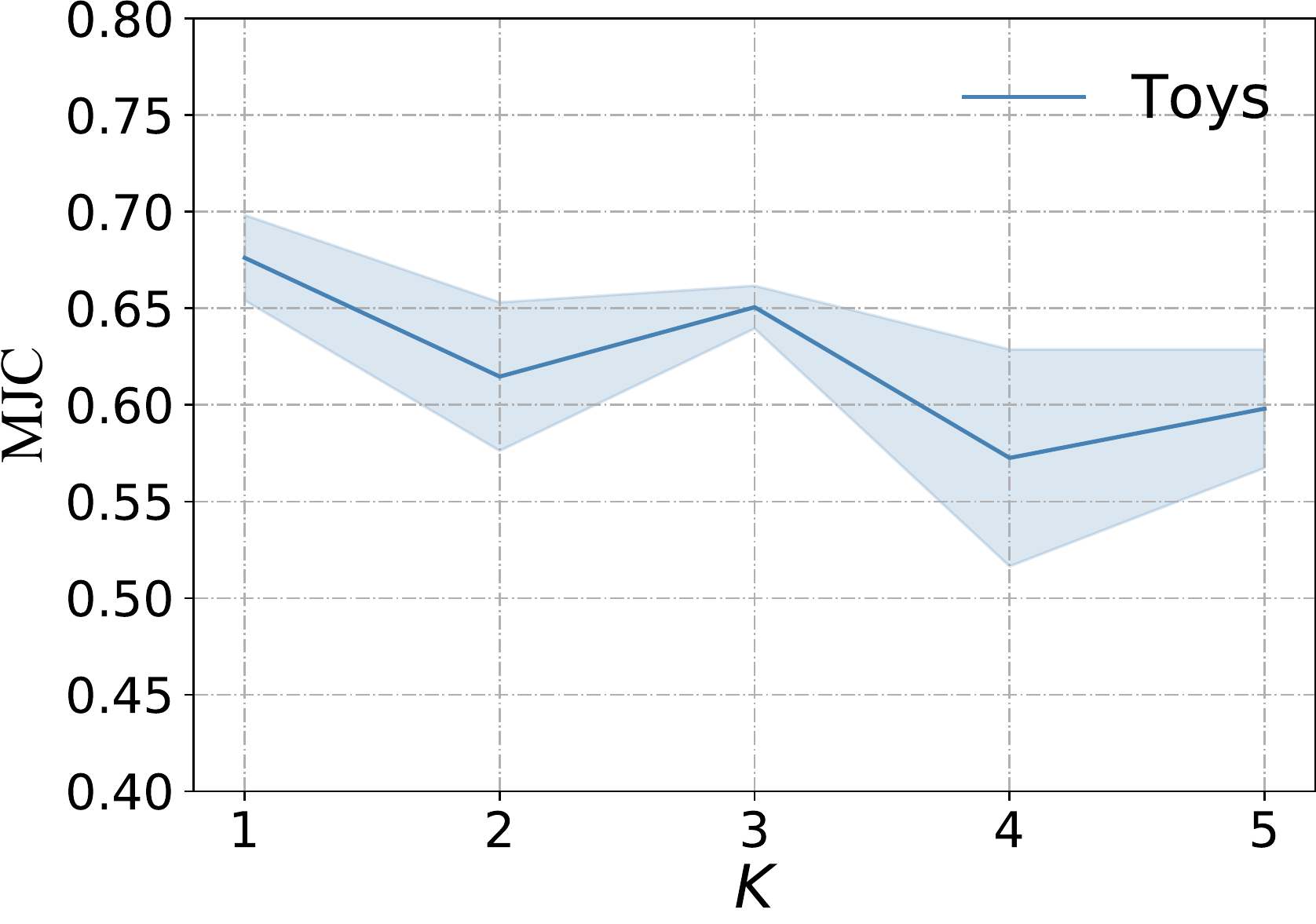}
			\includegraphics[width=1.15in]{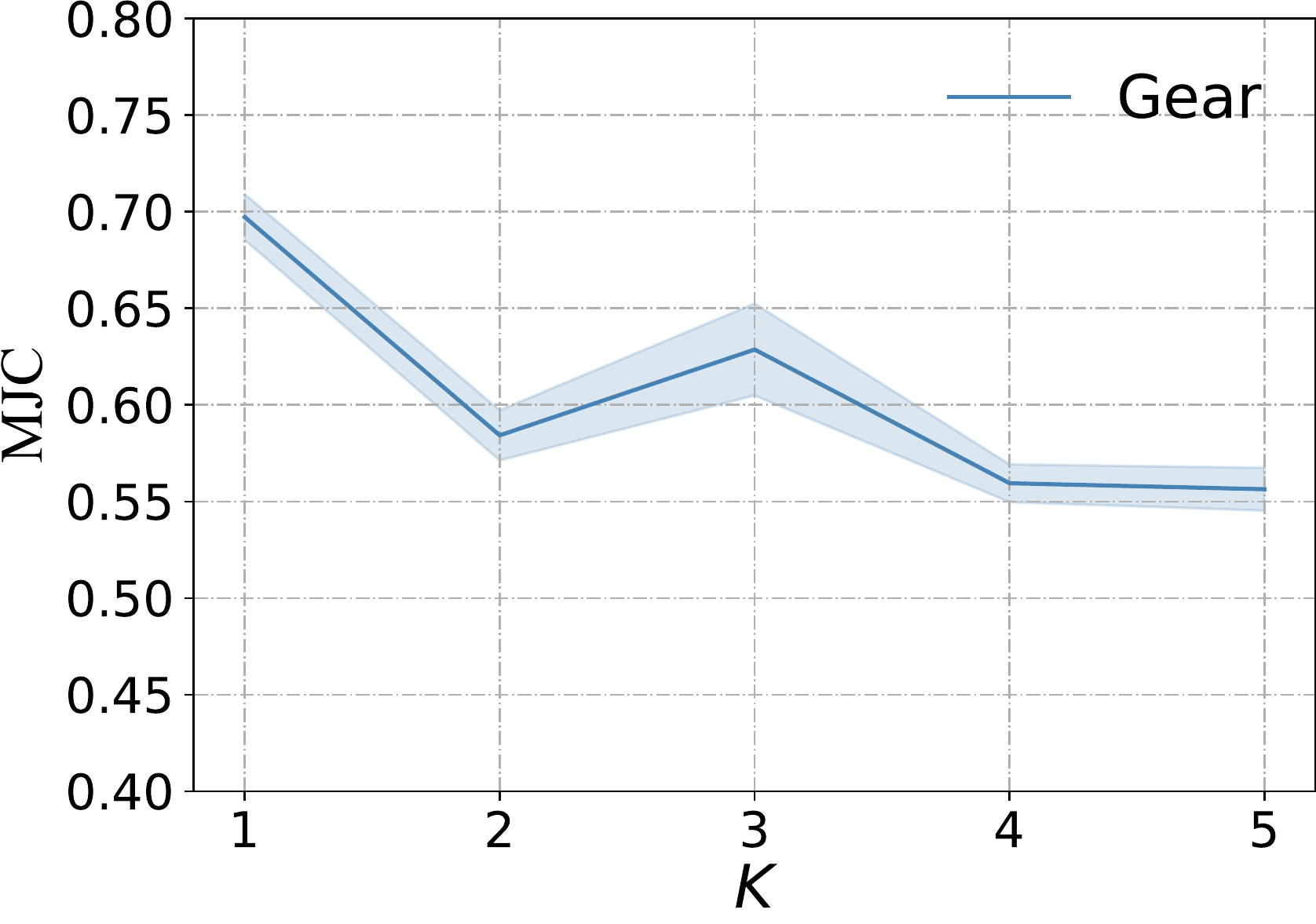}
			\includegraphics[width=1.15in]{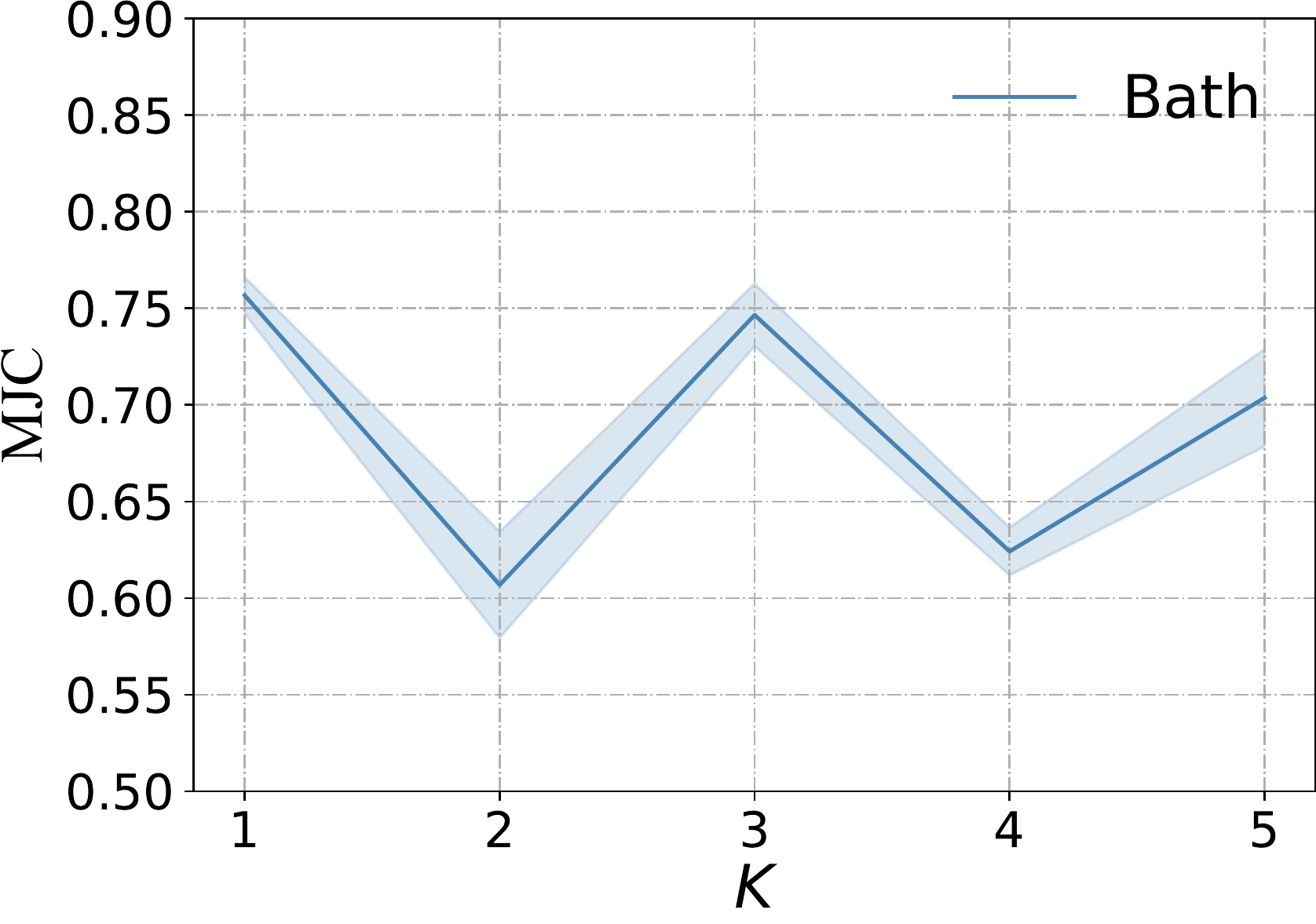}
			\includegraphics[width=1.15in]{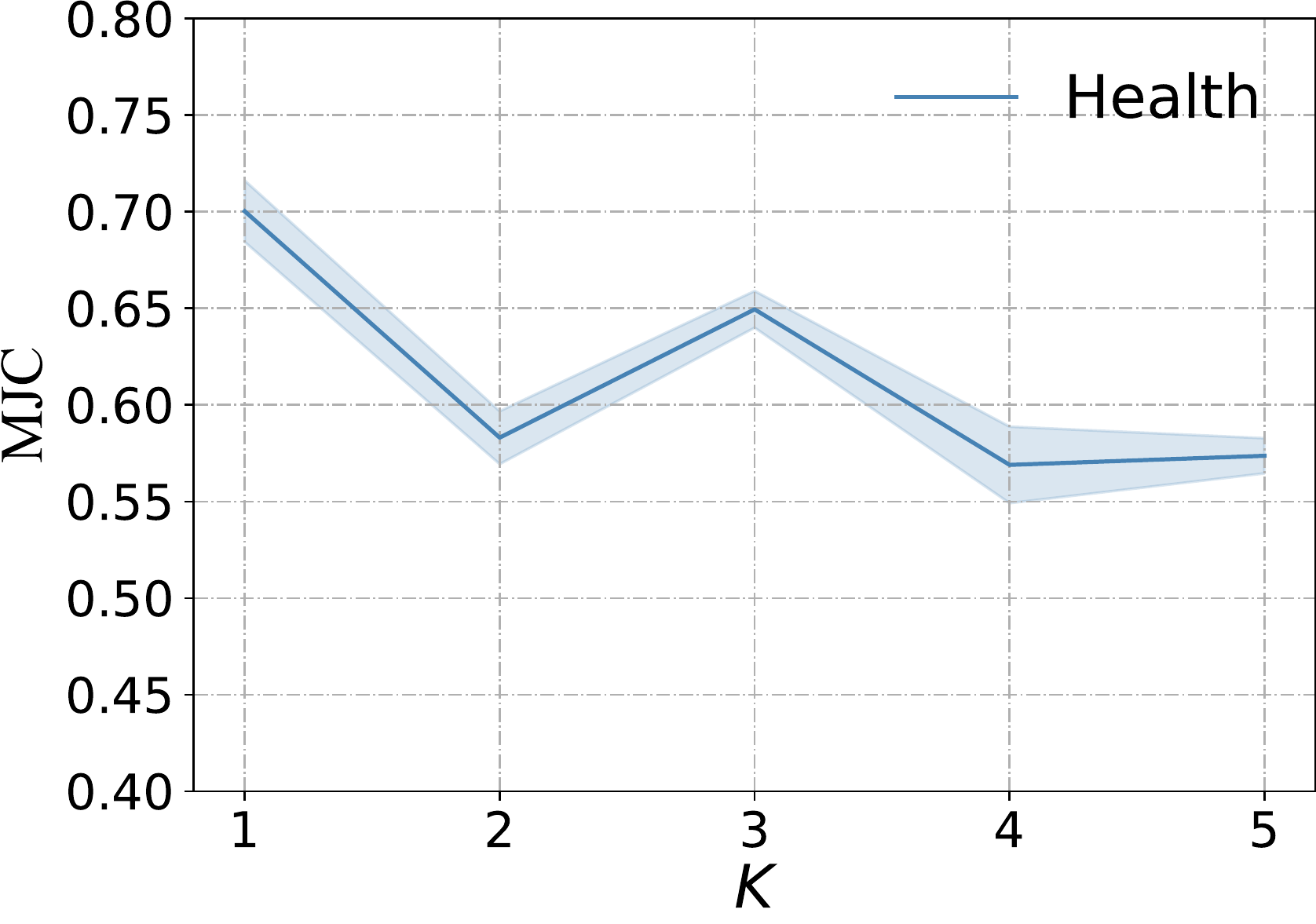}
			\includegraphics[width=1.15in]{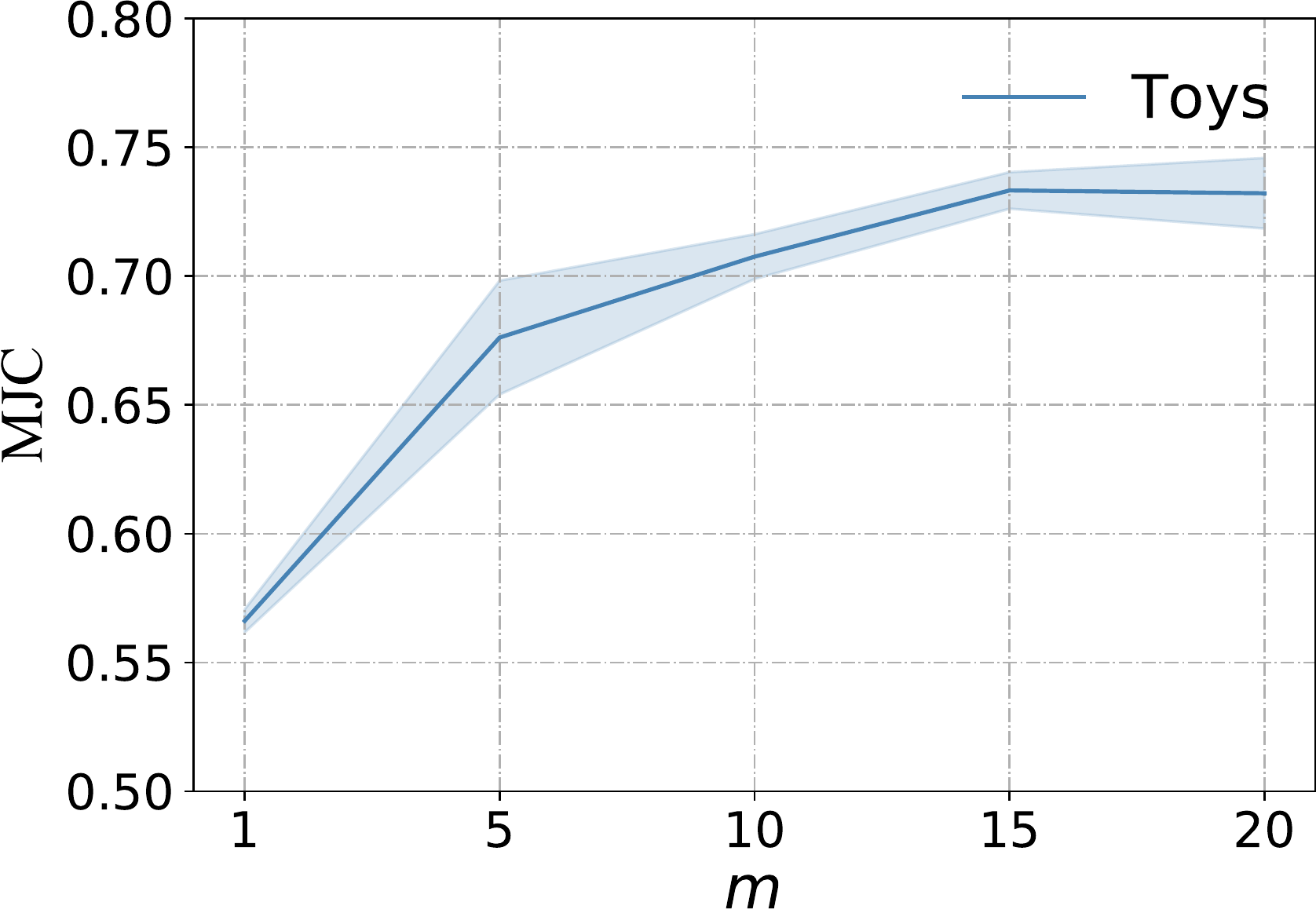}
			\includegraphics[width=1.15in]{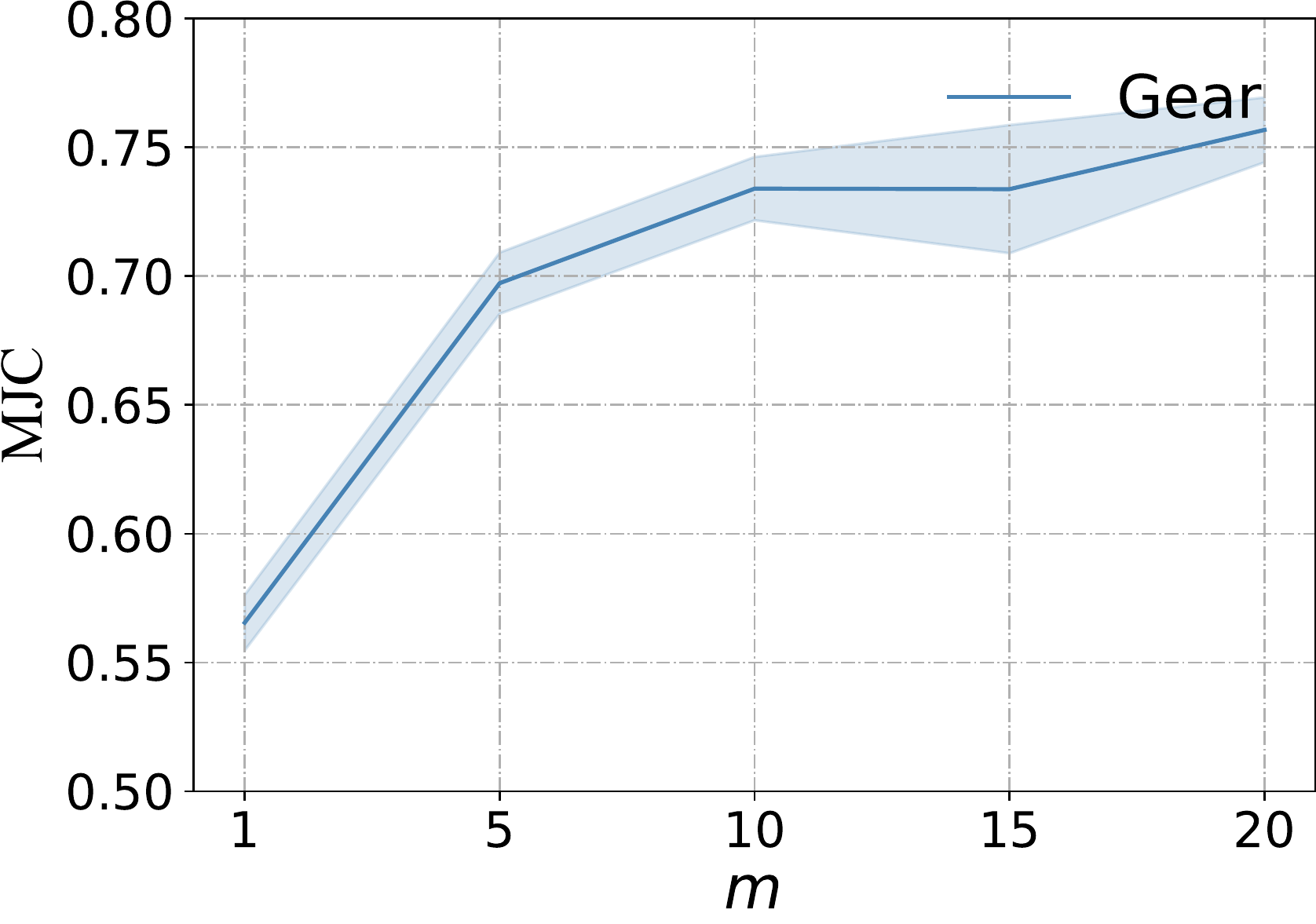}
			\includegraphics[width=1.15in]{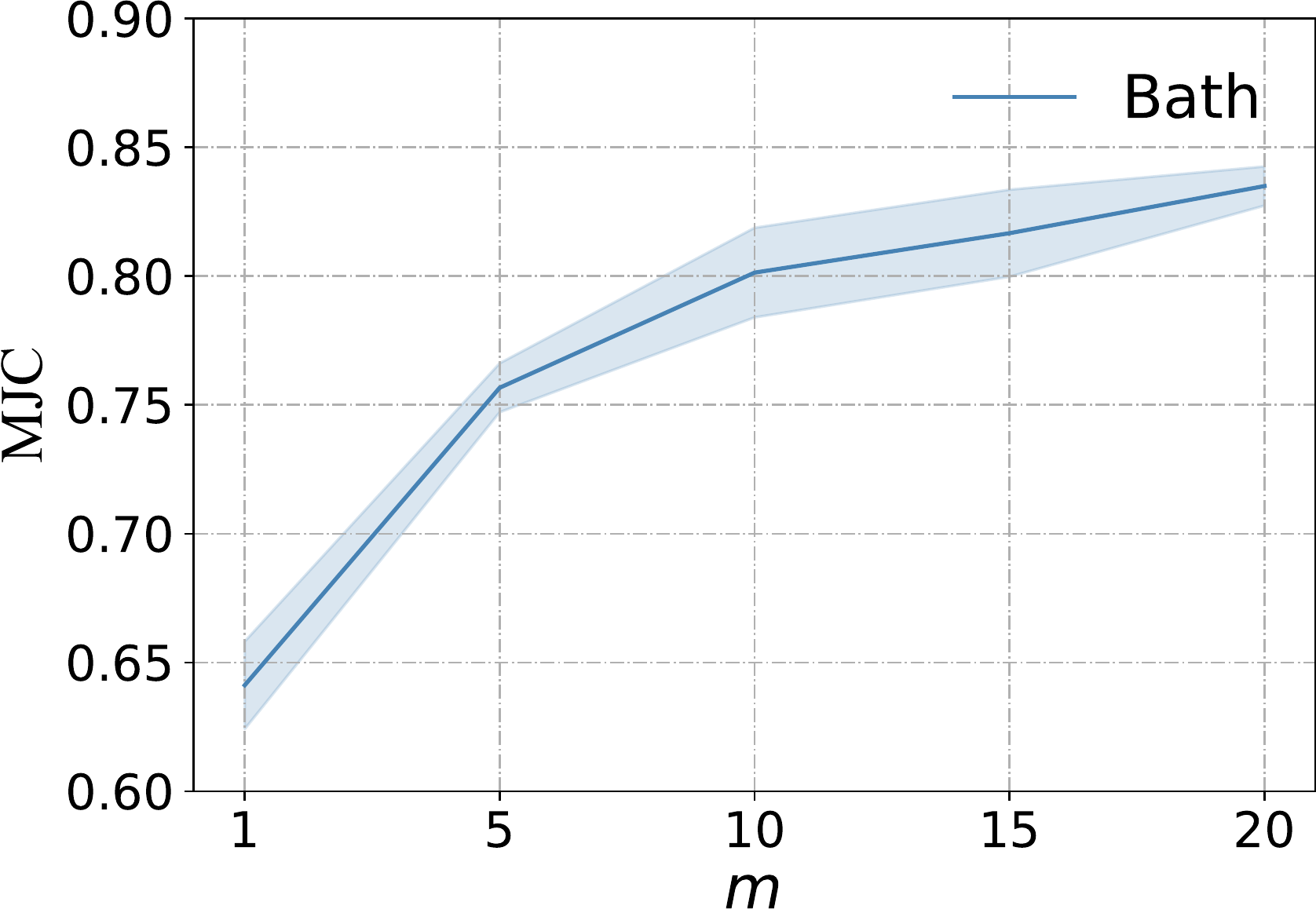}
			\includegraphics[width=1.15in]{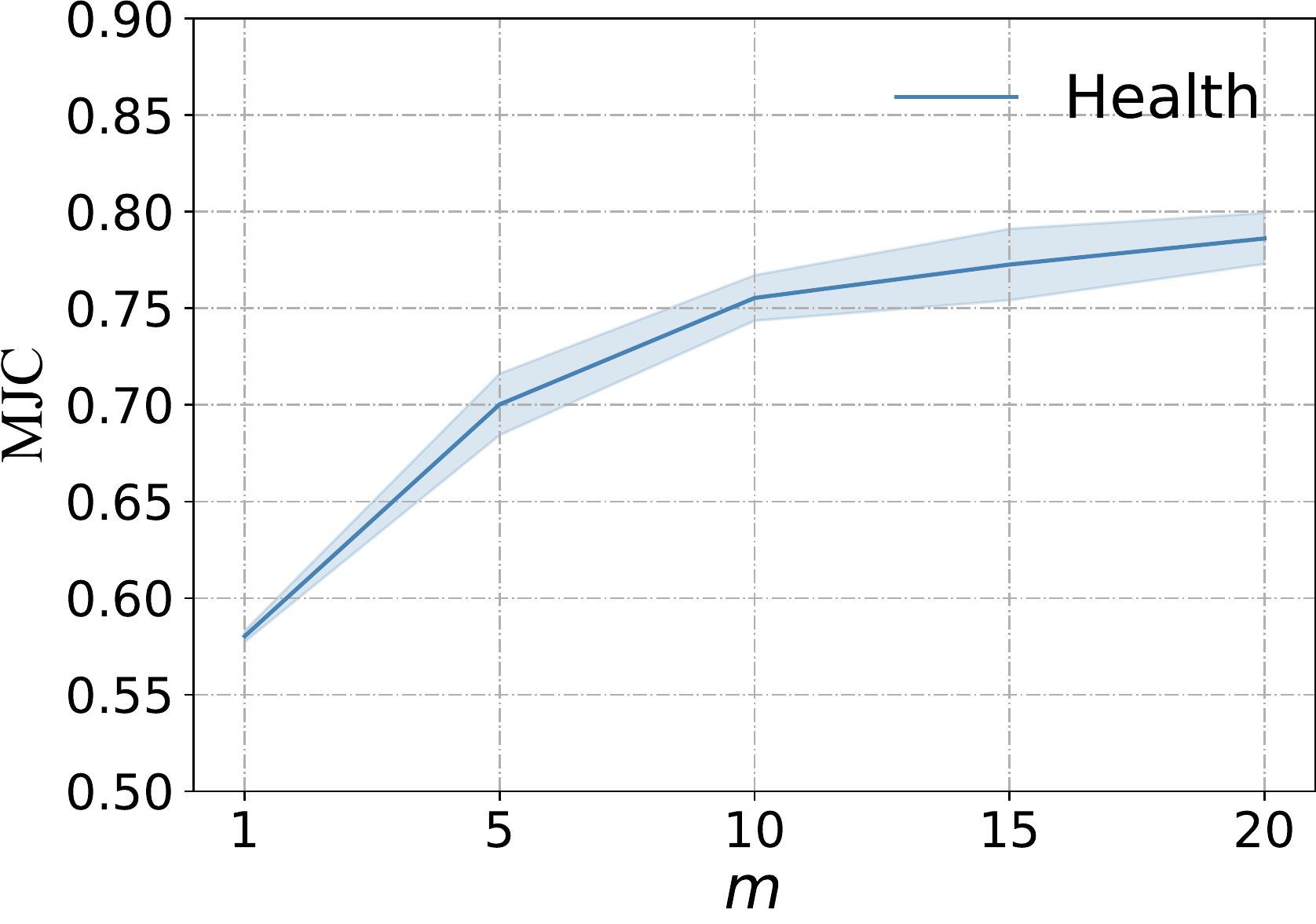}
			\includegraphics[width=1.15in]{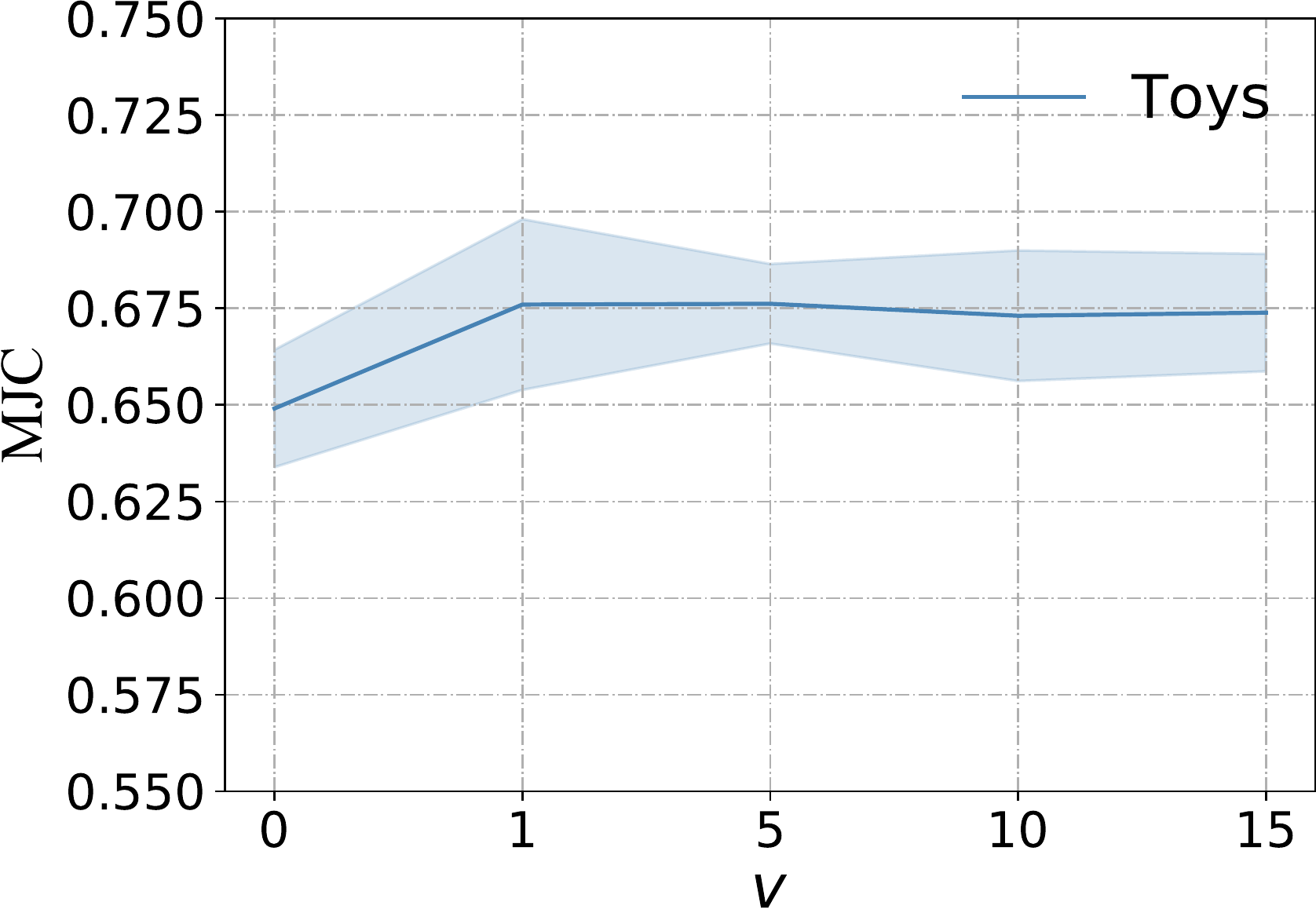}
			\includegraphics[width=1.15in]{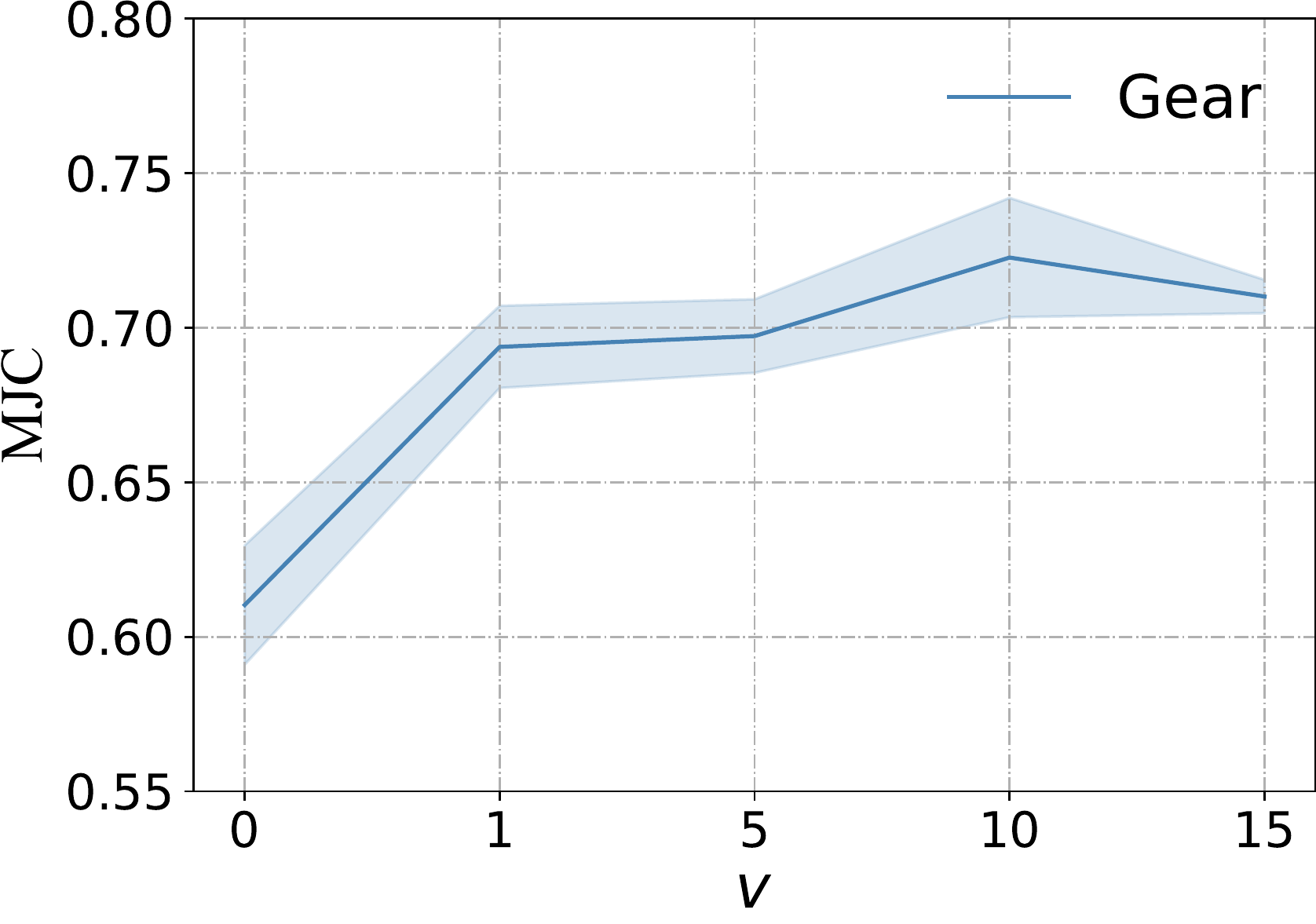}
			\includegraphics[width=1.15in]{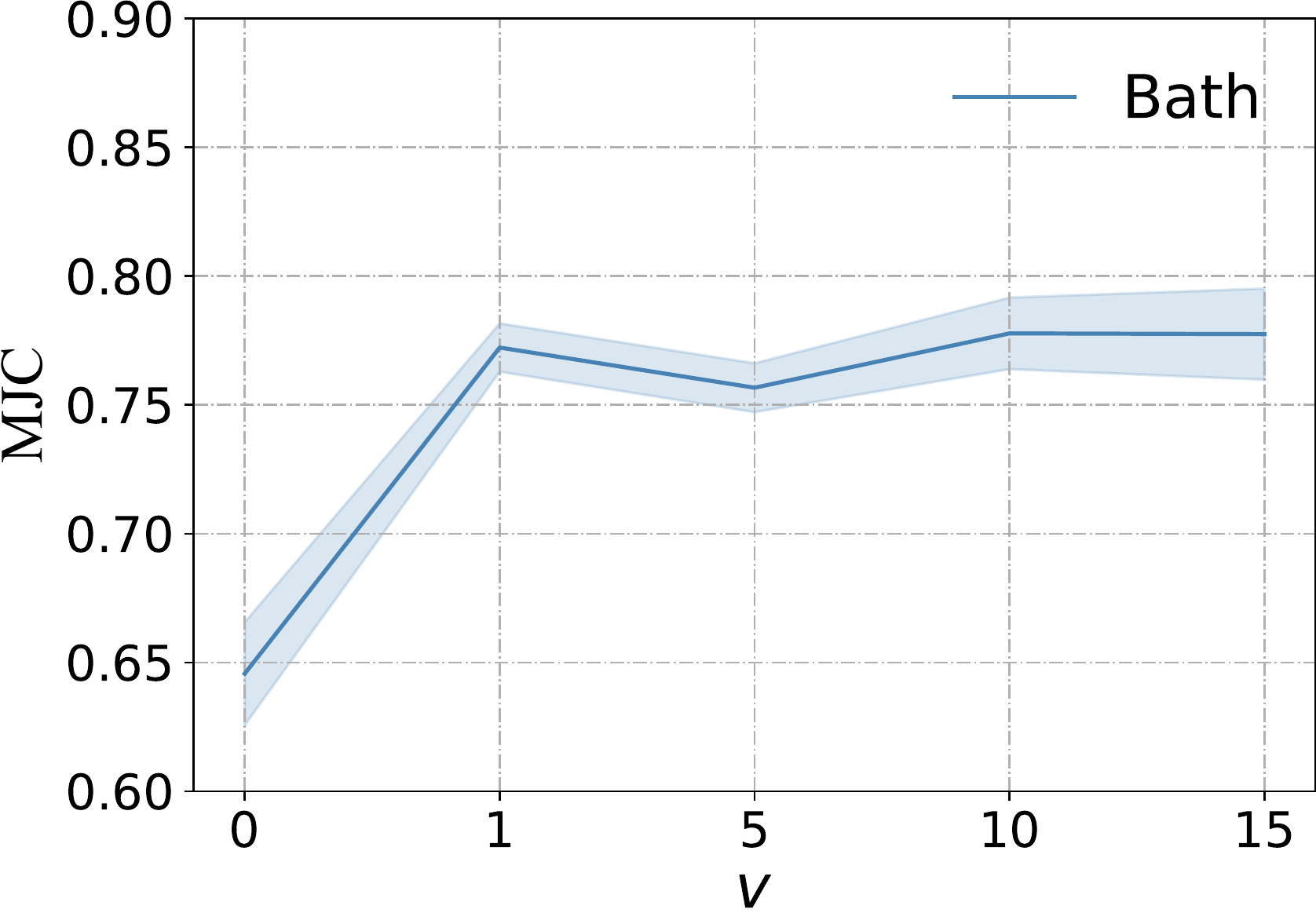}
			\includegraphics[width=1.15in]{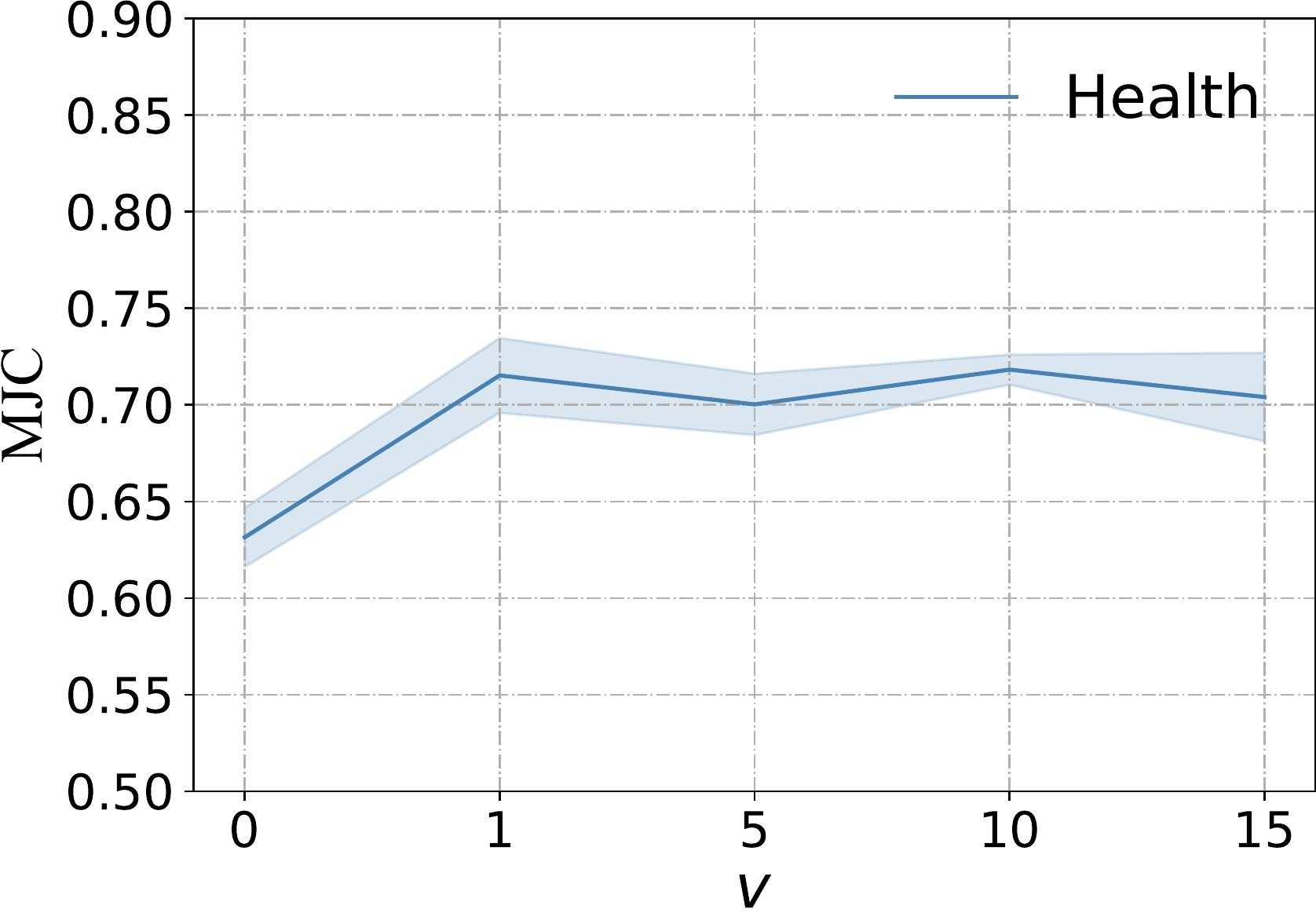}
			\includegraphics[width=1.15in]{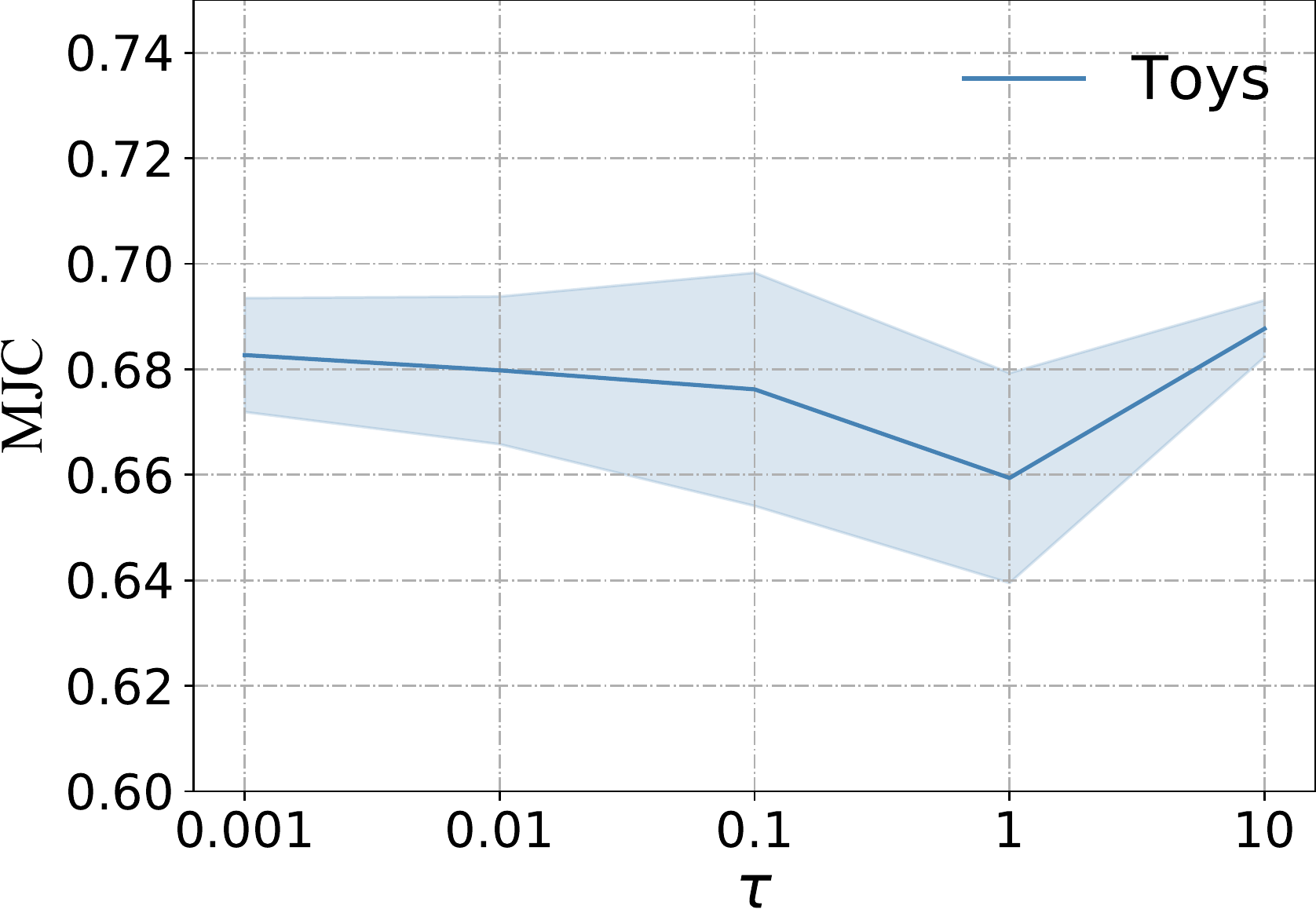}
			\includegraphics[width=1.15in]{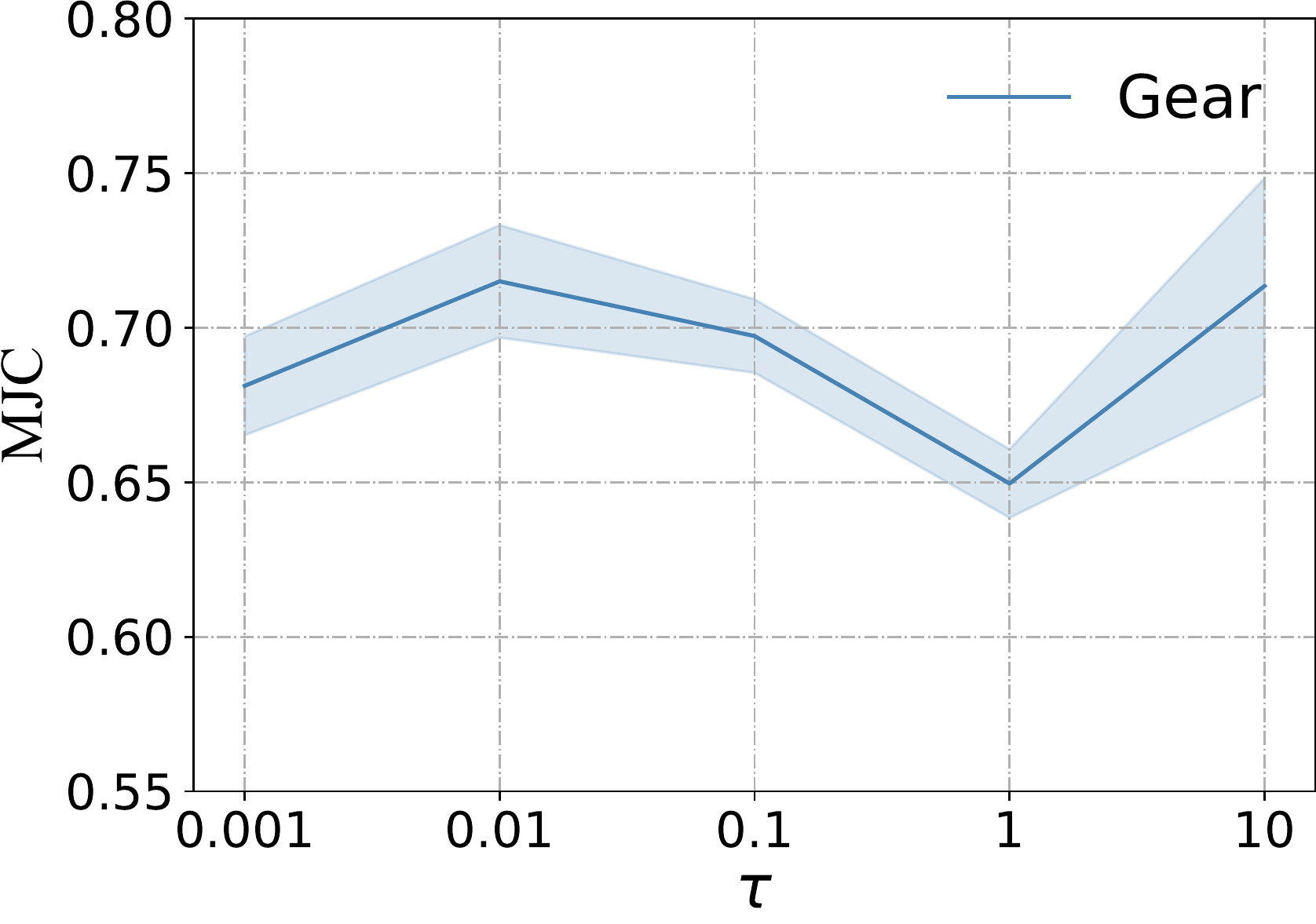}
			\includegraphics[width=1.15in]{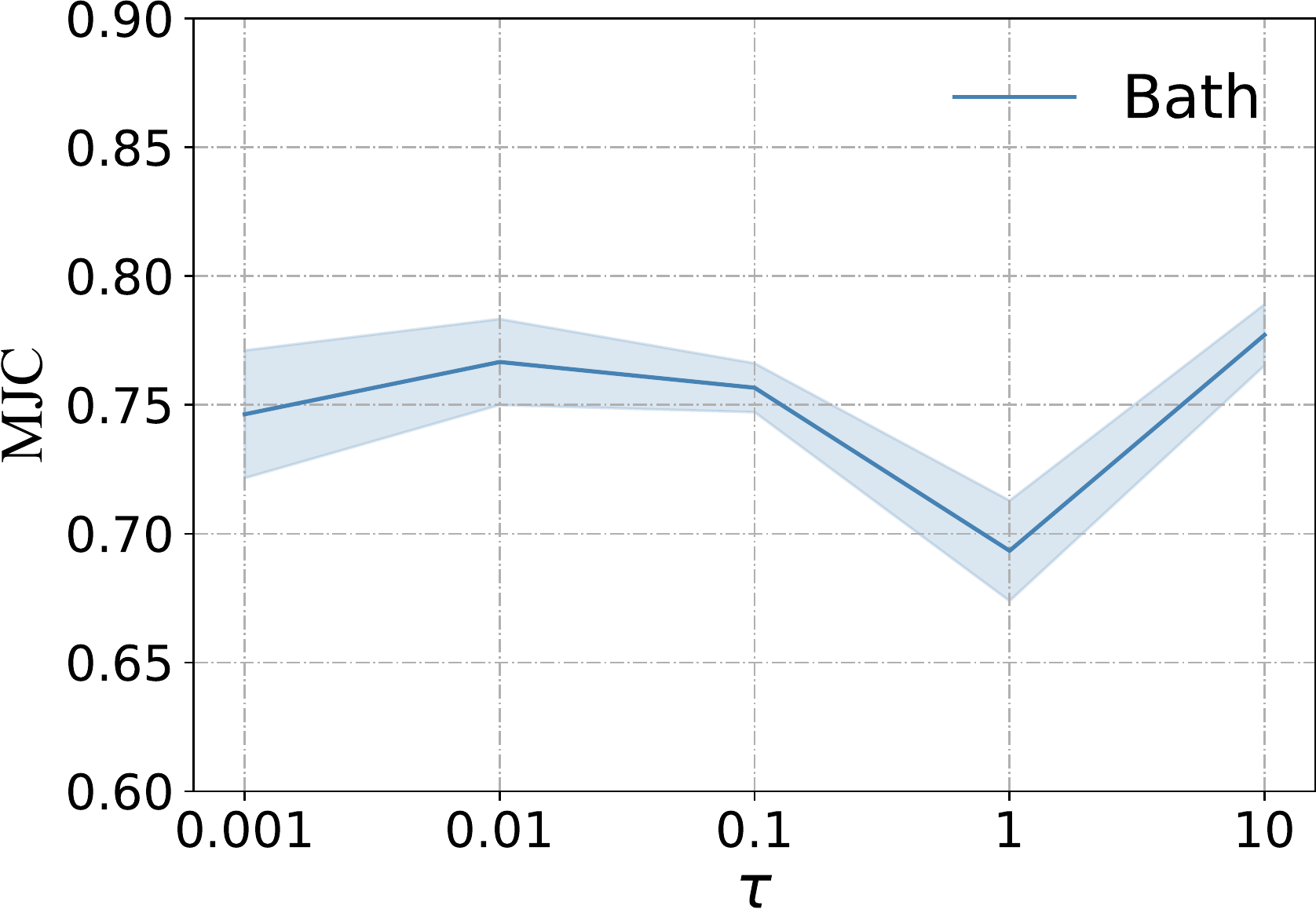}
			\includegraphics[width=1.15in]{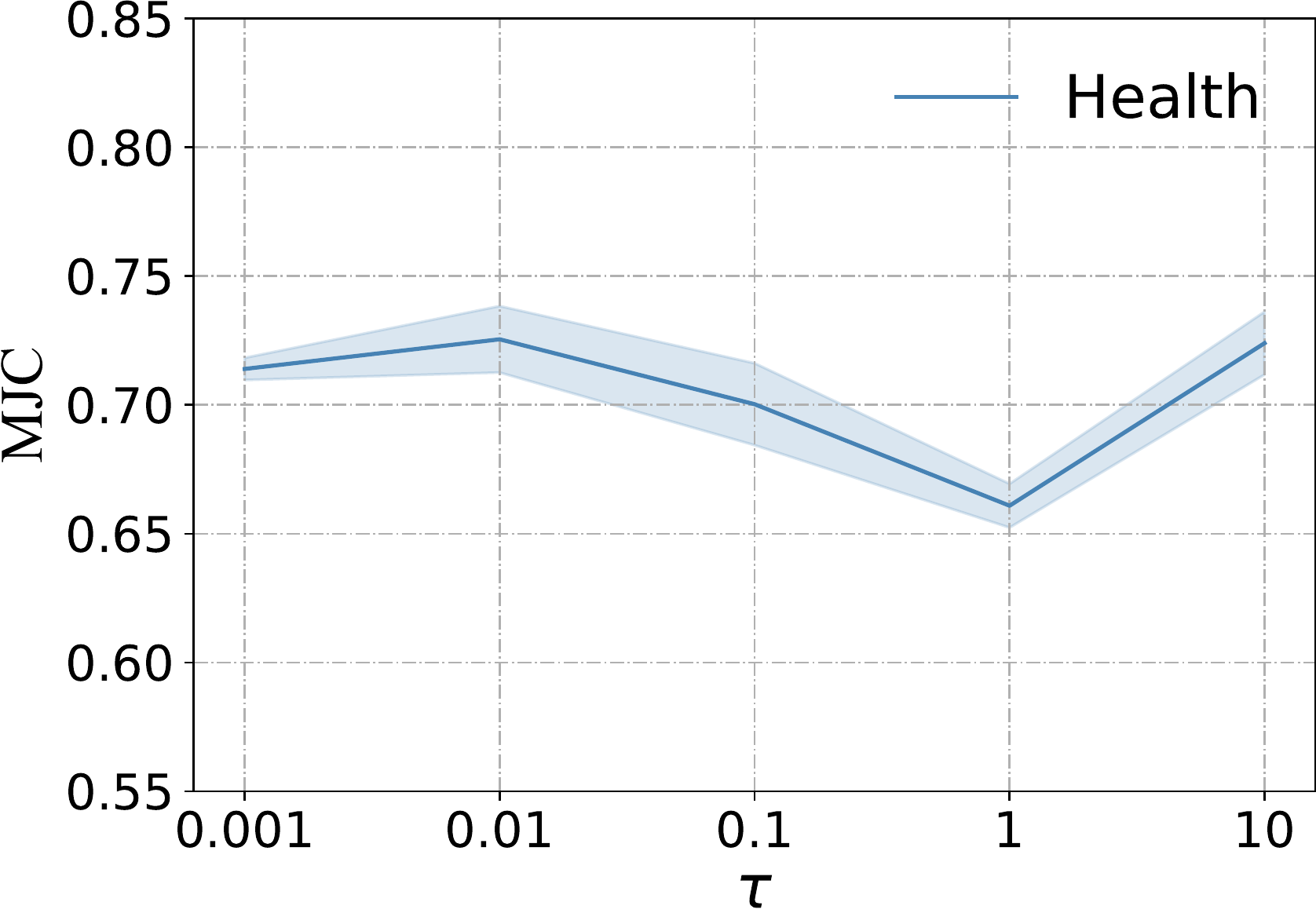}
	        \caption{ Sensitivity analysis of performance of \ourscopula under different hyperparameters (from top to bottom: the number of mean field iteration step $K$, the number of MC sampling $m$, the rank of lower-rank perturbation $v$, and the temperature of Gumbel-Softmax trick $\tau$).}
	        \label{fig:sensitive_analysis}
		\end{minipage}
\end{figure}

\end{document}